%% file: speed_main.tex
\documentclass[twoside]{article}

\usepackage[accepted]{aistats2024}
\usepackage{macros}

\begin{document}

%

%

\twocolumn[

\aistatstitle{SPEED: Experimental Design for Policy Evaluation in Linear Heteroscedastic Bandits}

\aistatsauthor{Subhojyoti Mukherjee$^{*}$ \And Qiaomin Xie$^{*}$ \And Josiah P. Hanna$^{*}$ \And Robert Nowak$^{*}$}


\aistatsaddress{ $^{*}$University of Wisconsin-Madison   } ]

\begin{abstract}
    In this paper, we study the problem of optimal data collection for policy evaluation in linear bandits. 
    In policy evaluation, we are given a \textit{target} policy and asked to estimate the expected reward it will obtain when executed in a multi-armed bandit environment. 
    Our work is the first work that focuses on such an optimal data collection strategy for policy evaluation involving heteroscedastic reward noise in the linear bandit setting. 
    We first formulate an optimal design for weighted least squares estimates in the heteroscedastic linear bandit setting with the knowledge of noise variances. This design minimizes the mean squared error (MSE) of the estimated value of the target policy and is termed the oracle design. 
    Since the noise variance is typically unknown, we then introduce a novel algorithm, \sp\ (\textbf{S}tructured \textbf{P}olicy \textbf{E}valuation \textbf{E}xperimental \textbf{D}esign), that tracks the oracle design and derive its regret with respect to the oracle design. We show that regret scales as $\widetilde{O}_{}(d^3 n^{-3/2})$ and prove a matching lower bound of $\Omega(d^2 n^{-3/2})$.
    Finally, we evaluate \sp\ on
    a set of policy evaluation tasks and demonstrate that it achieves MSE comparable to an optimal oracle and much lower than simply running the target policy.
\end{abstract}

\section{INTRODUCTION}
\label{sec:intro}

\input{intro}
\section{PRELIMINARIES}\label{sec:prelims}
\input{prelims}


\section{\hspace*{-0.5em}OPTIMAL DESIGN FOR POLICY EVALUATION}
\label{sec:loss-def}
\input{loss_definition}

\label{sec:oracle}
\input{oracle}

\section{\sp\ AND REGRET ANALYSIS}
\label{sec:speed}
\input{regret}
\section{EXPERIMENTS}
\label{sec:expts}

\input{expt}
\section{CONCLUSIONS AND FUTURE DIRECTIONS}
\label{sec:conclusions}
\input{conclusions}

\section*{Acknowledgements}
We would like to thank all the reviewers for their helpful feedback.
Q.\ Xie is supported in part by NSF grant CNS-1955997. J.\ Hanna is supported in part by American Family Insurance through a research partnership
with the University of Wisconsin-Madison’s Data Science Institute. We also thank Justin Weltz, Blake Mason, and Lalit Jain for pointing out errors in the previous version of this paper and helping to improve the draft.

\bibliography{biblio}
\bibliographystyle{plainnat}


\section*{Checklist}

 \begin{enumerate}

 \item For all models and algorithms presented, check if you include:
 \begin{enumerate}
   \item A clear description of the mathematical setting, assumptions, algorithm, and/or model. [Yes] We present the model and assumptions in Section 2. The algorithms are presented in Section 3 and Section 4. 
   \item An analysis of the properties and complexity (time, space, sample size) of any algorithm. [Yes] Results are presented in Section 3 and Section 4.
   \item (Optional) Anonymized source code, with specification of all dependencies, including external libraries. [Yes] In supplementary material.
 \end{enumerate}

 \item For any theoretical claim, check if you include:
 \begin{enumerate}
   \item Statements of the full set of assumptions of all theoretical results. [Yes] See Section 2 for the details.
   \item Complete proofs of all theoretical results. [Yes] Proofs outlines are provided in the main paper, with all detailed proofs are deferred to the appendix. 
   \item Clear explanations of any assumptions. [Yes]     
 \end{enumerate}

 \item For all figures and tables that present empirical results, check if you include:
 \begin{enumerate}
   \item The code, data, and instructions needed to reproduce the main experimental results (either in the supplemental material or as a URL). [Yes] In the supplementary material.
   \item All the training details (e.g., data splits, hyperparameters, how they were chosen). [Yes] In the supplementary material.
         \item A clear definition of the specific measure or statistics and error bars (e.g., with respect to the random seed after running experiments multiple times). [Yes] In the supplementary material.
         \item A description of the computing infrastructure used. (e.g., type of GPUs, internal cluster, or cloud provider). [Not Applicable]
 \end{enumerate}

 \item If you are using existing assets (e.g., code, data, models) or curating/releasing new assets, check if you include:
 \begin{enumerate}
   \item Citations of the creator If your work uses existing assets. [Yes]
   \item The license information of the assets, if applicable. [Not Applicable]
   \item New assets either in the supplemental material or as a URL, if applicable. [Not Applicable]
   \item Information about consent from data providers/curators. [Not Applicable]
   \item Discussion of sensible content if applicable, e.g., personally identifiable information or offensive content. [Not Applicable]
 \end{enumerate}

 \item If you used crowdsourcing or conducted research with human subjects, check if you include:
 \begin{enumerate}
   \item The full text of instructions given to participants and screenshots. [Not Applicable]
   \item Descriptions of potential participant risks, with links to Institutional Review Board (IRB) approvals if applicable. [Not Applicable]
   \item The estimated hourly wage paid to participants and the total amount spent on participant compensation. [Not Applicable]
 \end{enumerate}

 \end{enumerate}

\newpage
\appendix
\onecolumn

\section{APPENDIX}
\input{appendix}

\end{document}

%% file: intro.tex
Bandit policy optimization has been applied in various applications such as {web marketing \citep{bottou2013counterfactual}, web search \citep{li2011unbiased}, and healthcare recommendations \citep{zhou2017residual}}. 
In practice, before widely deploying a learned policy, it is often necessary to have an accurate estimation of its performance (i.e., expected reward).
To this effect, \textit{policy evaluation} is often a critical step as it allows practitioners to determine if a learned policy truly represents improved task performance.
%
%
While off-policy evaluation has been extensively studied as a potential solution \citep{dudik2014doubly, li2015toward, swaminathan2017off, wang2017optimal, su2020doubly, kallus2021optimal, cai2021deep}, in practice, some amount of online evaluation is often required before widescale deployment.
For instance, in web-marketing it is common to run an A/B test with a subset of users before a potential new policy is deployed for all users \citep{kohavi2017online}.
%
%
When online policy evaluation is required, we desire methods that provide an accurate estimate of policy performance with a minimal amount of data collected.
The default choice for online policy evaluation is to simply run the target policy and average the resulting rewards.
However, this approach is sub-optimal when the action space is large or different actions have reward distributions with different variances.

In this paper, we formulate a new experimental design for allocating action samples so as to obtain minimal mean squared error (MSE) for policy evaluation.
%
%
%
Specifically, we consider optimal policy evaluation under the following linear heteroscedastic bandit model.
 %
%
%

Let $\A$ be the set of \emph{actions} and each {$a \in \A$} is associated with a feature vector $\bx(a)\in \R^d$ and {$|\A|=A$}. The reward distribution for each action $a$ has mean $\btheta_*^\top \bx(a)$, for some $\btheta_* \in \R^d$.  Often the variance of the reward distribution is assumed to be the same for all actions, but in this paper, we depart from this assumption.  We consider the setting that the variance is governed by a quadratic function of the form $\bx(a)^\top \bSigma_* \bx(a)$, for some symmetric positive definite matrix $\bSigma_*\in\R^{d\times d}$.  This assumption allows us to capture problems in which both the mean reward and the variance may depend on the action taken, but both vary smoothly in $\bx(a)$.


We briefly contrast our studied setting with other work.
In policy evaluation, the common metric of algorithm performance is regret with respect to the mean squared error of an oracle algorithm that has knowledge of the variances of different reward distributions (i.e., knows $\Sigma^*$). There has been an increasing focus on studying data collection for policy evaluation in bandit settings \citep{zhu2021safe, pmlr-v151-zhu22a, wan2022safe} and there has been some theoretical progress \citep{chaudhuri2017active,fontaine2021online}. 
%
%
Several works \citep{antos2008active, carpentier2012minimax, carpentier2015adaptive, fontaine2021online} have shown that in the classical bandit setting a regret of $\widetilde{O}(A n^{-3/2})$ is possible where $n$ is the total budget of actions that can be tried and $\widetilde{O}$ hides logarithmic factors.
These works have also shown that simply running the target policy to take actions results in a slower decrease of regret at the rate of $\widetilde{O}(A n^{-1})$.
{Note that collecting data through running the target policy is called on-policy sampling.}
The work of \citet{pmlr-v151-zhu22a, wan2022safe} studies the same setting under safety constraints and provides asymptotic error bounds. 
However, none of the above works provides a finite-time regret guarantee for data collection for policy evaluation in the heteroscedastic linear bandit setting.

The closest works to ours \citep{antos2008active,carpentier2012minimax, carpentier2015adaptive,fontaine2021online} either consider unstructured settings or {consider the classical bandit setting.} 
As many real-world bandit applications have $d \ll A$, a natural question arises as to how to build an algorithm for policy evaluation in the heteroscedastic linear bandit setting with unknown $\btheta_*$ and $\bSigma_*$ that can leverage the structure. 
Further, we want the regret of such an algorithm to decrease at a rate faster than $\widetilde{O}(n^{-1})$ (the on-policy regret rate) and to scale with the dimension $d$ instead of actions as $A\gg d$. Note that the regret should scale at least by $d^2$ because the learner needs to probe in $d^2$ dimensions to estimate $\bSigma_*\in\R^{d\times d}$ \citep{wainwright2019high}. Thus, the goal of our work is to answer the question:
\begin{quote}
\begin{center}
    \textbf{Can we design an algorithm to collect data  for policy evaluation that adapts to the variance of each action, and its regret decreases at a rate {faster} than $\widetilde{O}(d^2n^{-1})$?}
\end{center}
\end{quote}
In this paper, we answer this question affirmatively. 
{We make the following novel contributions to the growing literature on online policy evaluation:}
\par
\textbf{1.} We are the first to formulate the policy evaluation problem for heteroscedastic linear bandit setting where the variance of each action $a\in \A$ depends on a lower dimensional co-variance matrix parameter $\bSigma_*\in\R^{d\times d}$ such that variance $\sigma^2(a) = \bx(a)^\top\bSigma_*\bx(a)$. This is a more general heteroscedastic linear bandit setting than studied in \citet{chaudhuri2017active, kirschner2018information, fontaine2021online}, and different than the time-dependent variance model of \citet{zhang2021improved, zhao2022bandit}.
\par
\textbf{2.} We characterize the {MSE} in this setting and show that the optimal design, denoted as \textbf{P}olicy \textbf{E}valuation (PE) Optimal  design that minimizes the {MSE} is different than  A-, D-, E-, G-optimality \citep{pukelsheim2006optimal}. We establish several key properties of this novel \PE design and discuss how we can solve for the design efficiently.
\par
\textbf{3.} Finally, we propose the agnostic algorithm, \sp, that does not know the underlying covariance matrix $\bSigma_*$.  \sp\ tracks the oracle design and we analyze 
its MSE. We then bound the regret of \sp\ compared to an oracle strategy that follows the optimal design with the knowledge of $\bSigma_*$. We show that the regret scales as $O(\tfrac{d^3\log(n)}{n^{3/2}})$ which is an improvement over the regret for the stochastic non-structured bandit setting which scales as $O(\tfrac{A\log(n)}{n^{3/2}})$ \citep{carpentier2011finite, carpentier2012minimax,carpentier2015adaptive, fontaine2021online}. Hence, we answer positively to our main query. We also prove the first lower bound for this setting that scales as $\Omega(\tfrac{d^2\log(n)}{n^{3/2}})$. Finally, we conduct experiments on synthetic and real-life data sets and show that \sp\ lowers the MSE of policy evaluation compared to baseline methods.
We discuss more related works and motivations in \Cref{sec:related}.

%% file: prelims.tex
We study the linear bandit setting where the expected reward for each action is assumed to be a linear function~\citep{mason2021nearly, jamieson2022interactive}.
We define $[m]\coloneqq [1,2,\ldots,m]$. We denote the action space as $\A$ and $|\A| = A$. 
Actions are indexed by $a\in[A]$, and {each action $a$ is associated with a} feature vector $\bx(a)\in\R^d$ with dimension {$d \ll A$}. 
Denote by $\triangle(\A)$ the probability simplex over the action space $\A$ and a policy $\pi\in\triangle(\A)$ as a mapping $\pi: \A \rightarrow [0,1]$ such that $\sum_a \pi(a) = 1$.
%
%

Data collection is performed over $n$ rounds of action selection.
Specifically, at each round $t\in[n]$, the selected action $a_t$ yields a reward:  $r_t =  \bx(a_t)^{\top}\btheta_* + \eta_t$,
where $\btheta_*\in\R^d$ is the \textit{unknown} reward parameter, and $\eta_t$ is zero-mean noise with variance $\sigma^2(a_t)$ and we further assume that $\eta_t$ is $\kappa^2$-subgaussian.
We assume that for each action $a \in \A$ the variance $\sigma^2(a)$ has a lower-dimensional structure such that $\sigma^2(a) = \bv(a)^{\top}\bSigma_*\bv(a)$ where $\bSigma_* \in \R^{d\times d}$ is an \emph{unknown} variance parameter. Observe that the variance depends on the action features, which is called the heteroscedastic noise model \citep{greene2002000, chaudhuri2017active} which differs from the unknown time-dependent variance model of \citet{ zhang2021improved, zhao2022bandit}.
Moreover, {\citet{fontaine2021online} do not consider structure in variances and}  \citet{chaudhuri2017active} only consider a special case of our setting where $\bSigma_*$ is a rank-1 matrix.
We also assume that the norms of the features are bounded such that $H^2_L \leq \|\bx(a)\|^2\leq H^2_U$ for all $a\in\A$. 
In our heteroscedastic linear bandit setting selecting any action gives information about $\btheta_*$ and also gives information about the noise covariance matrix $\bSigma_*$. 

The value of a policy $\pi$ is defined as $v(\pi) \coloneqq \E[R_t]$ where the expectation is taken over $a_{t}\!\sim\!\pi,{R_{t}\!\sim\! \bx(a_t)^\top\btheta_*}+\eta_t$. 
 In the policy evaluation problem, we are given a fixed, target policy $\pi$ and asked to estimate $v(\pi)$. 
Estimating $v(\pi)$ requires a dataset of actions and their associated rewards, $\D \coloneqq \{(a_1, r_1,...,a_{n},r_{n})\}$, which is collected by executing some policy.
We refer to the policy that collects $\D$ as the \textit{behavior policy}, denoted by $\bb\in\triangle(\A)$.
We then define the value estimate of a policy $\pi$ as $Y_n$, where $n$ is the sample budget. The exact nature of the value estimate for the linear bandit setting will be made clear in \Cref{sec:oracle}. Our goal is to choose a behavior policy that minimizes the mean squared error (MSE) defined as
    $\E_{\mathcal{D}}[(Y_n - v(\pi))^{2}]$,
 where the expectation is over the collected data set $\D$.

We  now 
state an assumption on the boundedness on the variance of each action $a\in[A]$. 
%
%
%
%
Let the  singular value decomposition of $\bSigma_*$ be $\bU \bD \bP^{\top}$ with orthogonal matrices $\bU, \bP^{\top}$ and $\bD=\operatorname{diag}\left(\lambda_{1}, \ldots, \lambda_{d}\right)$ where $\{\lambda_{i}\}$ are singular values. It follows that $\sigma^2_{\min} \leq \sigma^2(a)\leq \sigma^2_{\max}$ where $\sigma^2_{\min} = \min_{i}|\lambda_{i}|H_L^2$ and $\sigma^2_{\max} = \max_{i}|\lambda_{i}|H^2_U$ (see \Cref{remark:bound-variance}).
\begin{assumption}
\label{assm:bounded-variance}
We assume that $\bSigma_*$ has its minimum and maximum eigenvalues bounded such that for every action $a\in[A]$ the following holds $\sigma^2_{\min} \leq \sigma^2(a)\leq \sigma^2_{\max}$. 
\end{assumption}

%% file: loss_definition.tex
In this section, we first discuss why following the target policy to take actions can lead to a poor estimation of the value of the policy. This discussion motivates how a different behavior policy can produce more accurate estimates of the target policy's value.
After this motivation, we derive an expression for policy evaluation error in terms of the behavior sampling proportion $\bb\in\triangle(\A)$, target policy $\pi$, and action features $\bx(a)\in\R^d$. 
We call the minimizer of this expression the ``optimal design" \citep{pukelsheim2006optimal} as it minimizes the mean squared error for policy evaluation.
We then analyze the error incurred by an oracle that can compute and follow the optimal behavior policy through knowledge of problem-dependent parameters. 

\textbf{Motivating Example:} Consider the linear bandit environment where $d=2$ and $A=100$ actions. Let one action be along the x-axis, one action along the y-axis, and $98$ actions along the direction of $(\frac{1}{\sqrt{2}}, \frac{1}{\sqrt{2}})$. Assume $\btheta_*$ is in the direction of $x$-axis (so action $1$ is the optimal action). A similar canonical linear bandit setting has been studied by \citet{fiez2019sequential, katz2020empirical}. Consider a target policy $\pi$ such that $\pi(1) = 0.9$ and it distributes $0.1$ probability equally on the remaining actions.  In this case, just running the target policy $\pi$ for $n$ rounds leads to sampling uninformative actions for identifying $\btheta_*$. In fact, in our experiments, we show that the estimate $v(\pi)$ will be inaccurate compared to running the optimal behavior policy (called Oracle policy; see \Cref{fig:linear-expt} top-left). 


Now suppose we divide the budget of $n$ samples across the actions, and let $T_{n}(1), T_{n}(2), \ldots, T_{n}(A)$ be the number of samples allocated to actions $1,2, \ldots, A$ at the end of $n$ rounds. 
After observing $n$ samples, let the \textit{weighted} least square estimate (WLS) be: 
\vspace*{-0.5em}
\begin{align}
\wtheta_{n} \coloneqq \arg\min_{\btheta}\sum_{t=1}^{n}\tfrac{1}{\sigma^{2}(a_{t})}(r_t-\bv(a_t)^{\top}\btheta)^{2} \label{eq:weighted-least-square}
\end{align}
where $a_t$ is the action sampled at round $t$ and $\sigma^2(a_t)$ is the variance of  action $a_t$. Also note that this is an unbiased estimator of $\btheta_*$ (see \Cref{remark:unbiased-estimator}). 
In a linear bandit, we can define the value estimate of a \emph{target policy} as 
$Y_n \coloneqq \sum_a \bw(a)^\top \wtheta_n,$ 
where $\bw(a) \coloneqq \pi(a)\bv(a)$ is the expected feature for each action $a\in \A$ under the target policy, and $\wtheta_n$ is an unbiased estimate of $\btheta_*$ computed with $n$ samples in $\D$. 
As $\wtheta_n$ is an unbiased estimate, we have that $\E_{\D}[Y_{n}]=\sum^A_{a=1}\bw(a)^{\top}\btheta_* = v(\pi)$.
Since we have an unbiased estimator of $v(\pi)$, minimizing the MSE is equivalent to minimizing the variance,  $\min  \E_{\D}\big[\left(Y_n \!-\! \E[Y_n]\right)^2\big] = \min\E_{\D}\big[ \big(\sum_{a=1}^A\bw(a)^{\top}(\wtheta_n \!-\! \btheta_*)\big)^2\big]$,
%
where the minimization is with respect to the data distribution $\D$, which is determined by the behavior policy. In general, the behavior policy that minimizes the MSE may be different from the target policy. 
%
%
%
To identify this optimal behavior policy, following the optimal design literature \citep{pukelsheim2006optimal, fedorov2013theory} we define the design or information matrix $\bA_{\bb,\bSigma_*}\in\R^{d\times d}$ w.r.t.\ each $\bb\in \Delta(\A)$ as
\begin{align}
    \hspace*{-1em}\bA_{\bb,\bSigma_*} \!\!=\!\!\sum_{a\in\A}\bb(a)\big(\tfrac{\bx(a)}{\sigma(a)}\big)\big(\tfrac{\bx(a)}{\sigma(a)}\big)^\top 
    \!\!=\!\! \sum_{a\in\A}\bb(a)\tx(a)\tx(a)^\top \label{eq:design-matrix}
\end{align}
where $\tx(a) = \bx(a)/\sigma(a)$. 
Observe that our design matrix in \eqref{eq:design-matrix} captures the information about the action features $\bx(a)$, and variance $\sigma^2(a)$ and weights them by the sampling proportion $\bb(a)$. Then in the following proposition, we exactly characterize the MSE with respect to the design matrix $\bA_{\bb,\bSigma_*}$, target policy $\pi$ and action features $\bx$. 
Moving forward, we will use the term \textit{loss} interchangeably with MSE.
\begin{customproposition}{1}
\label{prop:linear-bandit}
Let $\wtheta_n$ be the Weighted Least Square (WLS) estimate \eqref{eq:weighted-least-square} of $\btheta_*$ after observing $n$ samples and define $\bw(a) = \pi(a)\bx(a)$. Define the design matrix as $\bA_{\bb,\bSigma_*}$ (see \eqref{eq:design-matrix}). Then the loss is given by
\begin{align*}
    \E_\D\!\big[\big(\sum_{a=1}^A\bw(a)^\top(\wtheta_n-\btheta_*)\!\big)^{2}\big] \!\!=\!\! \underbrace{\frac{1}{n}\! \!\sum_{a,a'}\bw(a)^\top\bA_{\bb,\bSigma_*}^{-1}\bw(a')\!}_{\coloneqq\L_n(\pi,\bb,\bSigma_*)} .
\end{align*}
\end{customproposition}

\textbf{Proof (Overview)}
The key idea is to show that the linear model yields for each action 
$a \in[A], \widetilde{Y}_{n}(a) = \widetilde{\bx}_{n}(a)^{\top} \btheta^{\star}+\widetilde{\eta}_{n}(a)$ where we define 
\vspace*{-0.25em}
\begin{align*}
&\widetilde{Y}_{n}(a)=\sum_{i=1}^{T_{n}(a)} \tfrac{R_{i}^{}(a)}{\sigma(a) \sqrt{T_{n}(a)}}, \tx_n(a)=\tfrac{\sqrt{T_n(a)} \bx(a)}{ \sigma(a)},\\
&\widetilde{\eta}_{n}(a)=\sum_{i=1}^{T_n(a)} \tfrac{\eta_{i}^{}(a)}{\sigma(a) \sqrt{T_n(a)}},
\end{align*}
with $R_i(a)$ being the reward observed for action $a$ taken for the $i$-th time, $\eta_i(a)$ being the corresponding noise, and $T_n(a)$ is the number of samples of action $a$. Next, observe that using the independent noise assumption, we have that  $\E[\widetilde{\eta}_{n}(a)]=0$ and $\Var\left[\widetilde{\eta}_{n}(a)\right]=1$. Let $\bX=$ $\left(\widetilde{\bx}_{n}(1)^{\top}, \cdots, \widetilde{\bx}_{n}(A)^{\top}\right)^{\top} \in \mathbb{R}^{A \times d}$ be the induced feature matrix of the policy and $\mathbf{Y} = [\widetilde{Y}_{n}(1), \widetilde{Y}_{n}(2),\ldots, \widetilde{Y}_{n}(A)]^\top$. The above weighted least squares (WLS) problem has an optimal unbiased estimator $\wtheta_n=\left(\bX^{\top} \bX\right)^{-1} \bX^{\top} \mathbf{Y}$ \citep{fontaine2021online}. 
Substituting the definition of $\wtheta_n$ yields the desired expression of the loss as stated in the proposition. The detailed proof is given in \Cref{app:bandit-prop}. \hfill$\blacksquare$

Observe that the loss in our setting depends on the inverse of the design matrix denoted by $\bA^{-1}_{\bb,\bSigma_*}$, the target policy, as well as features of action pairs $(a,a')\in\A\times \A$. Hence, minimizing the loss is equivalent to minimizing the quantity $1/n (\sum_{a,a'}\bw(a)^\top\bA_{\bb,\bSigma_*}^{-1}\bw(a'))$. As this design is different than a number of existing notions of optimality such as D-, E-, T-, or G-optimality \citep{pukelsheim2006optimal, fedorov2013theory, jamieson2022interactive}, we call this the \textit{\PE design}. None of these previously proposed designs capture the objective of minimal MSE for policy evaluation.
For example, G-optimality (as studied by \citep{katz2020empirical,  mason2021nearly, katz2021improved, mukherjee2023efficient, mukherjee2024multi}) minimizes the worst-case error of $\max_{\bx(a)}\E_{\D}[(\bx(a)^\top(\wtheta_n-\btheta_*))^2]$ by minimizing the quantity $\max_{\bx(a)}\bx(a)^{\top}\bA^{-1}_{\bb}\bx(a)$ for homoscedastic noise. The E-optimal design minimizes $\max_{\|\mathbf{u}\|\leq 1}\E_{\D}[(\mathbf{u}^\top(\wtheta_n-\btheta_*))^2]$ by minimizing the minimum eigenvalue of the inverse of design matrix  \citep{mukherjee2022chernoff} and the A-optimal design minimizes $\E_{\D}[(\wtheta_n - \btheta_*)^2]$ by minimizing the trace of the inverse of design matrix \citep{fontaine2021online}.
%

We now state a few more notations for ease of exposition. Using \Cref{prop:linear-bandit} we define the  optimal behavior policy when the matrix $\bSigma_*$ is known as:
\begin{align}
    \bb_* \coloneqq \argmin_{\bb}\L_n(\pi, \bb, \bSigma_*), \label{eq:opt-oracle-sol}
\end{align}
where the loss $\L_n(\pi, \bb, \bSigma_*)$ is defined in \Cref{prop:linear-bandit}. 
We define the optimal loss (with knowledge of $\bSigma_*$)  as: 
\begin{align}
    \L^*_n(\pi, \bb_*, \bSigma_*) = \min_{\bb}\L_n(\pi, \bb, \bSigma_*). \label{eq:opt-oracle-loss}
\end{align}

\vspace*{-1.5em}
\subsection{Computation of the optimal design $\bb_*$}
\vspace*{-0.5em}


In this section, we digress a bit to discuss the computational aspect of $\L_n(\pi, \bb, \bSigma_*)$. Since \PE design is a new type of design, the natural question to ask is \emph{how to optimize this loss function w.r.t.\ $\bb$?} We show 
%
%
in \Cref{prop:convex-loss} that the loss $\L_n(\pi, \bb, \bSigma_*)$ for any arbitrary design proportion $\bb\in\triangle(\A)$  is strictly convex with respect to the proportion $\bb$. The proposition and its proof are given in \Cref{app:convex-loss}. Next in \Cref{prop:gradient-loss} we show that the gradient of the loss function is bounded. Due to space constraints, both propositions and their proofs are given in \Cref{app:convex-loss} and \Cref{app:gradient-loss} respectively. We first state an assumption  that the minimum eigenvalue satisfies $\lambda_{\min }\left(\sum_{a=1}^A \bw(a) \bw(a)^{\top}\right)>0$, which is required for proving \Cref{prop:gradient-loss}.


\begin{assumption}
\label{assm:target-dist}
\textbf{(Distribution of $\pi$)} 
We assume that the set of actions $a$ such that $\!\pi(a) \!\!>\!\! 0$, spans $\R^d$ and $\R^{d\times d}$.
\end{assumption}
Note that this is a realistic and not a restrictive assumption, since if the target policy never takes an action that is needed to cover some dimension then we can avoid identifying $\btheta_*$ in that dimension. 
%
Using \Cref{prop:convex-loss}, \ref{prop:gradient-loss} we can effectively solve the \PE design with gradient descent approaches \citep{lacoste2013affine, berthet2017fast}. 
We capture this convergence guarantee with the assumption of the existence of an approximation oracle.
\begin{assumption}\textbf{(Approximation Oracle)}
\label{assm:oracle-approx}
We assume access to an approximation oracle. Given a convex loss function $\L_n(\pi,\bb ,\bSigma_*)$ with minimizer $\bb_*$, the approximation oracle returns a proportion 
$\wb_* = \argmin_{\bb} \L_n(\pi,\bb ,\bSigma_*)$ such that 
    $|\L_n(\pi,\wb_* ,\bSigma_*) - \L_n(\pi,\bb_* ,\bSigma_*)| \leq \epsilon$.
\end{assumption}
Therefore from \Cref{prop:convex-loss}, and \ref{prop:gradient-loss} and using \Cref{assm:target-dist}, and \ref{assm:oracle-approx} we can get a computationally efficient solution to $\min_{\bb\in \Delta(\A)}\L_n(\pi, \bb, \bSigma_*)$. 
%

%% file: oracle.tex
\vspace*{-0.25em}
\subsection{Oracle Loss}
\label{sec:oracle-loss}
\vspace*{-0.25em}



Recall from \Cref{sec:intro}, that our final goal is to control the regret (excess loss) of an agnostic algorithm that does not know $\bSigma_*$, with respect to an oracle that already knows $\bSigma_*$.
Towards this goal, in this section, we develop our theory for optimal data collection by considering an oracle for the heteroscedastic linear bandit setting. 
%
%
Specifically, we consider an oracle that has knowledge of $\bSigma_*$ but does not know $\btheta_*$. 
With this knowledge, it can solve \Cref{eq:opt-oracle-sol} (\Cref{assm:oracle-approx}) to determine the \PE design, $\bb_*$, that minimizes the loss. 
%
The oracle takes actions in proportion $\bb_*$ for $n$ samples and then computes the WLS estimate $\wtheta_n$ using $\bSigma_*$.
%
%
%
%
The following proposition then bounds the loss of the oracle after $n$ samples.

\begin{customproposition}{5}
\label{prop:loss-oracle}
\textbf{(Oracle Loss)}
Let the oracle sample each action $a$ for $\lceil n \bb_*(a)\rceil$ times, where $\bb_*$ is the solution to \eqref{eq:opt-oracle-sol}. Define $\lambda_1(\bV)$ as the maximum eigenvalue of $\bV:=\sum_{a,a'}\bw(a)\bw(a')^{\top}$. Then the loss satisfies $
    \L^*_n(\pi, \bb_*, \bSigma_*) \leq O_{\kappa^2,H^2_U}\footnote{Here $O_{\kappa^2,H^2_U}()$ hides the sub-Gaussian factor $\kappa^2$ and  upper bound $H^2_U$ on feature norm }\left(\tfrac{ d\lambda_1(\bV)
    \log n}{n}\right) + O_{\kappa^2,H^2_U}\left(\tfrac{1}{n}\right).
$
\end{customproposition}

\textbf{Proof (Overview)} Note that the oracle knows the $\bSigma_*$ and uses $\wtheta_n$ in \eqref{eq:weighted-least-square} to estimate $\btheta_*$. We use \Cref{corollary:kiefer} to show that $\L_n(\pi, \bb_*,\bSigma_*) \leq \lambda_1(\bV) d$ where $\bV = \sum_{a,a'}\bw(a)\bw(a')^\top$. 
The proof follows by showing that $(\sum_{a=1}^A \bw(a)^\top (\wtheta_n-\btheta_{*}))^2$ is a sub-exponential variable. 
Then using sub-exponential concentration inequality in \Cref{conc-lemma-sub-exp} (\Cref{app:prob-tools}) and setting $\delta=O(1/n^2)$ we can bound the expected loss with high probability. The full proof is given in \Cref{app:loss-bandit-oracle}. \hfill$\blacksquare$

\textbf{Connection to prior work:} 
Prior work has considered a similar oracle for the basic stochastic bandit setting, which is a special case of our setting with $\bv(a)$ being a one-hot vector in $\mathbb{R}^A$. 
In this case, we can see that $\bb_* = \arg\min_\bb \sum_a\frac{\pi^2(a)\sigma^2(a)}{\lceil \bb(a)n\rceil}$.
This captures the optimal number of times the actions should be pulled weighted by the target policy and their variance. 
%
Solving for $\bb_*$, we obtain $\bb_*(a)\propto \pi^2(a)\sigma^2(a)$.
This solution matches the optimal sampling proportion given by \citet{antos2008active,  carpentier2011finite, carpentier2012minimax, carpentier2015adaptive} for this special case.
The loss in prior work decays at the rate of $\widetilde{O}\footnote{Here $\widetilde{O}$ hides logarithmic and problem dependent factors like $\sigma^2_{\min}, \kappa^2, H^2_U$.}(An^{-1})$ whereas the loss in \Cref{prop:loss-oracle} decreases at the rate of $\widetilde{O}(d n^{-1})$. Also note the loss in \Cref{prop:loss-oracle} scales with $d$ instead of $d^2$ as the oracle knows the $\bSigma_*$ and does not need to explore $d^2$ directions to estimate $\bSigma_*$.
So we obtain an equivalence between the \PE design and the solution from prior work in the basic bandit setting while considering a more general setting.

%% file: regret.tex
\vspace*{-0.4em}
In this section, we first introduce an agnostic algorithm called \sp\ for data collection that does not know $\bSigma_*$, and then analyze its regret. Here, regret refers to the excess loss relative to the oracle that knows $\bSigma_*$.

\vspace*{-1em}
\subsection{Details of Algorithm \sp}
\vspace*{-0.4 em}

In practice, $\bSigma_*$ is unknown and so the oracle behavior policy cannot be directly computed.
Instead, we first conduct a small amount of exploration to estimate $\bSigma_*$ and then use the estimate in place of $\bSigma_*$ in \eqref{eq:design-matrix}.
Specifically, we define the forced exploration phase as the first $\Gamma$ rounds in which the algorithm conducts exploration to estimate $\bSigma_*$. 
%
%
To ensure adequate exploration, we first apply Principal Component Analysis (PCA) on the feature matrix $\bX$ and choose the most significant $d$ directions (directions having the highest variance). Then we choose one random action for each of these $d$ significant directions and sample these actions uniform randomly for $\Gamma$ rounds.
Since the algorithm explores first and then uses the estimate to compute the \PE design, it can be viewed as an explore-then-commit algorithm \citep{rusmevichientong2010linearly, lattimore2020bandit}. 
As we consider a structured setting we call this algorithm \textbf{S}tructured \textbf{P}olicy \textbf{E}valuation \textbf{E}xperimental \textbf{D}esign (\sp). 
After $\Gamma = \sqrt{n}$ rounds, \sp{} estimates the covariance matrix $\wSigma_\Gamma$ as follows:
\begin{align}
    \hspace*{-1.4em}\wSigma_\Gamma \!=\! \!\min_{\mathbf{S}\in \R^{d\times d}}\!\sum_{t=1}^{\Gamma}\!\! \big[\langle \bv(a_{t}) \bv(a_{t})^{\top}, \mathbf{S}\rangle\!-\!(r_t\!-\!\bv(a_{t})^{\top} \wtheta_{\Gamma})^{2}\big]^{2} \label{eq:wSigma-Gamma}
\end{align}
where $\wtheta_\Gamma$ is the ordinary least square (OLS) estimate of $\btheta_*$ using the data from the first $\Gamma$ rounds. Note that the OLS estimate is given by $\wtheta_\Gamma = (\bX^\top\bX)^{-1}\bX^\top \bY, $ where $\bX= \left(\bx_1^{\top}, \cdots, \bx_{\Gamma}^{\top}\right)^{\top}$ and $\mathbf{Y} = [r_1, \ldots, r_\Gamma]^\top$.
%
A covariance estimation technique similar to \eqref{eq:wSigma-Gamma} has been considered for the active regression setting though only for the case when $\bSigma_*$ has rank 1 \citep{chaudhuri2017active}.
The estimate of the covariance matrix $\wSigma_\Gamma$ is then fed to the oracle optimizer (\Cref{assm:oracle-approx}) to compute the sampling proportion $\wb$. 
Actions are chosen according to $\wb$ for the remaining $n-\Gamma$ rounds and then the WLS estimate $\wtheta_{n-\Gamma}$ is computed using $\wSigma_\Gamma$ as the covariance matrix parameter (\Cref{eq:weighted-least-square}). 
Finally, \sp\ outputs the dataset $\D$ to estimate the value of target policy $\pi$ and $\wtheta_{n-\Gamma}$.
Full pseudocode is given in \Cref{alg:linear-bandit}.
\begin{algorithm}[!tbh]
\caption{Structured Policy Evaluation Experimental Design (\sp)}
\label{alg:linear-bandit}
\begin{algorithmic}[1]
\State \textbf{Input:} Action set $\A$, target policy $\pi$, budget $n$.
\State Conduct forced exploration for $\Gamma=\sqrt{n}$ rounds and estimate $\wSigma_\Gamma$ using \eqref{eq:wSigma-Gamma}. 
\State Let $\wb_{} \in \triangle(\A)$ be the 
minimizer of $\L_n(\pi,\bb ,\wSigma_\Gamma)$. 
\State Pull each action $a$ exactly $T_n(a) = \left\lfloor\wb_{}(a) (n-\Gamma) \right\rfloor$ times, and let  $\H(a) \coloneqq \{a,R_i(a)\}_{i=1}^{T_n(a)}$ be the corresponding data. Set $\D\leftarrow \cup_a\H(a)$.  
\State Construct the weighted least squares estimator $\wtheta_{n-\Gamma}$ using only the observations $\D$ from step 4.
\State \textbf{Output:} $\D$ and $\wtheta_{n-\Gamma}$.
\end{algorithmic}
\end{algorithm}

\vspace*{-1em}
\subsection{Regret Analysis of \sp}
\vspace*{-0.5em}

In this section, we first state our regret definition and then analyze the regret of the agnostic algorithm \sp.
As an agnostic algorithm, \sp\ does not know the true covariance matrix $\bSigma_*$ and must estimate the covariance matrix $\wSigma_\Gamma$ after conducting exploration for $\Gamma$ rounds. 
%
%
%
We define the loss of an algorithm after exploring for $\Gamma$ rounds as the MSE of the resulting value estimate as follows:
\vspace*{-0.25em}
\begin{align}
    \bL_n(\pi, \wb, \wSigma_\Gamma) := \E_\D\!\big[\big(\sum_{a=1}^A\bw(a)^\top(\wtheta_{n-\Gamma}-\btheta_*)\!\big)^{2}\big],
    %
    \label{eq:opt-agnostic-loss}
\end{align}
where $\wtheta_{n-\Gamma}$ is the WLS estimate of $\btheta_*$ calculated from data of last $n-\Gamma$ rounds. 
%
%
We now define the regret for the agnostic algorithm with the estimated behavior policy $\wb$ as
\begin{align}
    \cR_n &= \bL_n(\pi,\wb,\wSigma_\Gamma) - \L^*_n(\pi,\bb_*,\bSigma_*). \label{eq:regret-definition}
\end{align}
where $\bL_n(\pi,\wb,\wSigma_\Gamma)$ is the loss of the agnostic algorithm and
$\L_n(\pi,\bb_*,\bSigma_*)$ is the oracle loss defined in \eqref{eq:opt-oracle-loss}.
%
%
%
We now state the main theorem for the regret of \sp.
\begin{customtheorem}{1}
\label{thm:regret-linear-bandit} \textbf{(Regret of \Cref{alg:linear-bandit}, informal)}
Running \Cref{alg:linear-bandit} with budget $n\geq O_{\kappa^2,H^2_U}(\tfrac{d^4 \sigma^4_{\max}\log^2 (A/\delta)}{\sigma^4_{\min}})$, the resulting regret satisfies $\cR_n = O_{\kappa^2,H^2_U}\left(\frac{d^3 \sigma^2_{\max}\log(n )}{\sigma^2_{\min} n^{3/2}}\right)$.
%
%
\end{customtheorem}

\textbf{Discussion of Regret:} \Cref{thm:regret-linear-bandit} states that the regret of \Cref{alg:linear-bandit} scales as $O_{\kappa^2,H^2_U}(d^3 \sigma^2_{\max}\log(n)/n^{3/2})$ where $d$ is the dimension of $\btheta_*$. Note that our regret bound depends on the underlying feature dimension $d$ instead of actions $A$, and scales as $\widetilde{O}_{}(d^3 n^{-3/2})$ which gives a positive answer to the main question of whether such a result is possible.
In comparison to earlier work, when $d^3 < A$, we have a tighter bound than that given by \citet{carpentier2011finite}. 
Furthermore, the results of \citet{carpentier2011finite,carpentier2012minimax,carpentier2015adaptive} are for the standard multi-armed bandit setting and cannot be easily extended to incorporate structure in the linear bandit setting. 
Our new bound also improves upon the A-optimal design method given by \cite{fontaine2021online}, as their regret depends on the number of actions $A$ and scales as $O(\frac{A\log n}{n^{3/2}})$.



\textbf{Proof (Overview) of \Cref{thm:regret-linear-bandit}:} We now outline the key steps for proving \Cref{thm:regret-linear-bandit}.

\textbf{Step 1 (Regret Decomposition):}
We first decompose the regret $\cR_n = \bL_n(\pi,\wb,\wSigma_\Gamma) - \L^*_n(\pi,\bb_*,\bSigma_*)\nonumber$. Recall that $\bb_*\in\triangle(\A)$ is the optimal design in \eqref{eq:opt-oracle-sol} and $\wb\in\triangle(\A)$ is the design followed by \sp.  
%
However, we cannot directly go after the loss $\bL_n(\pi,\wb,\wSigma_\Gamma)$ as it does not admit a simple structure like $\L^*_n(\pi,\bb_*,\bSigma_*)$. Rather we establish an upper bound on the loss $\bL_n(\pi,\wb,\wSigma_\Gamma)$, given by $\L'_{n-\Gamma}(\pi,\wb_*,\wSigma_\Gamma)$ (defined formally in \eqref{eq:loss-upper-main-paper}). 
Consequently, we can decompose the regret $\cR_n$ into three parts as follows:
\begin{align}
\vspace*{-2em}
    \cR_n&\overset{(a)}{\leq} \underbrace{\L'_{n-\Gamma}(\pi,\wb,\wSigma_\Gamma) - \L'_{n-\Gamma}(\pi,\wb_*,\wSigma_\Gamma)}_{\textbf{Approximation error}} \nonumber\\
    &+ \underbrace{\L'_{n-\Gamma}(\pi,\wb_*,\wSigma_\Gamma) -  \L_n(\pi,\bb_*,\wSigma_\Gamma)}_{\textbf{Comparing two different loss}}\nonumber\\
    &+ \underbrace{\L_n(\pi,\bb_*,\wSigma_\Gamma) - \L^*_n(\pi,\bb_*,\bSigma_*)}_{\textbf{Estimation error of $\bSigma_*$}}. \label{eq:regret-decomp}
\end{align}
where $(a)$ follows as we show that 
\begin{align}
    & \bL_n(\pi, \wb, \wSigma_\Gamma) \nonumber = \E\left[\big(\sum_{a=1}^A\bw(a)^{\top}(\wtheta_{n-\Gamma} - \btheta_*)\big)^2\right] \nonumber\\
    &\leq \tfrac{1}{n - \Gamma}\left(1+\tfrac{2C d^2  \log (A / \delta)}{\sigma^2_{\min}\Gamma}\right)\sum_{a,a'}\bw(a)^\top\bA_{\wb_*, \wSigma_\Gamma}^{-1}\bw(a')\nonumber\\
    &\coloneqq \L'_{n-\Gamma}(\pi,\wb_*,\wSigma_\Gamma) \label{eq:loss-upper-main-paper},
\end{align}
where $C>0$ is a constant.
Note that the inequality above is shown in \Cref{prop:loss-bandit-tracker} which we discuss in depth in step 2. Finally note that $\wb_*$ is the empirical \PE design  returned by the approximator after it is supplied with $\wSigma_\Gamma$.

\textbf{Step 2 (Bounding the loss $\bL_n(\pi,\wb,\wSigma_\Gamma)$):} In this step we discuss how to upper bound the agnostic loss $\bL_n(\pi,\wb,\wSigma_\Gamma)$ with $\L'_{n-\Gamma}(\pi,\wb_*,\wSigma_\Gamma)
$ as defined in \eqref{eq:loss-upper-main-paper}.

We first state a concentration lemma that is key to proving this upper bound. 
This lemma is novel for our proof because we estimate the underlying covariance matrix $\bSigma_*$ using OLS estimator for $\Gamma$ rounds. We then use the estimation $\wSigma_\Gamma$ in the WLS estimator.
For our lemma, we first define the variance concentration good event under $\Gamma$ rounds of forced exploration as:
\begin{align}
    \xi^{var}_\delta(\Gamma) \coloneqq \bigg\{\forall a, &\bigg|\bx(a)^{\top}(\wSigma_\Gamma - \bSigma_*)\bx(a)\bigg| \nonumber\\
    &\qquad < \dfrac{2C d^2 \sigma^2_{\max} \log ({A}/{\delta})}{\Gamma}\bigg\} \label{eq:event-xi-delta}
\end{align}
\begin{lemma}
\label{lemma:conc}\textbf{(OLS-WLS Concentration Lemma)}
After $\Gamma$ samples of exploration, we can show that $\Pb\left(\xi^{var}_\delta(\Gamma)\right)\geq  1 -8\delta,$
where $C > 0$ is a constant.
\end{lemma}

\textbf{Proof (Overview) of \Cref{lemma:conc}:} Note that we construct an initial estimate $\wtheta_\Gamma$ of $\btheta_*$ using OLS estimate based on the first $\Gamma$ rounds of data $\{a_t,r_t\}_{t=1}^{\Gamma}.$ Let the feature of $a_t$ be $\bx_t$ and the squared residual $y_{t}:=(\bx_{t}^{\top}\wtheta_\Gamma-r_{t})^{2}.$ Recall that \sp\ estimates $\bSigma_*$ via $
\min_{\bS\in\R^{d\times d}}\sum_{t=1}^{\Gamma}(\left\langle \bx_{t}\bx_{t}^{\top},\bS\right\rangle -y_{t})^{2}$.
Let $\zeta_\Gamma \coloneqq \wtheta_\Gamma-\btheta_*$. Then we can show that $y_{t}=\bx_{t}^{\top}\bSigma_* \bx_{t}+\epsilon_{t}$ and the noise $\epsilon_t$ can be bounded by 
\begin{align*}
    \epsilon_t =  \underbrace{\eta_{t}^{2} - \E[\eta_t^2]}_{\textbf{Part A}} + \underbrace{2\eta_{t}\bx_{t}^{\top}\zeta_\Gamma}_{\textbf{Part B}} + \underbrace{\left(\bx_{t}^{\top}\zeta_\Gamma\right)^{2}}_{\textbf{Part C}}.
\end{align*}
For the part A, observe that $\eta^2_t$ is a sub-exponential random variable as  $\eta_t\sim\SG(0,\bx_{t}^{\top}\bSigma_* \bx_{t})$. Hence we can use sub-exponential concentration inequality from \Cref{conc-lemma-sub-exp} (\Cref{app:prob-tools}) to bound it. For part C first recall that $\zeta_\Gamma \coloneqq \wtheta_\Gamma-\btheta_*$ and we use  \Cref{lemma:least-square-conc} (\Cref{app:prob-tools}) to bound it.
Finally, for part B, we can decompose $2\eta_{t}\bx_{t}^{\top}\zeta_\Gamma \leq 2\eta_t^2 + \tfrac{1}{2}(\bx_{t}^{\top}\zeta_\Gamma)^2$. Then using the same technique for parts A and C we bound the total deviation for part B.
Combining the three parts gives the desired concentration inequality. The proof is in \Cref{app:loss-bandit-tracker}. \hfill$\blacksquare$

\Cref{lemma:conc} directly leads to  \Cref{corollary:additive} (\Cref{app:regret-linear-bandit}) which shows that for $n\geq 16C^2 d^4\log^2 (A/\delta)/ \sigma^4_{\min}$ we have that $\bL_n(\pi, \wb, \wSigma_\Gamma) \leq  \L'_{n-\Gamma}(\pi,\wb,\wSigma_\Gamma)$.
Compared to earlier work,  \citet{fontaine2021online} does not require this approach as the variances of each action lack a common structure. Similarly, this approach differs from the time-dependent variance model of \citet{ zhang2021improved, zhao2022bandit}.
%
%


\textbf{Step 3 (Bounding the approximation error and comparing two different losses):} For the approximation error in \eqref{eq:regret-decomp} we need access to an optimization oracle  that gives $\epsilon$ approximation error (\Cref{assm:oracle-approx}). Then setting $\epsilon=\tfrac{1}{\sqrt{n}}$ we have that the estimation error is upper bounded by $n^{-3/2}$.
For comparing the two different losses in \eqref{eq:regret-decomp}, we use their definition of to bound it as $O_{\kappa^2, H^2_u}(\tfrac{d^2\log(A/\delta)}{n^{3/2}})$ 
as shown in \eqref{eq:comparing-two-loss} in \Cref{app:loss-alg-1}.

\textbf{Step 4 (Bounding Estimation Error):}
Now observe that the third quantity in \eqref{eq:regret-decomp} (estimation error of $\bSigma_*$) contains $\L_n(\pi, \bb_*, \wSigma_\Gamma)$ that depends on the design matrix $\bA^{-1}_{\bb_*,\wSigma_\Gamma}$ which in turn depends on the estimation of $\wSigma_\Gamma$.
Similarly $\L_n(\pi, \bb_*, \bSigma_*)$ in the third quantity depends on the design matrix $\bA^{-1}_{\bb_*,\bSigma_*}$ which in turn depends on the true $\bSigma_*$.
Hence, we now bound the concentration of the loss under $\bA^{-1}_{\bb_*,\wSigma_\Gamma}$ against the design matrix $\bA^{-1}_{\bb_*,\bSigma_*}$ in the following lemma.
\begin{lemma}
\label{lemma:gradient-conc}
\textbf{(Concentration of the design matrix)} Let $\wSigma_\Gamma$ be the empirical estimate of $\bSigma_*$, and  $\bV=\sum_{a,a'}\bw(a)\bw(a')^{\top}$. For any arbitrary proportion $\bb$, with probability at least $(1-\delta)$, we have the following: 

\begin{align*}
    \bigg|\sum_{a,a'} &\bw(a)^{\top}(\bA^{-1}_{\bb_*, \wSigma_\Gamma} - \bA^{-1}_{\bb_*, \bSigma_*})\bw(a')\bigg|\\
    &\qquad\leq \frac{2C B^* d^3 \sigma^2_{\max}\log (A/ \delta)}{\Gamma},
\end{align*}
where $B^*$ is a problem-dependent quantity
and $C>0$ is a universal constant.
\end{lemma}
\textbf{Proof (Overview) of \Cref{lemma:gradient-conc}:} We can upper bound 
    $|\sum_{a,a'} \!\bw(a)^{\top}(\bA^{-1}_{\bb_*, \wSigma_\Gamma}\!-\!\bA^{-1}_{\bb_*, \bSigma_*})\bw(a')| 
    \!\leq\! \|\bu\|\!\underbrace{\left\|\bA_{\bb_*,\bSigma_*} - \bA_{\bb_*,\wSigma_\Gamma}\right\|}_{\Delta}\!\|\mathbf{v}\| $
where, $\|\bu\|=\|\bA^{-1}_{\bb_*, \bSigma_*} \bw\|$ and $\|\mathbf{v}\| = \|\bA^{-1}_{\bb_*, \wSigma_\Gamma}\bw\|$. First, observe that $\|\bu\|$ is a problem-dependent quantity. Then to bound $\Delta$ we use the \Cref{lemma:conc} on the concentration of $\wsigma^2_\Gamma(a)$.
%
%
Finally to bound $\|\mathbf{v}\|$ we need to bound $\wsigma^2_\Gamma(a) \leq \sigma^2(a) + \frac{2C d^2 \sigma^2_{\max}\log (A / \delta)}{\Gamma}$ where $\wsigma^2_\Gamma(a)$ is the empirical variance of $\sigma^2(a)$. Combining everything yields the desired result. The proof is in \Cref{app:regret-linear-bandit} \hfill $\blacksquare$

One of our key technical contributions in \Cref{lemma:gradient-conc} is to show that the difference between the two losses $\L_n(\pi, \bb_*, \wSigma_\Gamma)$, and $\L_n(\pi, \bb_*, \bSigma_*)$ scales with $d^3$ instead of the number of actions $A$. In contrast to prior work, a similar loss concentration in \citet{fontaine2021online} scales with $A$. 
Now using \Cref{lemma:gradient-conc}, setting the exploration factor $\Gamma=\sqrt{n}$, and $\delta=\frac{1}{n}$ we can show that the estimation error is upper bounded by 
$\frac{B^*Cd^3 \sigma^2_{\max}\log (n )}{\sigma^2_{\min} n^{3/2}}  + \frac{d^2}{n^2}\Tr(\sum_{a,a'}\bw(a)\bw(a')^\top)$.
Combining steps 1 -- 4 we have the regret of \sp\ as $ O_{\kappa^2,H^2_U}(\tfrac{B^*d^3 \sigma^2_{\max} \log (n )}{\sigma^2_{\min} n^{3/2}})$. The full proof of \Cref{thm:regret-linear-bandit} is in \Cref{app:speed-regret}. \hfill$\blacksquare$.

%

\subsection{Lower Bound}
\Cref{thm:regret-linear-bandit} upper bounds the regret of our agnostic algorithm \sp\ compared to an oracle algorithm with knowledge of $\bSigma_*$.
To quantify the tightness of our upper bound, we now turn to the question of whether we can lower bound the regret for any behavior policy learning algorithm.
For our final theoretical result, we consider a slightly different notion of regret:
$
\cR'_n \coloneqq \L_n(\pi, \wb, \bSigma_*) - \L_n(\pi, \bb_*, \bSigma_*).
$
This notion of regret captures how sub-optimal the estimated $\wb$ is compared to $\bb_*$, \textit{without} additional error incurred by using an estimate of $\bSigma_*$ in the WLS estimator.
We conjecture that $\cR'_n$ is indeed a lower bound to $\cR_n$ as we have established in \Cref{prop:linear-bandit} that the minimum variance estimator is the WLS estimator using $\bSigma_*$. 
Intuitively, $\L_{n}(\pi,\wb,\bSigma_*)$ is a lower bound to $\overline{\L}_n(\pi, \wb, \wSigma_\Gamma)$ as estimation error will likely increase when using $\wSigma_\Gamma$ in place of $\bSigma_*$ in the WLS estimator.
We leave proving that $\cR'_n$ is a lower bound to $\cR_n$ to future work.

\begin{customtheorem}{2}\!\!\!\textbf{(Lower Bound)}
\label{thm:minimax} 
\!\!\!Let $|\bTheta| \!\!=\!\! 2^d$, $\btheta_*\!\in\!\bTheta$. Then any arbitrary $\delta$-PAC policy following the design $\bb\in\triangle(\A)$  
satisfies $\cR'_{n} \!\!=\!\! \L_{n}(\pi,\bb,\bSigma_*) \!-\! \L_{n}(\pi,\bb_*,\bSigma_*) \!\geq\! \Omega\left(\frac{d^2\lambda_d(\bV)\log({n})}{{n}^{3/2}}\right)$ for the environment specified
in \eqref{eq:minimax-environment}. 
\end{customtheorem}
\textbf{Proof (Overview:)} The proof follows the change of measure argument \citep{lattimore2020bandit} and we follow the proof technique of \citet{huang2017structured, mukherjee2022chernoff}. We reduce the policy evaluation problem to the hypothesis testing setting and state a worst-case environment as in \eqref{eq:minimax-environment}. We then show that the regret of any $\delta$-PAC algorithm against an oracle in this environment must scale as $\Omega(\log n/n^{3/2})$. The proof is given in \Cref{app:lower-regret-bound}. \hfill$\blacksquare$    

From the above result, the upper bound of \sp\ regret $\cR_n$ matches the lower bound of regret $\cR'_n$ in $n$ but suffers an additional factor of $d$. 

%% file: expt.tex
We now conduct numerical experiments to show that \sp\ decreases MSE faster than other baselines. 
These experiments complement our theoretical analysis as they do not have the conditions on budget $n$ required in \Cref{thm:regret-linear-bandit}.
Thus, our experimental analysis will show that the theoretically motivated \sp\ algorithm still provides benefit even outside of the sample regime considered in theory.
%
%
As baselines, we compare against \onp, \ora, \ao \citep{fontaine2021online}, and \go \citep{wan2022safe}. 
The \onp\ algorithm simply runs the target policy to collect data, whereas the \ora\ (as discussed in \Cref{sec:loss-def}) samples according to the optimal $\bb_*.$
%
Of existing optimal design methods, \ao{}, and \go are the closest in relation to our work. 
We experiment with \ao design because this criterion minimizes the average variance of the estimates of the regression coefficients and is most closely aligned with our goal.
%
The work of \citet{wan2022safe} considers data collection under safety constraints using Inverse Propensity Weighting. In our unconstrained policy evaluation setting their approach boils down to just G-optimal design.
Further experimental details are in \Cref{app:addl-expt}.
%
%
\begin{figure}[!ht]
\centering
\begin{tabular}{cc}
\hspace*{-1.2em}\includegraphics[scale = 0.25]{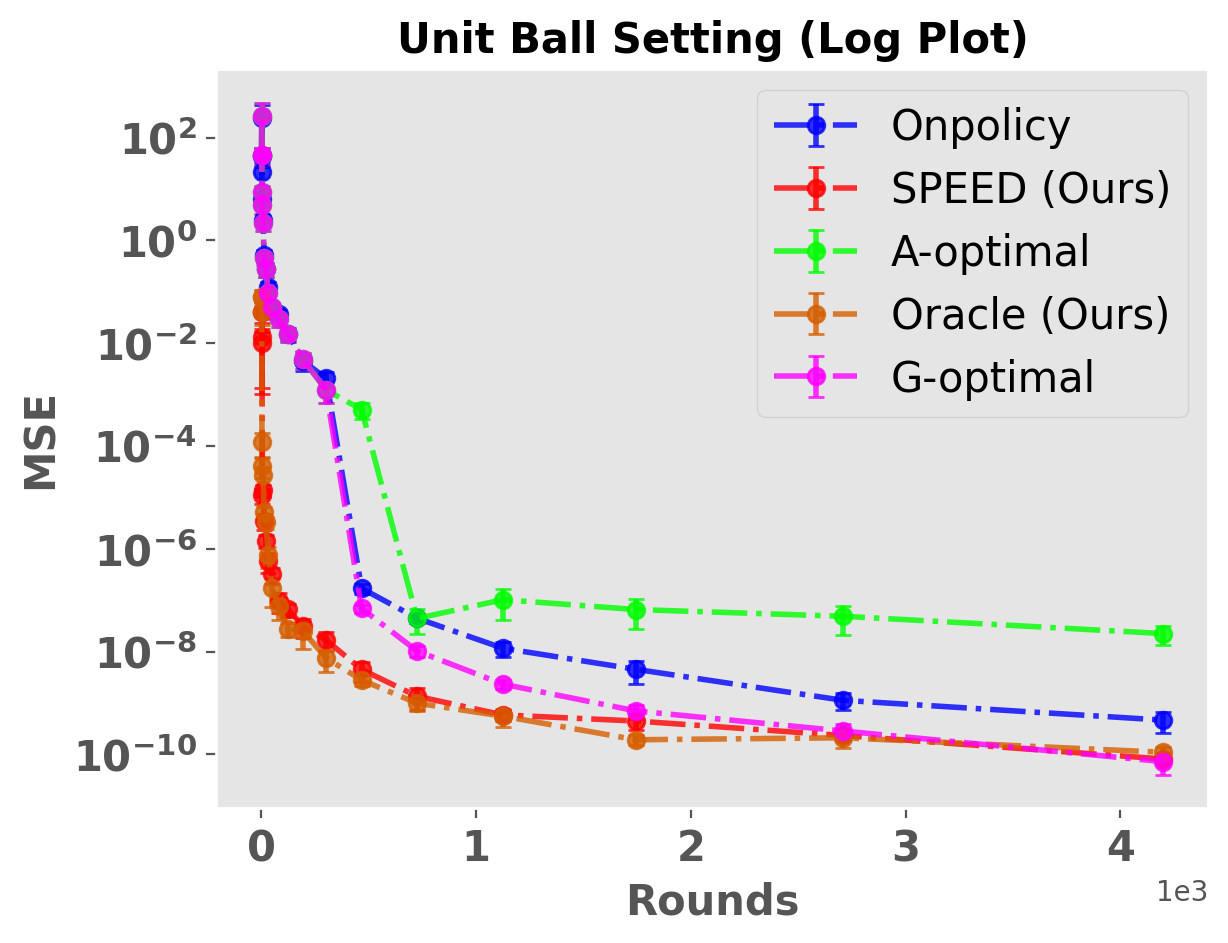} &
\label{fig:expt-linear}\hspace*{-1.2em}\includegraphics[scale = 0.25]{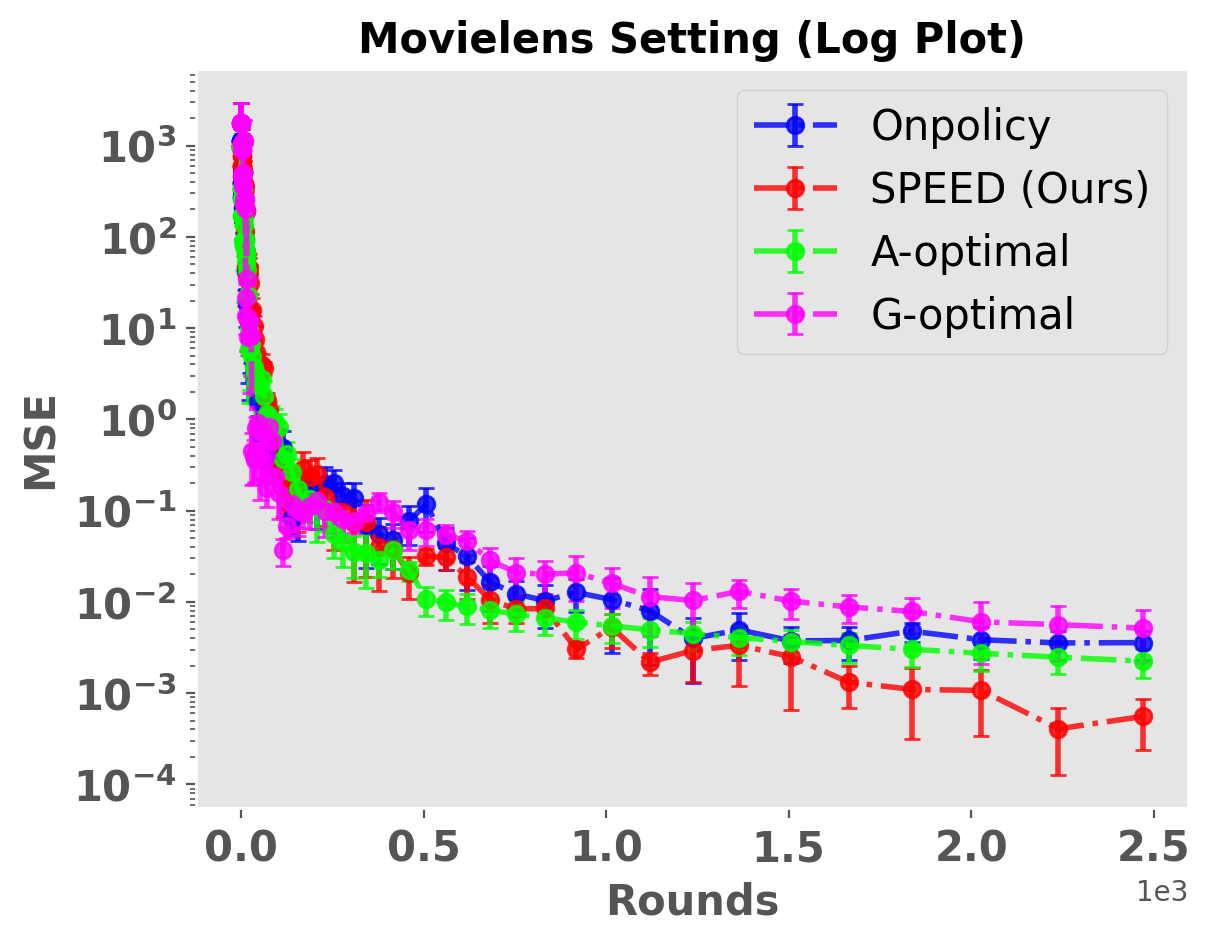} \\
\label{fig:env1}\hspace*{-1.2em}\includegraphics[scale = 0.25]{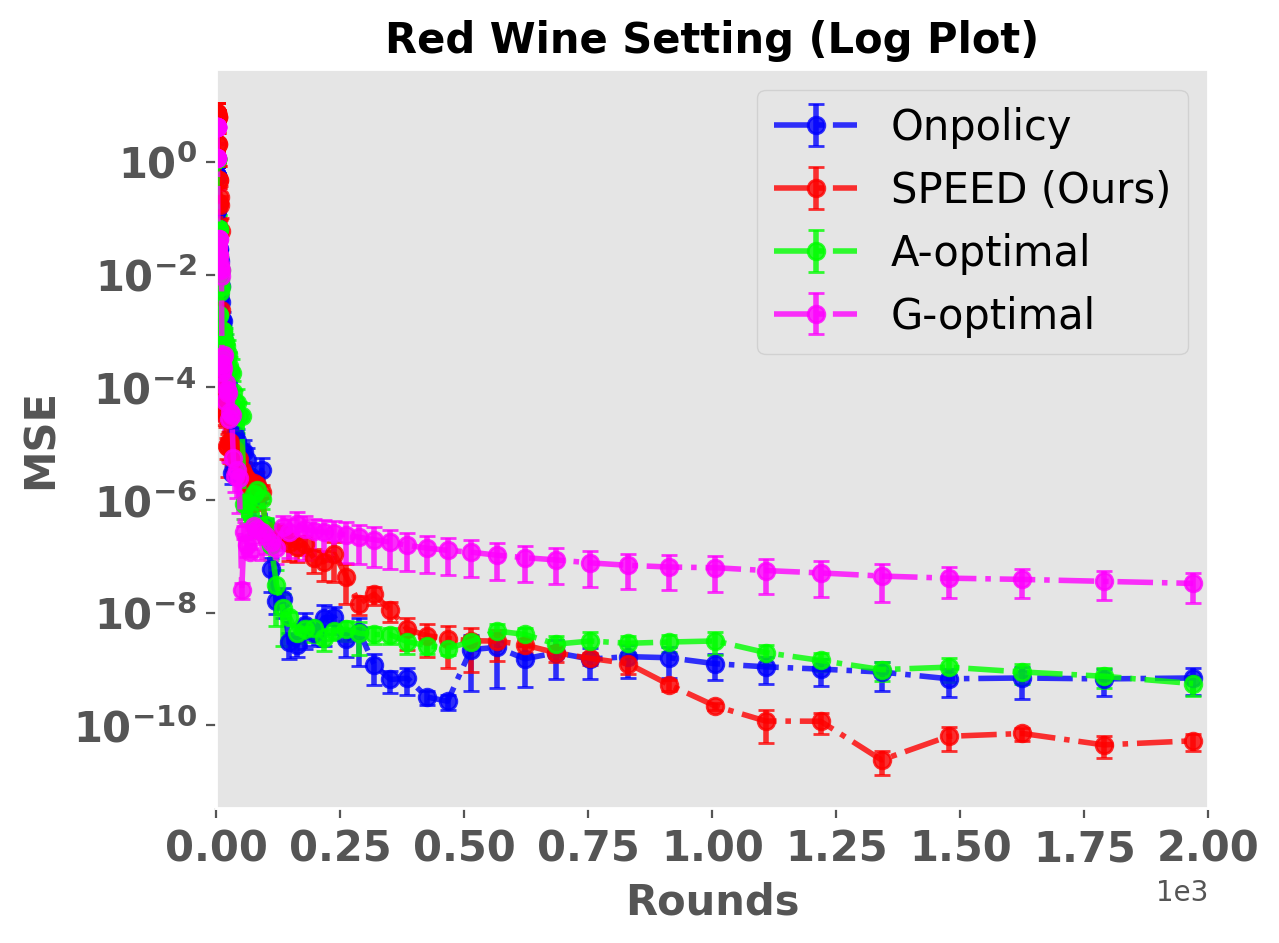} &
\label{fig:expt-linear1}\hspace*{-1.2em}\includegraphics[scale = 0.25]{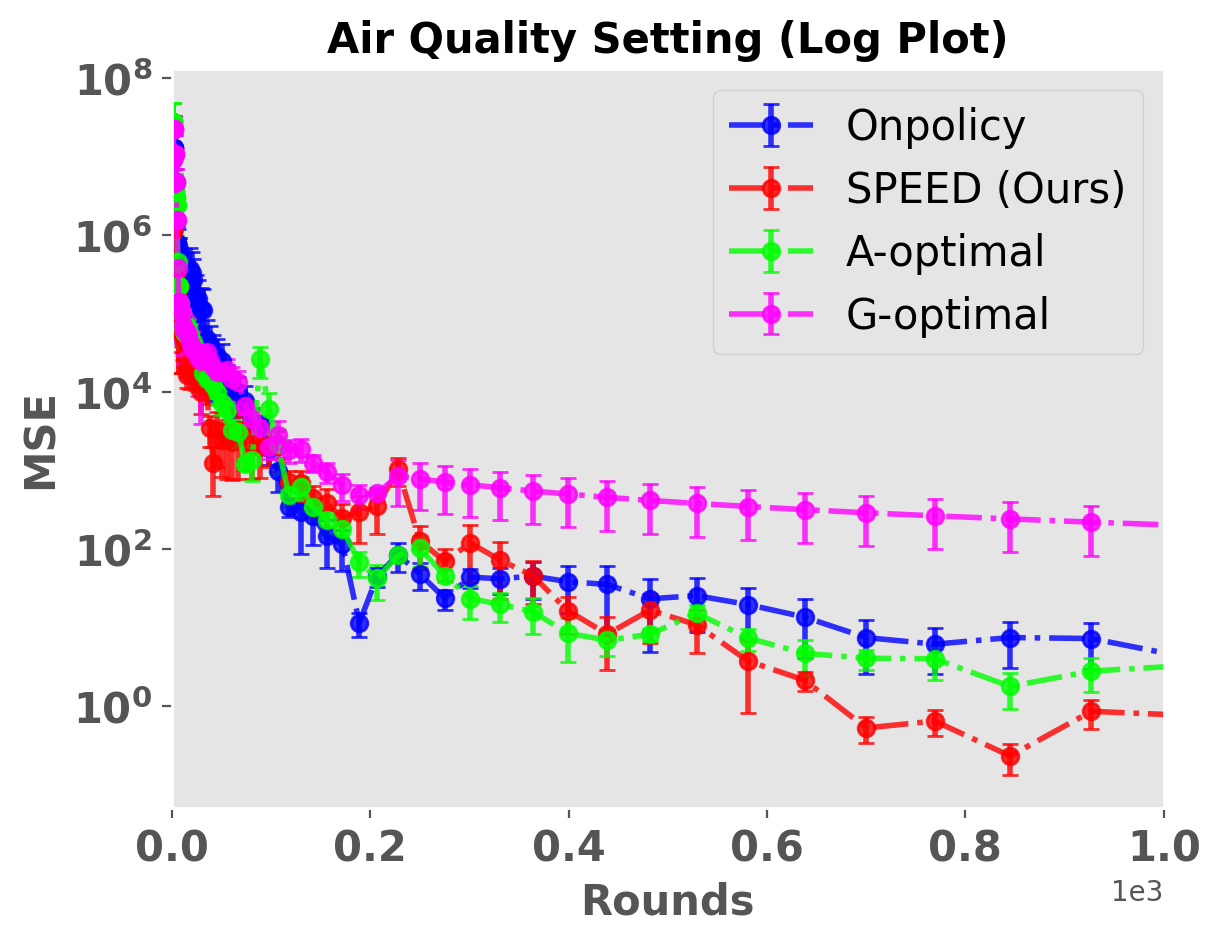} 
\end{tabular}
\vspace{-0.5em}
\caption{(Top-left) MSE plot for the Unit ball. (Top-right) MSE plot for the Movielens dataset. (Bottom-left) MSE plot for Red Wine Quality dataset. (Bottom-right) MSE plot for Air Quality dataset. The vertical axis gives MSE and the horizontal axis is the number of rounds. 
The vertical axis is log-scaled and confidence bars show one standard error. }
\label{fig:linear-expt}
\end{figure}

\textbf{Unit Ball:} We perform this experiment on a set of $5$ actions that are arranged in a unit ball in $\R^2$ to show that \sp\ allocates proportion to the most informative action (weighted by their variance). 
\Cref{fig:linear-expt} (Top Left) shows that \sp\  reduces the MSE faster than \onp, \go, and \ao. 
We also include \ora\ in this setting to show how quickly \sp\ converges to it. However, for settings based on real-life data, we do not have such oracles.

\textbf{Movielens Dataset:}  
Consider a startup that wants to recommend movies to users based on their ratings. 
They have access to a target policy and want to evaluate it on a limited informative dataset before deploying it for full public use. 
We use real-world Movielens 1M dataset \citep{movielens} datasets for this experiment. 
We apply low-rank factorization to the rating matrix to obtain $5$-dimensional representations of users and movies.
We then fit a weighted least square estimate of $\btheta_*$ and $\bSigma_*$.
We generate the reward using this $\btheta_*$ and $\bSigma_*$.
Then we use \sp\ and other baselines to generate the small informative dataset to evaluate the target policy and this experiment is shown in \Cref{fig:linear-expt} (Top Right).
\sp\ initially conducts forced exploration to estimate $\btheta_*$, $\bSigma_*$ and incurs slightly higher MSE but the MSE decreases faster than other baselines as the number of rounds increases.

\textbf{Red Wine Quality:} Consider an online wine company that wants to recommend wines to users and wants to evaluate a target policy before full deployment. 
We perform this  experiment on real-world dataset \textit{Red Wine Quality} from UCI datasets \citep{cortez2009modeling}.
The dataset consists of $1600$ samples (actions) of red wine with each sample $a$ having feature $\bx(a)\in\mathbb{R}^{11}$ and their ratings. We fit a weighted least square estimate to the original dataset and get an estimate of $\btheta_*$ and $\bSigma_*$.
%
%
Then we use \sp\ to generate the informative dataset to evaluate the target policy. 
\Cref{fig:linear-expt3} (Bottom-left) shows  \sp\  outperforming other baselines as horizon increases.

\textbf{Air Quality:} We now consider a setting where a government agency wants to record air quality and notify the public. However, it wants to evaluate a target policy on a limited informative dataset before full deployment. 
We perform this experiment on real-world dataset \textit{Air-Quality} from UCI datasets \citep{de2008field}.
The dataset consists of $1500$ samples (actions) with each sample $a$ having feature $\bx(a)\in\mathbb{R}^{6}$ and their air quality value. Similar to red wine dataset we estimate of $\btheta_*$ and $\bSigma_*$.
%
%
Then we use \sp\ and other baselines (which do not know $\btheta_*$ and $\bSigma_*$) to generate the informative dataset to evaluate the target policy and this experiment is shown in \Cref{fig:linear-expt} (Bottom-right). 
%
%
Observe that \sp's MSE decreases faster than other baselines as the number of rounds increases.

%% file: conclusions.tex
We proposed \sp\ for optimal data collection for policy evaluation in linear bandits with  heteroscedastic reward noise. We formulated a novel optimal design problem, \PE design, for which the optimal behavior policy is the solution that will produce minimal MSE policy evaluation when using a weighted least square estimate of the hidden reward parameters $\btheta_*$ and $\bSigma_*$.
We showed the regret of \sp\ degrades at the rate of $\widetilde{O}(d^3n^{-3/2})$ and matches the lower bound of $\widetilde{O}(d^2n^{-3/2})$ except a factor of $d$.
In contrast the \onp\ suffers a regret of $\widetilde{O}(n^{-1})$ \citep{carpentier2015adaptive}.
We showed empirically that our design outperforms other optimal designs. 
In future work, we intend to extend the result to a more general class of hard problems such as collecting data to minimize the MSE of multiple target policies.

%% file: appendix.tex
\subsection{Related Works and Motivations}
\label{sec:related}

\input{related}

\subsection{Probability Tools}
\label{app:prob-tools}
\begin{lemma}\textbf{\citep{kiefer1960equivalence}}
Assume that $\A \subset \mathbb{R}^{d}$ is compact and $\operatorname{span}(\A)=\mathbb{R}^{d}$. Let $\pi: \A \rightarrow[0,1]$ be a distribution on $\A$ so that $\sum_{a \in \A} \pi(a)=1$ and $\bV(\pi) \in \mathbb{R}^{d \times d}$ and $g(\pi) \in \mathbb{R}$ be given by

\begin{align*}
\bV(\pi)=\sum_{a \in \A} \pi(a) a a^{\top}, \quad g(\pi)=\max _{a \in \A}\|a\|_{\tX(\pi)^{-1}}^{2}
\end{align*}
Then the following are equivalent:
\begin{enumerate}[(a)]
    \item $\pi^{*}$ is a minimizer of $g$.
    \item  $\pi^{*}$ is a maximizer of $f(\pi)=\log \det \bV(\pi)$.
    \item $g\left(\pi^{*}\right)=d$.
\end{enumerate}
Furthermore, there exists a minimizer $\pi^{*}$ of $g$ such that $\left|\operatorname{Supp}\left(\pi^{*}\right)\right| \leq d(d+1) / 2$.
\end{lemma}

\begin{lemma}
\label{conc-lemma-sub-exp}
\textbf{(Sub-Exponential Concentration)}
Suppose that $X$ is sub-exponential with parameters $(\nu, \alpha)$. Then
\begin{align*}
\Pb[X \geq \mu+t] \leq \begin{cases}e^{-\frac{t^{2}}{2 \nu^{2}}} & \text { if } 0 \leq t \leq \frac{\nu^{2}}{\alpha} \\ e^{-\frac{t}{2 \alpha}} & \text { if } t>\frac{\nu^{2}}{\alpha}\end{cases}
\end{align*}
which can be equivalently written as follows:
\begin{align*}
\Pb[X \geq \mu+t] \leq \exp \left\{-\frac{1}{2} \min \left\{\frac{t}{\alpha}, \frac{t^{2}}{\nu^{2}}\right\}\right\}.
\end{align*}
\end{lemma}

\begin{lemma}
\label{lemma:least-square-conc}
\textbf{(Restatement of Theorem 2.2 in \cite{rigollet2015high})} Assume that the linear model holds where the noise $\varepsilon \sim \operatorname{subG}_{n}\left(\sigma^{2}\right)$. Then the least squares estimator $\wtheta_{\Gamma}$ satisfies
\begin{align*}
\mathbb{E}\left[\operatorname{MSE}\left(\bX \wtheta_{\Gamma}\right)\right]=\frac{1}{n} \mathbb{E}\left|\bX \wtheta_{\Gamma}-\bX \theta^{*}\right|_{2}^{2} \lesssim \sigma^{2} \frac{r}{n}
\end{align*}
where $r=\operatorname{rank}\left(\bX^{\top} \bX\right)$. Moreover, for any $\delta>0$, with probability at least $1-\delta$, it holds
\begin{align*}
\operatorname{MSE}\left(\bX \wtheta_{\Gamma}\right) \lesssim \sigma^{2} \frac{r+\log (1 / \delta)}{n}
\end{align*}
\end{lemma}

\subsection{Formulation for \PE Design to Reduce MSE}
\label{app:bandit-prop}
\input{prob_form}

\subsection{Loss is convex}
\label{app:convex-loss}


\begin{customproposition}{2}
\label{prop:convex-loss}
The loss function 
\begin{align*}
    \L_n(\pi, \bb, \bSigma_*) = \frac{1}{n} \left(\sum_{a,a'}\bw(a)^\top\bA_{\bb,\bSigma_*}^{-1}\bw(a')\right) 
\end{align*}
for any arbitrary design proportion $\bb\in\triangle(\A)$ and co-variance matrix $\bSigma_*$ is strictly convex. 
\end{customproposition}

\begin{proof}
Let $\bb, \bb' \in \triangle(\A)$, so that $\bA_{\bb}$ and $\bA_{\bb'}$ are invertible. Recall that we have the loss for a design proportion $\bb$ as

\begin{align*}
    \L_n(\pi, \bb, \bSigma_*) = \frac{1}{n} \left(\sum_{a,a'}\bw(a)^\top\bA_{\bb,\bSigma_*}^{-1}\bw(a')\right) \overset{(a)}{=} \frac{1}{n} \Tr\left(\sum_{a,a'}\bw(a)^\top\bA_{\bb,\bSigma_*}^{-1}\bw(a')\right) &= \frac{1}{n}\Tr\left(\bA_{\bb,\bSigma_*}^{-1}\sum_{a,a'}\bw(a)\bw(a')^\top\right) \\
    &= \frac{1}{n}\Tr\left(\bV\bA_{\bb,\bSigma_*}^{-1}\right)
\end{align*}
where, in $(a)$ we can introduce the trace as the R.H.S. is a scalar quantity,
$\bw(a) = \pi(a)\bx(a)$ and $\bV = \sum_{a,a'}\bw(a)\bw(a')^\top$. Similarly for a $\lambda \in[0,1]$ we have 
\begin{align*}
    \L_n(\pi, \lambda\bb + (1-\lambda)\bb', \bSigma_*) = \frac{1}{n} \Tr\left(\bA^{-1}_{\bb,\bb',\bSigma_*}\sum_{a,a'}\bw(a)\bw(a')^\top\right) = \frac{1}{n} \Tr\left(\bV\bA^{-1}_{\bb,\bb',\bSigma_*}\right).
\end{align*}
Let the matrix $\bA_{\bb,\bb',\bSigma_*}$ be defined as
\begin{align*}
    \bA_{\bb,\bb',\bSigma_*} \coloneqq \lambda\bA_{\bb,\bSigma_*} + (1-\lambda)\bA_{\bb',\bSigma_*}.
\end{align*}
Now observe that 
\begin{align*}
    \bA_{\bb,\bb',\bSigma_*} = \lambda\bA_{\bb,\bSigma_*} + (1-\lambda)\bA_{\bb',\bSigma_*}  = \sum_{a=1}^A\left(\lambda\bb(a) + (1-\lambda)\bb'(a)\right)\tx(a)\tx(a)^\top.
\end{align*}
Also observe that this is a positive semi-definite matrix. Now using Lemma 1 from \citep{whittle1958multivariate} we can show that
\begin{align*}
    \left(\lambda\bA_{\bb,\bSigma_*} + (1-\lambda)\bA_{\bb',\bSigma_*}\right)^{-1} \prec \lambda\bA_{\bb,\bSigma_*}^{-1} + (1-\lambda)\bA_{\bb',\bSigma_*}^{-1}
\end{align*}
for any positive semi-definite matrices $\bA_{\bb}, \bA_{\bb'}$, and $\lambda\in[0,1]$. 
Now taking the trace on both sides we get
\begin{align*}
    \Tr\left(\lambda\bA_{\bb,\bSigma_*} + (1-\lambda)\bA_{\bb',\bSigma_*}\right)^{-1} \prec \Tr\lambda\bA_{\bb,\bSigma_*}^{-1} + \Tr(1-\lambda)\bA_{\bb',\bSigma_*}^{-1}.
\end{align*}
Now using Lemma 2 from \citet{whittle1958multivariate} we can show that
\begin{align*}
    \Tr\left(\lambda\bV\bA_{\bb,\bSigma_*} + (1-\lambda)\bV\bA_{\bb',\bSigma_*}\right)^{-1} \prec \Tr\lambda\bV\bA_{\bb,\bSigma_*}^{-1} + \Tr(1-\lambda)\bV\bA_{\bb',\bSigma_*}^{-1}.
\end{align*}
for any positive semi-definite matrix $\bV$. This implies that
\begin{align*}
    \L_n(\pi, \lambda\bb + (1-\lambda)\bb', \bSigma_*) < \lambda\L_n(\pi, \bb, \bSigma_*) + (1-\lambda)\L_n(\pi, \bb', \bSigma_*).
\end{align*}
Hence, the loss function is convex.
\end{proof}

\begin{remark}\textbf{(Bound on variance)}
\label{remark:bound-variance}
We can use  singular value decomposition of $\bSigma_*$ as $\bSigma_* = \bU \bD \bP^{\top}$ with orthogonal matrices $\bU, \bP^{\top}$ and $\bD=\operatorname{diag}\left(\lambda_{1}, \ldots, \lambda_{d}\right)$ where $\lambda_{i}$ denotes a singular value. Then we can bound $\bx(a)^{\top} \bSigma_* \bx(a)$ as
\begin{align*}
\left\|\bx(a)^{\top} \bSigma_* \bx(a)\right\|=\left\|\bx(a)^{\top} \bU \bD \bP^{\top} \bx(a)\right\|
&\overset{(a)}{=}\left\|\bu^{\top} \bD \mathbf{p}\right\| \leq\left\|\bu^{\top}\right\| \max_{i}|\lambda_{i}| \left\|\mathbf{p} \right\| \\
&\overset{(b)}{=} \| \bx(a)\| \max_{i} |\lambda_{i}| \left\| \bx(a)\right\| 
= \max_{i}|\lambda_{i}|\left\| \bx(a) \right\|^{2}
\end{align*}
where in $(a)$ we have $\bu = \bU^{\top} \bx(a)$,  $\mathbf{p} = \bP^{\top} \bx(a)$ and $(b)$ uses the fact that $\left\|\bU^{\top} \bx(a)\right\|=\|\bx(a)\|$ for any orthogonal matrix $\bU^{\top}$. Similarly we can show that $\left\|\bx(a)^{\top} \bSigma_* \bx(a)\right\|\geq \min_{i}|\lambda_{i}|\left\| \bx(a) \right\|^{2}$. Let $H_L^2\leq \|\bx(a)\|^2\leq H_U^2$ for any $a\in[A]$. 
This implies that
\begin{align*}
    &\underbrace{\min_{i}|\lambda_{i}|H_L^2}_{\sigma^2_{\min}} \leq \min_{i}|\lambda_{i}|\left\| \bx(a) \right\|^{2} \leq \underbrace{\bx(a)^{\top} \bSigma_* \bx(a)}_{\sigma^2(a)} 
    \leq \max_{i}|\lambda_{i}|\left\| \bx(a) \right\|^{2}\leq \underbrace{\max_{i}|\lambda_{i}|H^2_U}_{\sigma^2_{\max}}
\end{align*}
\end{remark}

\subsection{Loss Gradient is Bounded}
\label{app:gradient-loss}

\begin{customproposition}{3}
\label{prop:gradient-loss}
Let $\bb, \bb' \in \triangle(\A)$, so that $\bA_{\bb,\bSigma_*}$ and $\bA_{\bb',\bSigma_*}$ are invertible and define $\bV = \sum_{a,a'}\bw(a)\bw(a')^{\top}$. Then the gradient of the loss function is bounded such that
\begin{align*}
    \|\nabla_{\bb(a)}\L(\pi,\bb,\bSigma_*) - \nabla_{\bb(a)}\L(\pi,\bb',\bSigma_*)\|_2 \leq C_{\kappa}
\end{align*}
where, the 
$$
C_\kappa = \frac{\lambda_d(\bV)H^2_U}{\sigma^2(a)\left(\min_{a'\in\A}\frac{\bb(a')}{\sigma(a')^2} \lambda_{\min }\left(\sum_{a=1}^A \bw(a) \bw(a)^{\top}\right)\right)^2} 
+ \frac{\lambda_1(\bV)H^2_U}{\sigma^2(a)\left(\min_{a'\in\A}\frac{\bb'(a')}{\sigma(a')^2} \lambda_{\min }\left(\sum_{a=1}^A \bw(a) \bw(a)^{\top}\right)\right)^2}.
$$
\end{customproposition}

\begin{proof}
Let $\bb, \bb' \in \triangle(\A)$, so that $\bA_{\bb,\bSigma_*}$ and $\bA_{\bb',\bSigma_*}$ are invertible. 
Observe that the gradient of the loss is given by
\begin{align*}
    \nabla_{\bb(a)}\L(\pi,\bb,\bSigma_*) &= \nabla_{\bb(a)} \Tr\left(\sum_{a,a'}\bw(a)^{\top}\bA^{-1}_{\bb,\bSigma_*}\bw(a')\right) \\
    &\overset{(a)}{\leq} \lambda_1(\bV)\nabla_{\bb(a)}\Tr(\bA^{-1}_{\bb,\bSigma_*}) \\
    &\overset{}{=} -\lambda_1(\bV)\Tr\left(\left(\dfrac{\bw(a)\bw(a)^\top}{\sigma^2(a)} \right)\bA^{-2}_{\bb,\bSigma_*}\right)\\
    &= - \lambda_1(\bV) \dfrac{1}{\sigma^2(a)}\left\|\bA^{-1}_{\bb,\bSigma_*} \bw(a)\right\|^2_2
\end{align*}
where, in $(a)$ we denote $\bV = \sum_{a,a'}\bw(a)\bw(a')^{\top}$. 
Similarly, the gradient of the loss is lower bounded by
\begin{align*}
    &\nabla_{\bb(a)}\L(\pi,\bb,\bSigma_*) \geq - \lambda_d(\bV) \dfrac{1}{\sigma^2(a)}\left\|\bA^{-1}_{\bb,\bSigma_*} \bw(a)\right\|^2_2
\end{align*}
which yields a bound on the gradient difference as
\begin{align*}
    \|\nabla_{\bb(a)}\L(\pi,\bb,\bSigma_*) - \nabla_{\bb'(a)}\L(\pi,\bb',\bSigma_*)\|_2
    &\leq \left\|\lambda_d(\bV) \dfrac{1}{\sigma^2(a)}\left\|\bA^{-1}_{\bb,\bSigma_*} \bw(a)\right\|^2_2 - \lambda_1(\bV) \dfrac{1}{\sigma^2(a)}\left\|\bA^{-1}_{\bb',\bSigma_*} \bw(a)\right\|^2_2\right\|_2\\
    &\leq \left|\lambda_d(\bV) \dfrac{1}{\sigma^2(a)}\left\|\bA^{-1}_{\bb,\bSigma_*} \bw(a)\right\|^2_2\right|+ \left|\lambda_1(\bV) \dfrac{1}{\sigma^2(a)}\left\|\bA^{-1}_{\bb',\bSigma_*} \bw(a)\right\|^2_2\right|.
\end{align*}
So now we focus on the quantity
\begin{align*}
    \left\|\bA^{-1}_{\bb,\bSigma_*} \bw(a)\right\|^2_2\leq \|\bA^{-1}_{\bb,\bSigma_*}\|^2_2\|\bw(a)\|^2_2 \leq \|\bA^{-1}_{\bb,\bSigma_*}\|^2_2 H^2_U.
\end{align*}

Now observe that when $\bb(a)\in\triangle(\A)$ and initialized uniform randomly, then the optimization in \eqref{eq:opt-agnostic-loss} results in a non-singular $\bA^{-1}_{\bb,\bSigma_*}$ if each action has been sampled at least once which is satisfied by \sp. So now we need to bound the minimum eigenvalue of $\bA^{}_{\bb,\bSigma_*}$ denoted as $\lambda_{\min}(\bA^{}_{\bb,\bSigma_*})$. Using Lemma 7 of \citet{fontaine2021online} we have that for all $\bb \in \triangle(\A)$,
\begin{align*}
\min_{a \in[A]} \frac{\bb(a)}{\sigma(a)^2} \sum_{a=1}^A \bw(a) \bw(a)^{\top} \preccurlyeq \sum_{a=1}^A \frac{\bb(a)}{\sigma(a)^2} \bw(a) \bw(a)^{\top} .
\end{align*}
And finally
\begin{align*}
\min_{a \in[A]} \frac{\bb(a)}{\sigma(a)^2} \lambda_{\min }\left(\sum_{a=1}^A \bw(a) \bw(a)^{\top}\right) \leq \lambda_{\min}(\bA^{}_{\bb,\bSigma_*})
\end{align*}
This implies that
\begin{align*}
    \lambda_{\min}(\bA^{-1}_{\bb,\bSigma_*}) \leq \dfrac{1}{\min_{a \in[A]} \frac{\bb(a)}{\sigma(a)^2} \lambda_{\min }\left(\sum_{a=1}^A \bw(a) \bw(a)^{\top}\right)}
\end{align*}
Plugging everything back we get that
\begin{align*}
    \|\nabla_{\bb(a)}\L(\pi,\bb,\bSigma_*) - \nabla_{\bb'(a)}\L(\pi,\bb',\bSigma_*)\|_2 &\leq \dfrac{\lambda_d(\bV)H^2_U}{\sigma^2(a)\left(\min_{a'\in\A}\frac{\bb(a')}{\sigma(a')^2} \lambda_{\min }\left(\sum_{a=1}^A \bw(a) \bw(a)^{\top}\right)\right)^2} \\
    &\qquad + \dfrac{\lambda_1(\bV)H^2_U}{\sigma^2(a)\left(\min_{a'\in\A}\frac{\bb'(a')}{\sigma(a')^2} \lambda_{\min }\left(\sum_{a=1}^A \bw(a) \bw(a)^{\top}\right)\right)^2}.
\end{align*}
The claim of the lemma follows.
\end{proof}


\subsection{Kiefer-Wolfowitz Equivalence}
\label{app:bound-loss}
\input{kiefer-wolfowitz}

\begin{remark}
\label{remark:unbiased-estimator}
Note that the estimator $\wtheta_n$ is an unbiased estimator of $\btheta_*$. Recall that
\begin{align*}
    \wtheta_{n} \coloneqq \arg\min_{\btheta}\sum_{t=1}^{n}\tfrac{1}{\sigma^{2}(a_{t})}(r_t-\bv(a_t)^{\top}\btheta)^{2}
\end{align*}
where, $a_t$ is the action sampled at timestep $t$. 
Define the $\mathbf{diag}(\bSigma_n) = [\sigma^2(a_1), \sigma^2(a_2), \ldots, \sigma^2(a_n)]$, $\bR_n = [r_1, r_2, \ldots, r_n]^{\top} \in \R^{n\times 1}$ be the $n$ rewards observed and $\mathbf{\eta}\in\R^{n\times 1}$ is the noise vector, where $a_1, a_2, \ldots, a_n$ are the actions pulled at time $t=1,2,\ldots,n$.
Then it can be shown that
\begin{align*}
    \E\left[\wtheta_n\right] - \btheta_* &= \E\left[\left(\bX_{n}^{\top}\bSigma_{n}^{-1}\bX_{n}\right)^{-1}\bX_{n}^{\top}\bSigma_{n}^{-1}\bR_{n}\right] - \btheta_*\\
    &= \E\left[\left(\bX_{n}^{\top}\bSigma_{n}^{-1}\bX_{n}\right)^{-1}\bX_{n}^{\top}\bSigma_{n}^{-1}\left(\bX_{n}\btheta_*+\mathbf{\eta}\right)\right] - \btheta_*\\
    &=  \E\left[\left(\bX_{n}^{\top}\bSigma_{n}^{-1}\bX_{n}\right)^{-1}\bX_{n}^{\top}\bSigma_{n}^{-1}\bX_{n}\btheta_*\right]+\E\left[\left(\bX_{n}^{\top}\bSigma_{n}^{-1}\bX_{n}\right)^{-1}\bX_{n}^{\top}\bSigma_{n}^{-1}\mathbf{\eta}\right] - \btheta_*\\
    &= \btheta_* + \left(\bX_{n}^{\top}\bSigma_{n}^{-1}\bX_{n}\right)^{-1}\bX_{n}^{\top}\bSigma_{n}^{-1}\E\left[\mathbf{\eta}\right] - \btheta_* \overset{(a)}{=} 0
\end{align*}
where, $(a)$ follows as noise is zero mean. 
\end{remark}

\section{Bandit Regret Proofs}

\input{bandit_app1}

\section{Regret Lower Bound}
\label{app:lower-regret-bound}
\input{regret_lower_bound.tex}

\section{Additional Experiments}
\label{app:addl-expt}
\input{addl_expt}
\section{Table of Notations}
\label{table-notations}

\begin{table}[!tbh]
    \centering
    \begin{tabular}{|p{10em}|p{33em}|}
        \hline\textbf{Notations} & \textbf{Definition} \\\hline
        $\pi(a)$  & Target policy probability for action $a$ \\\hline
        $\bb(a)$  & Behavior policy probability for action $a$\\\hline
        $\bx(a)$  & Feature of action $a$\\\hline
        $\btheta_*$  & Optimal mean parameter\\\hline
        $\wtheta_n$  & Estimate of $\btheta_*$\\\hline
        $\mu(a) = \bx^\top\btheta_*$  & Mean of action $a$\\\hline
        $\wmu_t(a) = \bx^\top\wtheta_t$  & Empirical mean of action $a$ at time $t$\\\hline
        $R_t(a)$  & Reward for action $a$ at time $t$\\\hline
        $\bSigma_*$  & Optimal co-variance matrix\\\hline
        $\wSigma_t$  & Empirical co-variance matrix at time $t$\\\hline
        $\sigma^2(a) = \bx(a)^\top\bSigma_*\bx(a)$  & Variance of action $a$ \\\hline
        $\wsigma_t^2(a) = \bx(a)^\top\wSigma_t\bx(a)$  & Empirical variance of action $a$ at time $t$ \\\hline
        $n$ & Total budget \\\hline
        $T_n(a)$  & Total Samples of action $a$ after $n$ timesteps\\\hline
    \end{tabular}
    \vspace{1em}
    \caption{Table of Notations}
    \label{tab:my_label}
\end{table}

%% file: related.tex
Our work is most closely related to existing work on data collection for policy evaluation.
Perhaps the most natural choice of behavior policy is to simply run the target policy, i.e., on-policy data collection \citep{sutton2018reinforcement}.
%
%
The works in adaptive Monte Carlo for bandits \citep{oosterhuis2020taking,tucker2022variance-optimal} and MDPs \citep{hanna2017data-efficient,ciosek2017offer,bouchard2016online,zhong2022robust,corrado_on-policy_2023} have shown how to lower the variance of Monte Carlo estimation through the choice of behavior policy. 
In contrast to these works, we consider estimating $v(\pi)$ by estimating the reward distributions rather than using Monte Carlo estimation.
Such \textit{certainty-equivalence} estimators take advantage of the setting's structure and are thus typically of lower variance than Monte Carlo estimators \citep{sutton2018reinforcement}.
The work of \citet{wan2022safe} studies a different estimator for reducing the variance of the importance sampling in constrained MDP setting whereas we study certainty equivalence estimator.
Another set of work has studied sample allocation for stratified Monte Carlo estimators -- a problem that is formally equivalent to behavior policy selection for policy evaluation in the bandit setting with linearly independent arms \citep{antos2008active,carpentier2015adaptive}.
This line of work was recently extended to tabular, tree-structured MDPs by \citet{mukherjee2022revar}.
In contrast, we consider the structured linear bandit setting which incorporates generalization across actions.
\citet{li2024optimal} use A-optimal design to find an optimal behavior policy for the doubly robust estimator. Their focus is different though as they consider tabular MDPs rather than linear heteroscedastic bandits.

Our work is closely related to optimal experimental design and active learning literature. 
We formulate determining the optimal behavior policy in the bandit setting as an optimal design problem.
In contrast to prior work, we introduce a new type of optimality that is tailored to the policy evaluation problem. We are also, to the best of our knowledge, the first to consider both heteroscedastic noise and weighted least squares estimators in formulating our design.
The heteroscedastic noise model and weighted least squares estimator have been considered by \cite{chaudhuri2017active} in the active learning literature and in linear bandit setting by \citet{kirschner2018information} using information directed sampling.
In contrast to these works (and the active learning setting in general), we aim to minimize the weighted error $\sum_{a\in\bA} \pi(a) \bv(a)^\top (\btheta^* - \wtheta)^2$ whereas in the active learning setting the goal is to minimize $\|\btheta^* - \wtheta\|^2$ which results in A-optimal design \citep{fontaine2021online, pukelsheim2006optimal}. Moreover the regret bounds in \citet{fontaine2021online} holds for $d=|\A|$.
\cite{riquelme2017active} extends the results of \cite{carpentier2011finite} to a different linear regression setting than ours but under the homoscedastic noise model. 
%

Data collection for policy evaluation is also related to the problem of exploration for policy learning in MDPs or best-arm identification in bandits.
In those contexts, the aim of exploration is to find the optimal policy and the exploration-exploitation trade-off describes the tension between reducing uncertainty and focusing on known promising actions.
In bandits, the exploration-exploitation trade-off is often navigated under the ``\textbf{O}ptimism in the \textbf{F}ace of \textbf{U}ncertainty" principle using techniques such as UCB \citep{lai1985asymptotically,auer2002finite-time, abbasi2011improved} or Thompson Sampling \citep{thompson1933likelihood, agrawal2012analysis}.
%
%
In contrast to the standard exploration problem, we focus on evaluating a fixed policy.
Instead of balancing exploration and exploitation, a behavior policy for policy evaluation should take actions that reduce uncertainty about $v(\pi)$ with emphasis on actions that have high probability under $\pi$. 
Also, note that heteroscedastic bandits have been studied from the perspective of policy improvement \citep{kirschner2018information,zhao2022bandit} however, in this paper we focus on optimal data collection for policy evaluation.

We note that heteroscedasticity is also studied for the policy improvement setup \citep{kirschner2018information, zhou2022computationally, zhou2021nearly, zhang2021improved, zhao2022bandit}. In these prior works the reward variances are time-dependent as opposed to the quadratic structure studied in this paper.
%
%
Note that policy improvement requires a different approach than policy evaluation. These works build tight confidence sets around the unknown model parameter $\btheta_*$ by employing weighted ridge regression involving an estimated upper bound to the time-dependent variances. However, in our setting, the variances of each action share the unknown low dimensional co-variance matrix $\bSigma_*$. Hence we deviate from these approaches and employ an alternating OLS-WLS estimation to learn the underlying parameter $\bSigma_*$.


%



%% file: prob_form.tex
\begin{customproposition}{1}
Let $\wtheta_n$ be the Weighted Least Square (WLS) estimate \eqref{eq:weighted-least-square} of $\btheta_*$ after observing $n$ samples and define $\bw(a) = \pi(a)\bx(a)$. Define the design matrix as $\bA_{\bb,\bSigma_*}$ (see \eqref{eq:design-matrix}). Then the loss is given by
\begin{align*}
    \E\left[\left(\sum_{a=1}^A\bw(a)^\top(\wtheta_n-\btheta_*)\right)^{2}\right] = \frac{1}{n} \left(\sum_{a,a'}\bw(a)^\top\bA_{\bb,\bSigma_*}^{-1}\bw(a')\right).
\end{align*}
\end{customproposition}

\begin{proof}
Let $T_{n}(a) \geq 0$ be the number of samples of $\bx(a)$, hence $n=\sum_{a=1}^{A} T_{n}(a)$. For each $a \in[A]$, the linear model yields:
\begin{align*}
\dfrac{1}{T_{n}(a)} \sum_{i=1}^{T_{n}(a)} R_{i}^{}(a) = \bx(a)^{\top} \btheta_* + \dfrac{1}{T_{n}(a)} \sum_{i=1}^{T_{n}(a)} \eta_{i}^{}(a) .
\end{align*}
with $R_i(a)$ being the reward observed for action $a$ taken for the $i$-th time, $\eta_i(a)$ being the corresponding noise, and $T_n(a)$ is the number of samples of action $a$.
We define the following:
\begin{align*}
\widetilde{Y}_{n}(a)=\sum_{i=1}^{T_{n}(a)} \dfrac{R_{i}^{}(a)}{\sigma(a) \sqrt{T_{n}(a)}},\qquad \tx_n(a)=\dfrac{\sqrt{T_n(a)} \bx(a)}{ \sigma(a)},\qquad \widetilde{\eta}_{n}(a)=\sum_{i=1}^{T_n(a)} \dfrac{\eta_{i}^{}(a)}{\sigma(a) \sqrt{T_n(a)}}
\end{align*}
so that for all $a \in[A], \widetilde{Y}_{n}(a) = \widetilde{\bx}_{n}(a)^{\top} \btheta_*+\widetilde{\eta}_{n}(a)$  
where we can show the following regarding the expectation of $\widetilde{\eta}_{n}(a)$ as
\begin{align*}
    \E[\widetilde{\eta}_{n}(a)] = \E\left[\sum_{i=1}^{T_n(a)} \dfrac{\eta_{i}(a)}{\sigma(a) \sqrt{T_n(a)}}\right] = \sum_{i=1}^{T_n(a)} \dfrac{\E\left[\eta_{i}(a)\right]}{\sigma(a) \sqrt{T_n(a)}} = 0
\end{align*}
and the variance as
\begin{align*}
    \Var\left[\widetilde{\eta}_{n}(a)\right] = \Var\left[\sum_{i=1}^{T_n(a)} \dfrac{\eta_{i}(a)}{\sigma(a) \sqrt{T_n(a)}}\right] &\overset{(a)}{=} \sum_{i=1}^{T_n(a)} \Var\left[\dfrac{\eta_{i}(a)}{\sigma(a) \sqrt{T_n(a)}}\right] 
    = \sum_{i=1}^{T_n(a)} \dfrac{\Var\left[\eta_{i}(a)\right]}{\sigma^2(a) T_n(a)} =  \dfrac{T_n(a)\sigma^2(a)}{\sigma^2(a) T_n(a)} = 1
\end{align*}
where $(a)$ follows as the noises are independent.
We denote by $\bX=$ $\left(\widetilde{\bx}_{n}(1)^{\top}, \cdots, \widetilde{\bx}_{n}(A)^{\top}\right)^{\top} \in \mathbb{R}^{A \times d}$ the induced design matrix of the policy. Under the assumption that $\bX$ has full rank, the above weighted least squares (WLS) problem has an optimal unbiased estimator $\wtheta_n=\left(\bX^{\top} \bX\right)^{-1} \bX^{\top} \mathbf{Y}$, where $\mathbf{Y} = [\widetilde{Y}_{n}(1), \widetilde{Y}_{n}(2),\ldots, \widetilde{Y}_{n}(A)]^\top$. Let $\mathbf{\eta}  = [\widetilde{\eta}_n(1), \widetilde{\eta}_n(2), \ldots, \widetilde{\eta}_n(A)]^{\top}$. Let $\bw(a)=\pi(a)\bx(a)$. Then the objective is to bound the loss as follows
\begin{align*}
\E&\left[\left(\sum_{a=1}^A\bw(a)^\top\wtheta_n-\sum_{a=1}^A\bw(a)^\top\btheta_*\right)^{2}\right] = \E\left[\left(\sum_{a=1}^A\bw(a)^\top(\wtheta_n-\btheta_*)\right)^{2}\right] \\
&= \E\left[\left(\sum_{a=1}^A\bw(a)^\top\left(\left(\bX^{\top} \bX\right)^{-1} \bX^{\top} \mathbf{Y} - \btheta_*\right)\right)^2\right] 
= \E\left[\left(\sum_{a=1}^A\bw(a)^\top\left(\left(\bX^{\top} \bX\right)^{-1} \bX^{\top} \left(\bX\btheta_* + \mathbf{\eta}\right) - \btheta_*\right)\right)^2\right]\\
&= \E\left[\left(\sum_{a=1}^A\bw(a)^\top\left(\bX^{\top} \bX\right)^{-1} \bX^{\top} \mathbf{\eta}\right)^2\right]
\overset{(a)}{=} \E\left[\Tr\left(\sum_{a=1}^A\bw(a)^\top\left(\bX^{\top} \bX\right)^{-1} \bX^{\top} \mathbf{\eta}\mathbf{\eta}^\top\bX\left(\bX^{\top} \bX\right)^{-1}\sum_{a=1}^{A}\bw(a)\right)\right]\\
&\overset{}{=} \Tr\left(\sum_{a=1}^A\bw(a)^\top\left(\bX^{\top} \bX\right)^{-1} \bX^{\top} \E\left[\mathbf{\eta}\mathbf{\eta}^\top\right]\bX\left(\bX^{\top} \bX\right)^{-1}\sum_{a=1}^{A}\bw(a)\right)\\
&\overset{(b)}{=} \Tr\left(\sum_{a=1}^A\bw(a)^\top\left(\bX^{\top} \bX\right)^{-1} \bX^{\top} \bI\bX\left(\bX^{\top} \bX\right)^{-1}\sum_{a=1}^{A}\bw(a)\right)\\
&=\Tr\left(\sum_{a=1}^A\bw(a)^\top\left(\bX^{\top} \bX\right)^{-1}\sum_{a=1}^{A}\bw(a)\right)=\Tr\left(\sum_{a=1}^A\bw(a)^\top\left(\sum_{a=1}^{A} \tx_{n}(a) \tx_{n}(a)^{\top}\right)^{-1}\sum_{a=1}^{A}\bw(a)\right)\\
&=\frac{1}{n} \Tr\left(\sum_{a=1}^A\bw(a)^\top\left(\sum_{a=1}^{A} \dfrac{\bb(a) \bx(a) \bx(a)^{\top} }{\sigma(a)^{2}}\right)^{-1}\!\!\!\sum_{a=1}^{A}\bw(a)\right)\\
&\overset{(c)}{=} \frac{1}{n} \Tr\left(\!\!\sum_{a=1}^A\bw(a)^\top\left(\sum_{a=1}^{A} \bb(a) \tx(a) \tx(a)^{\top} \right)^{-1}\sum_{a=1}^{A}\!\bw(a)\!\!\right) \\
&= \frac{1}{n} \Tr\left(\sum_{a,a'}\bw(a)^\top\bA_{\bb,\bSigma_*}^{-1}\bw(a')\right) 
\end{align*}
where, in $(a)$ we can introduce the trace operator as for any vector $\bx$ we have $\Tr(\bx^{\top}\bx) = \|\bx\|^2$, $(b)$ follows as the matrix $\E[\mathbf{\eta}\mathbf{\eta}^{\top}]$ has all the non-diagonal element as $0$ (since noises are independent and $\mathbf{Cov}(\widetilde{\epsilon}_n(a),\widetilde{\epsilon}_n(a')) = 0$) and the diagonal element are the  $\Var[\widetilde{\epsilon}_n(a)]=1$, and $(c)$ follows as we redefine $\tx(a) = \bx(a)/\sigma(a)$. 
\end{proof}


%% file: kiefer-wolfowitz.tex


We now introduce a Kiefer-Wolfowitz type equivalence \citep{kiefer1960equivalence} for  the quantity $\Tr(\bA^{-1}_{\bb_*,\bSigma_*})$ for optimal $\bb_*\in\Delta(\A)$ and co-variance matrix $\bSigma_*$ in \Cref{prop:kiefer-wolfowitz}. 

\begin{customproposition}{4}\textbf{(Kiefer-Wolfowitz for \PE)}
\label{prop:kiefer-wolfowitz}
Define the heteroscedastic design matrix as $\bA_{\bb,\bSigma_*} = \sum_{a=1}^A \bb(a)\tx(a)\tx(a)^\top$.  
Assume that $\mathcal{A} \subset \mathbb{R}^{d}$ is compact and $\operatorname{span}(\mathcal{A})=\mathbb{R}^{d}$. Then the following are equivalent:
\begin{enumerate}[(a)]
    \item $\bb_*$ is a minimiser of $\tg(\bb,\bSigma_*) = \Tr\left( \bA_{\bb,\bSigma_*}^{-1}\right)$.
    \item $\bb_*$ is a maximiser of $ f(\bb,\bSigma_*)=\log \operatorname{det}\left( \bA_{\bb,\bSigma_*} \right)$.
    \item $\tg\left(\bb_*,\bSigma_*\right)=d$.
\end{enumerate}
Furthermore, there exists a minimiser $\bb_*$ of $\tg(\bb,\bSigma_*)$ such that $\left|\operatorname{Supp}\left(\bb_*\right)\right| \leq d(d+1) / 2$.
\end{customproposition}

\begin{proof}
We follow the proof technique of \citet{lattimore2020bandit}. 
Let $\bb: \mathcal{A} \rightarrow[0,1]$ be a distribution on $\mathcal{A}$ so that $\sum_{a \in \mathcal{A}} \bb(a)=1$ and $\bA_{\bb,\bSigma_*} \in \mathbb{R}^{d \times d}$ and $g(\bb) \in \mathbb{R}$ be given by
\begin{align*}
\bA_{\bb,\bSigma_*}&=\sum_{a=1}^A \bb(a)\pi^2(a)\sigma^{-2}(a)\ \bx(a) \bx(a)^{\top} = \sum_{a=1}^A \bb(a) \dfrac{\pi(a)\bx(a)}{\sigma(a)} \left(\dfrac{\pi(a)\bx(a)}{\sigma(a)}\right)^{\top} 
\end{align*}
where, $(a)$ follows by setting $\tx(a) = \bx(a)/\sigma(a)$. 
First recall that for a square matrix $\bA$ let adj $(\bA)$ be the transpose of the cofactor matrix of $\bA$. Use the facts that the inverse of a matrix $\bA$ is $\bA^{-1}=\operatorname{adj}(\bA)^{\top} / \operatorname{det}(\bA)$ and that if $\bA: \mathbb{R} \rightarrow \mathbb{R}^{d \times d}$, then
\begin{align*}
\frac{d}{d t} \operatorname{det}(\bA(t))=\Tr\left(\operatorname{adj}(\bA) \frac{d}{d t} \bA(t)\right).
\end{align*}
It follows then that
\begin{align*}
    \nabla f(\bb,\bSigma_*)_{b(a)}&\overset{(a)}{=}\frac{\Tr\left(\operatorname{adj}(\bA_{\bb,\bSigma_*}) \tx(a) \tx(a')^{\top}\right)}{\operatorname{det}(\bA_{\bb,\bSigma_*})}\\
    &=\frac{\tx(a)^{\top} \operatorname{adj}(\bA_{\bb,\bSigma_*}) \tx(a')}{\operatorname{det}(\bA_{\bb,\bSigma_*})}
    \overset{(b)}{=}\tx(a)^{\top} \bA_{\bb,\bSigma_*}^{-1} \tx(a') = \tg(\bb)
\end{align*}
where, in $(a)$ we show the $a$-th component of $f(\bb)$ when we differentiate w.r.t to $\bb(a)$, and $(b)$ follows as $\frac{\operatorname{adj}(\bA_{\bb,\bSigma_*})}{\operatorname{det}(\bA_{\bb,\bSigma_*})} = \bA_{\bb,\bSigma_*}^{-1}$. 
Also observe that
\begin{align}
&\left(\sum_{a=1}^A \bb(a)\|\tx(a)\|^2_{\bA_{\bb,\bSigma_*}^{-1}}\right)=\Tr\left(\sum_{a=1}^A \bb(a) \tx(a) \tx(a')^{\top} \bA_{\bb,\bSigma_*}^{-1}\right) = d.
 \label{eq:I-D-1}
\end{align}
Hence, $\max_\bb\log\det \bA_{\bb, \bSigma_* }$ is  lower bounded by $d$ as in average we have that $\left(\sum_{a=1}^A \bb(a)\|\tx(a)\|^2_{\bA_{\bb,\bSigma_*}^{-1}}\right) = d$.

(b) $\Rightarrow$ (a): Suppose that $\bb_*$ is a maximiser of $f$. By the first-order optimality criterion, for any $\bb$ distribution on $\mathcal{A}$,
\begin{align*}
0 & \geq\left\langle\nabla  f\left(\bb_*,\bSigma_*\right), \bb-\bb_*\right\rangle \\
&\geq\left(\sum_{a=1}^A \bb(a)\|\tx(a)\|^2_{\bA_{\bb_*,\bSigma_*}^{-1}}-\sum_{a=1}^A \bb_*(a)\|\tx(a)\|^2_{\bA_{\bb_*,\bSigma_*}^{-1}}\right) \\
&\geq\left(\sum_{a =1}^A \bb(a)\|\tx(a)\|_{\bA_{\bb_*,\bSigma_*}^{-1}}^{2}-d\right) .
\end{align*}
For an arbitrary $a \in \mathcal{A}$, choosing $\bb$ to be the Dirac at $a \in \mathcal{A}$ proves that $\sum_{a=1}^A\|\tx(a)\|^2_{\bA_{\bb_*,\bSigma_*}^{-1}} \leq d$. 
Since $\tg(\bb) \geq d$ for all $\bb$ by \eqref{eq:I-D-1}, it follows that $\bb_*$ is a minimiser of $\tg$ and that $\min_{\bb} \tg(\bb)=d$.

(c) $\Longrightarrow$ (b): Suppose that $\tg\left(\bb_*\right)=d$. Then, for any $\bb$,
\begin{align*}
\left\langle\nabla f\left(\bb_*,\bSigma_*\right), \bb-\bb_*\right\rangle=\left(\sum_{a=1}^A \bb(a)\|\tx(a)\|^2_{\bA_{\bb_*,\bSigma_*}^{-1}}-d\right) \leq 0 .
\end{align*}
And it follows that $\bb_*$ is a maximiser of $f$ by the first-order optimality conditions and the concavity of $f$. This can be shown as follows:

Let $\bb$ be a Dirac at $a$ and $\bb(t)=\bb_*+t\left(\bb_*-\bb\right)$. Since $\bb_*(a)>0$ it follows for sufficiently small $t>0$ that $\bb(t)$ is a distribution over $\mathcal{A}$. Because $\bb_*$ is a minimiser of $f$,
\begin{align*}
0 \geq\left.\frac{d}{d t} f(\bb(t),\bSigma_*)\right|_{t=0}=\left\langle\nabla f\left(\bb_*,\bSigma_*\right), \bb_*-\bb\right\rangle=d-\sum_{a =1}^A\|\tx(a)\|^2_{\bA_{\bb,\bSigma_*}^{-1}}.
\end{align*}

We now show (a) $\Longrightarrow$ (c). To prove the second part of the theorem, let $\bb_*$ be a minimiser of $\tg$, which by the previous part is a maximiser of $f$. Let $S=\operatorname{Supp}\left(\bb_*\right)$, and suppose that $|S|>d(d+1) / 2$. Since the dimension of the subspace of $d \times d$ symmetric matrices is $d(d+1) / 2$, there must be a non-zero function $v: \mathcal{A} \rightarrow \mathbb{R}$ with $\operatorname{Supp}(v) \subseteq S$ such that
\begin{align}
\sum_{a \in S} v(a) \tx(a) \tx(a)^{\top}=\mathbf{0} \label{eq:equality-1}.
\end{align}
Notice that for any $\tx(a) \in S$, the first-order optimality conditions ensure that $\sum_{a=1}^A\|\tx(a)\|_{\bA_{\bb_*,\bSigma_*}^{-1}}^{2}=d$. Hence
\begin{align*}
d \sum_{a \in S} v(a)=\sum_{a \in S} v(a)\|\tx(a)\|_{\bA_{\bb_*,\bSigma_*}^{-1}}^{2}=0,
\end{align*}
where the last equality follows from \eqref{eq:equality-1}. Let $\bb(t)=\bb_*+t v$ and let $\tau=\max \left\{t>0: \bb(t) \in \mathcal{P}_{\mathcal{A}}\right\}$, which exists since $v \neq 0$ and $\sum_{a \in S} v(a)=0$ and $\operatorname{Supp}(v) \subseteq S$. By \eqref{eq:equality-1}, $\bA_{\bb(t),\bSigma_*}=\bA_{\bb_*,\bSigma_*}$, and hence $ f(\bb(\tau),\bSigma_*)= f\left(\bb_*,\bSigma_*\right)$, which means that $\bb(\tau)$ also maximises $f$. The claim follows by checking that $|\operatorname{Supp}(\bb(T))|<\left|\operatorname{Supp}\left(\bb_*\right)\right|$ and then using induction.
\end{proof}

\begin{corollary}{1}
\label{corollary:kiefer}
From \Cref{prop:kiefer-wolfowitz} we know that $\bb_*$ is a minimizer for $\Tr(\bA^{-1}_{\bb,\bSigma_*})$ and $\Tr(\bA^{-1}_{\bb_*,\bSigma_*}) = d$. This implies that the loss is bounded at $\bb_*$ as $\frac{\lambda_d(\bV) d}{n} \leq \L_n(\pi, \bb_*, \bSigma_*) \leq \frac{\lambda_1(\bV) d}{n}$ where $\bV = \sum_{a,a'}\bw(a)\bw(a')^\top$.
\end{corollary}

\begin{proof}
First recall that we can rewrite the loss for any arbitrary proportion $\bb$ and co-variance $\bSigma_*$ as
\begin{align*}
    \L_n(\pi, \bb, \bSigma_*) = \frac{1}{n} \left(\sum_{a,a'}\bw(a)^\top\bA_{\bb,\bSigma_*}^{-1}\bw(a')\right) = \frac{1}{n} \left(\bA_{\bb,\bSigma_*}^{-1}\sum_{a,a'}\bw(a)\bw(a')^\top\right) = \frac{1}{n} \left(\bA_{\bb,\bSigma_*}^{-1}\bV\right).
\end{align*}
From \citep{fang1994inequalities} we know that for any positive semi-definite matrices $\bA^{-1}_{\bb,\bSigma_*}$ and $\bV$ we have that
\begin{align*}
\lambda_d(\bV) \Tr(\bA^{-1}_{\bb,\bSigma_*}) \leq \Tr(\bV\bA^{-1}_{\bb,\bSigma_*}) \leq \lambda_1(\bV) \Tr(\bA^{-1}_{\bb,\bSigma_*})
\end{align*}
where $\lambda_i(\bV)$ is the $i$ th largest eigenvalue of $\bV$. Now from \Cref{prop:kiefer-wolfowitz} we know that for $\bb_*$ is a minimizer for $\Tr(\bA^{-1}_{\bb,\bSigma_*})$ and $\Tr(\bA^{-1}_{\bb_*,\bSigma_*}) = d$. This implies that the loss is bounded at $\bb_*$ as 
\begin{align*}
    &\lambda_d(\bV) \Tr(\bA^{-1}_{\bb_*,\bSigma_*}) \leq \Tr(\bV\bA^{-1}_{\bb_*,\bSigma_*}) \leq \lambda_1(\bV) \Tr(\bA^{-1}_{\bb_*,\bSigma_*})
    \implies \frac{\lambda_d(\bV) d}{n} \leq \L_n(\pi, \bb_*, \bSigma_*) \leq \frac{\lambda_1(\bV) d}{n}.
\end{align*}
The claim of the corollary follows.
\end{proof}

%% file: bandit_app1.tex
\subsection{Loss of Bandit Oracle}
\label{app:loss-bandit-oracle}

\begin{customproposition}{5}
\textbf{(Bandit Oracle MSE)}
Let the oracle sample each action $a$ for $\lceil n \bb_*(a)\rceil$ times, where $\bb_*$ is the solution to \eqref{eq:opt-oracle-sol}. Define $\lambda_1(\bV)$ as the maximum eigenvalue of $\bV:=\sum_{a,a'}\bw(a)\bw(a')^{\top}$. Then the loss satisfies 
\begin{align*}
    \L^*_n(\pi, \bb_*, \bSigma_*) \leq O_{\kappa^2,H^2_U}\left(\frac{ d\lambda_1(V)
    \log n}{n}\right) + O_{\kappa^2,H^2_U}\left(\frac{1}{n}\right).
\end{align*}
\end{customproposition}

\begin{proof}
Recall the matrix $\bX_n = [\bv_1, \bv_2, \ldots, \bv_n]^{\top} \in \R^{n\times d}$ are the observed features for the $n$ samples taken. Let $\bR_n = [r_1, r_2, \ldots, r_n]^{\top} \in \R^{n\times 1}$ be the $n$ rewards observed and $\mathbf{\eta}\in\R^{n\times 1}$ is the noise vector. Then using weighted least square estimates we have
\begin{align*}
\wtheta_{n} \coloneqq \argmin_{\btheta}\sum_{t=1}^{n}\frac{1}{\sigma^{2}(a_{t})}(r_t-\bv(a_t)^{\top}\btheta)^{2}
\end{align*}
where, in $(a)$ we $a_t$ is the action sampled at timestep $t$. 
Recall that the $\mathbf{diag}(\bSigma_n) = [\sigma^2(a_1), \sigma^2(a_2), \ldots, \sigma^2(a_n)]$, where $a_1, a_2, \ldots, a_n$ are the actions pulled at time $t=1,2,\ldots,n$.
We have that:
\begin{align*}
    \wtheta_n - \btheta_* = (\bX_n^{\top}\bSigma_{n}^{-1}\bX_n)^{-1}\bX_n^{\top}\bSigma_{n}^{-1}\mathbf{\eta}
\end{align*}
where the noise vector $\eta\sim\SG(0,\bSigma_n)$ where $\bSigma_n\in\mathbb{R}^{n\times n}$.
For any $\bz \in \R^d$ we have
\begin{align*}
    \bz^{\top}(\wtheta_n - \btheta_*) = \bz^{\top}(\bX_n^{\top}\bSigma_{n}^{-1}\bX_n)^{-1}\bX_n^{\top}\bSigma_{n}^{-1}\mathbf{\eta}.
\end{align*}
Let $\bb_*$ be the \PE design for $\A$ defined in \eqref{eq:opt-oracle-sol}.
Then the oracle pulls action $a \in \A$ exactly $\left\lceil n \bb_*\right\rceil$ times for some $n>d(d+1)/2$ and computes the least square estimator $\wtheta_n$. Observe that 
$$
\sum_{a=1}^A \bw(a)^\top (\wtheta_n-\btheta_*) \sim \SG\left(0, \sum_{a,a'}\bw(a)^{\top}(\bX_n^{\top}\bSigma^{-1}_n\bX_n)^{-1}\bw(a')\right).
$$ 
So $\left(\sum_{a=1}^A \bw(a)^\top (\wtheta_n-\btheta_*)\right)^2 \sim \SE\left(0,\sum_{a,a'}\bw(a)^{\top}(\bX_n^{\top}\bSigma^{-1}_n\bX_n)^{-1}\bw(a')\right)$ where $\SE$ denotes the sub-exponential distribution. Denote the quantity
\begin{align*}
    t \coloneqq \sqrt{2\sum_{a,a'}\bw(a)^{\top}(\bX_n^{\top}\bSigma^{-1}_n\bX_n)^{-1}\bw(a') \log (1 / \delta)}.
\end{align*}
Now using sub-exponential concentration inequality in \Cref{conc-lemma-sub-exp}, setting 
$$
\nu^2=\sum_{a,a'}\bw(a)^{\top}(\bX_n^{\top}\bSigma^{-1}_n\bX_n)^{-1}\bw(a'),
$$ 
and $\alpha = \nu$, we can show that
\begin{align*}
    \Pb&\left(\left(\sum_{a=1}^A\bw(a)^{\top}(\wtheta_n - \btheta_*)\right)^2 > t \right) \leq \delta, \qquad \text{ if } t \in (0,1]\\
    \Pb&\left(\left(\sum_{a=1}^A\bw(a)^{\top}(\wtheta_n - \btheta_*)\right)^2 > t^2 \right) \leq \delta, \qquad \text{ if } t > 1 .
\end{align*}
Combining the above two we can show that
\begin{align*}
     \Pb&\left(\left(\sum_{a=1}^A\bw(a)^{\top}(\wtheta_n - \btheta_*)\right)^2 > \min\{t, t^2\} \right) \leq \delta, \forall t > 0.
\end{align*}
Further define matrix $\bbSigma_n \in \mathbb{R}^{d\times d}$ as $\bbSigma_n^{-1} \coloneqq (\bX_n^{\top}\bSigma^{-1}_n\bX_n)^{-1}$. 
This means that we have with probability $(1-\delta)$ that
\begin{align*}
    \left(\sum_{a=1}^A\bw(a)^{\top}(\wtheta_n - \btheta_*)\right)^2 
    &\leq \!\!\min\left\{\!\!\sqrt{2\sum_{a,a'} \bw(a)^{\top}\bbSigma_n^{-1} \bw(a') \log (1 / \delta)}, 2\sum_{a,a'} \bw(a)^{\top}\bbSigma_n^{-1} \bw(a') \log (1 / \delta) \right\}\\
&\overset{(a)}{=} \min\bigg\{\sqrt{\frac{2}{n}\sum_{a,a'}\bw(a)^{\top}\bA^{-1}_{\bb_*,\bSigma_*}\bw(a') \log (1 / \delta)},
\frac{2}{n}\sum_{a,a'}\bw(a)^{\top}\bA^{-1}_{\bb_*,\bSigma_*}\bw(a') \log (1 / \delta)\bigg\}\\
&\overset{(b)}{\leq} \min\left\{\sqrt{\frac{8 d\lambda_1(\bV) \log (1 / \delta)}{n}}, \frac{8 d\lambda_1(\bV) \log (1 / \delta)}{n}\right\}
\end{align*}
and we have taken at most $n$ pulls such that $n > \frac{d(d+1)}{2}$ pulls. Here $(a)$ follows as $n\bA^{}_{\bb_*,\bSigma_*} = \bbSigma_n^{}$ and observing that oracle has access to $\bSigma_*$, and optimal proportion $\bb_*$. The $(b)$ follows from applying \Cref{corollary:kiefer} such that $\sum_{a,a'}\bw(a)^{\top}\bA^{-1}_{\bb_*,\bSigma_*}\bw(a') \leq d\lambda_1(\bV)$ where $\bV = \sum_{a,a'}\bw(a)\bw(a')^{\top}$.
Thus, for any $\delta \in(0,1)$ we have 
\begin{align}
\mathbb{P}\left(\left\{\left(\sum_{a=1}^A \tx(a)^{\top} (\wtheta_n-\btheta_*)\right)^2> \min\left\{\sqrt{\frac{8 d\lambda_1(\bV) \log (1 / \delta)}{n}}, \frac{8 d\lambda_1(\bV) \log (1 / \delta)}{n}\right\}\right\}\right) \leq \delta. \label{eq:prob-oracle-loss}
\end{align}
Define the good event $\xi_{\delta}(n)$ as follows:
\begin{align*}
    \xi_\delta(n) \coloneqq \left\{\left(\sum_{a=1}^A \tx(a)^{\top} (\wtheta_{n}-\btheta_*)\right)^2 \leq \min\left\{\sqrt{\frac{8 d\lambda_1(\bV) \log (1 / \delta)}{n}}, \frac{8 d\lambda_1(\bV) \log (1 / \delta)}{n}\right\}\right\}.
\end{align*}

Then the loss of the oracle following \PE $\bb_*$ is given by
\begin{align*}
    \L^*_n(\pi, \bb_*, \bSigma_*) &= \E_{\D}\left[ \left(\sum_{a=1}^A\bw(a)^{\top}\left(\wtheta_n - \btheta_*\right)\right)^2\right]\\
    &\leq \E_{\D}\left[ \left(\sum_{a=1}^A\bw(a)^{\top}\left(\wtheta_n - \btheta_*\right)\right)^2\xi_{\delta}(n)\right] + \E_{\D}\left[ \left(\sum_{a=1}^A\bw(a)^{\top}\left(\wtheta_n - \btheta_*\right)\right)^2\xi^c_{\delta}(n)\right]\\
    &\overset{(a)}{\leq} \E_{\D}\left[ \left(\sum_{a=1}^A\bw(a)^{\top}\left(\wtheta_n - \btheta_*\right)\right)^2\xi_{\delta}(n)\right] + \sum_{t=1}^n A H^2_U \kappa^2\Pb(\xi^c_\delta(n))\\
    &\overset{(b)}{\leq} \min\left\{\sqrt{\frac{8 d\lambda_1(\bV) \log (1 / \delta)}{n}}, \frac{8 d\lambda_1(\bV) \log (1 / \delta)}{n}\right\} + \sum_{t=1}^n AH^2_U \kappa^2\Pb(\xi^c_\delta(n))\\
    &\overset{(c)}{\leq} \min\left\{\sqrt{\frac{16 d\lambda_1(\bV) \log n}{n}}, \frac{16 d
    \lambda_1(\bV)\log n}{n}\right\} + O_{\kappa^2,H^2_U}\left(\frac{1}{n}\right)\\
    &\overset{}{\leq} \frac{48 d\lambda_1(\bV)
    \log n}{n} + O_{\kappa^2,H^2_U}\left(\frac{1}{n}\right)
\end{align*}
where, $(a)$ follows as the noise $\eta^2\leq \kappa^2$ and $\sum_a\|\bx(a)\|^2\leq A H^2_U$ which implies
\begin{align*}
    \E_{\D}\left[ \left(\sum_{a=1}^A\bw(a)^{\top}\left(\wtheta_n - \btheta_*\right)\right)^2\right] \leq n A H_U^2 \kappa^2.
\end{align*}
The $(b)$ follows from \eqref{eq:prob-oracle-loss}, and $(c)$ follows by setting $\delta = 1/n^3$, and noting that $n > A$. 
\end{proof}

\subsection{OLS-WLS Concentration Lemma}
\label{app:loss-bandit-tracker}

\begin{lemma}
\textbf{(Concentration Lemma)}
After $\Gamma$ samples of exploration, we can show that $\Pb\left(\xi^{var}_\delta(\Gamma)\right)\geq  1 -8\delta$
where, $C > 0$ is a constant.
\end{lemma}

\begin{proof}
We observed $(\bx_{t},r_{t})\in\R^{d}\times\R,i=1,\ldots,\Gamma$ from
the model 
\begin{align}
r_{t} & =\bx_{t}^{\top}\btheta_*+\eta_{t},\label{eq:linear_model}\\
\eta_{t} & \sim \SG(0,\bx_{t}^{\top}\bSigma_* \bx_{t}),\label{eq:variance_model}
\end{align}
where $\btheta_*\in\R^{d}$ and $\bSigma_*\in\R^{d\times d}$
are unknown.

Given an initial estimate $\wtheta_\Gamma$ of $\btheta_*$, we first
compute the squared residual $y_{t}:=\left(\bx_{t}^{\top}\wtheta_\Gamma-r_{t}\right)^{2}$,
and then obtain an estimate of $\bSigma_*$ via
\begin{equation}
\min_{\bS\in\R^{d\times d}}\sum_{t=1}^{\Gamma}\left(\left\langle \bx_{t}\bx_{t}^{\top},\bS\right\rangle -y_{t}\right)^{2}.\label{eq:prog}
\end{equation}
Observe that if $\wtheta_\Gamma=\btheta_*$, then the expectation of
the squared residual $y_{t}$ is 
\[
\E\left[y_{t}\right]=\E\left[\left(\bx_{t}^{\top}\btheta_*-r_{t}\right)^{2}\right]=\E\left[\eta_{t}^{2}\right]=\bx_{t}^{\top}\bSigma_* \bx_{t}=\left\langle \bx_{t}\bx_{t}^{\top},\bSigma_*\right\rangle ,
\]
which is a linear function of $\bSigma_*$. The program (\ref{eq:prog})
is thus a least square formulation for estimating $\bSigma_*$.

Let $\bX_{t}:=\bx_{t}\bx_{t}^{\top}$. Below we abuse notation and view
$\bSigma_*,\wSigma_\Gamma,\bX_{t},\bS$ as vectors in $\R^{d^{2}}$ endowed
with the trace inner product $\left\langle \cdot,\cdot\right\rangle $.
Let $\bX\in\R^{\Gamma\times d^{2}}$ have rows $\left\{ \bX_{t}\right\} ,$
and $y=(y_{1},\ldots,y_{\Gamma})^{\top}\in\R^{\Gamma}$. Suppose $\bx_{t}$
can only take on $M$ possible values from $\left\{ \phi_{1},\ldots,\phi_{M}\right\} ,$
so $\bX_{t}\in\left\{ \Phi_{1},\ldots,\Phi_{M}\right\} $, where $\Phi_{m}:=\phi_{m}\phi_{m}^{\top}$. 
Note that for the forced exploration setting we have $M=d < A$. 
Moreover, each value appears exactly $\Gamma/M$ times.  Then (\ref{eq:prog})
can be rewritten as 
\begin{align*}
\min_{\bS\in\R^{d^{2}}}\sum_{m=1}^{M}\sum_{t:\bX_{t}=\Phi_{m}}\left(\left\langle \Phi_{m},\bS\right\rangle -y_{t}\right)^{2} & =\min_{\bS\in\R^{d^{2}}}\sum_{m=1}^{M}\left(\left\langle \Phi_{m},\bS\right\rangle -\frac{1}{\Gamma/M}\sum_{t:\bX_{t}=\Phi_{m}}y_{t}\right)^{2}.
\end{align*}
Let $z_{m}:=\frac{1}{\Gamma/M}\sum_{t:\bX_{t}=\Phi_{m}}y_{t}$. Then it becomes
\[
\min_{\bS\in\R^{d^{2}}}\sum_{m
=1}^{M}\left(\left\langle \Phi_{m},\bS\right\rangle -z_{m}\right)^{2}=\min_{\bS\in\R^{d^{2}}}\left\Vert \Phi \bS-z\right\Vert _{2}^{2},
\]
where $\Phi\in\R^{m\times d^{2}}$ has rows $\left\{ \Phi_{m}\right\} ,$
and $z:=(z_{1},\ldots,z_{m})^{\top}\in\R/M$. Note that $\left\{ \Phi_{m}\right\} $
may or may not span $\R^{d^{2}}$. Observe that  $\wSigma_\Gamma$ be an optimal
solution to the above problem. Then 
\begin{align*}
\left\Vert \Phi(\wSigma_\Gamma-\bSigma_*)\right\Vert _{2}^{2}+\left\Vert \Phi\bSigma_*-z\right\Vert _{2}^{2}+2\left\langle \Phi(\wSigma_\Gamma-\bSigma_*),\Phi\bSigma_*-z\right\rangle 
&=\left\Vert \Phi\wSigma_\Gamma-\Phi\bSigma_*+\Phi\bSigma_*-z\right\Vert _{2}^{2}\\
&=\left\Vert \Phi\wSigma_\Gamma-z\right\Vert _{2}^{2}\le\left\Vert \Phi\bSigma_*-z\right\Vert _{2}^{2}.
\end{align*}
Hence, we can show that
\begin{align*}
\left\Vert \Phi(\wSigma_\Gamma-\bSigma_*)\right\Vert _{2}^{2} & \leq -2\left\langle \Phi(\wSigma_\Gamma-\bSigma_*),\Phi\bSigma_*-z\right\rangle \\
& \overset{(a)}{\leq} 2\left\Vert \Phi(\wSigma_\Gamma-\bSigma_*)\right\Vert _{2}\left\Vert \Phi\bSigma_*-z\right\Vert _{2}.
\end{align*}
where, $(a)$ follows from Cauchy Schwarz inequality. So
\begin{align*}
    \left\Vert \Phi(\wSigma_\Gamma-\bSigma_*)\right\Vert _{2}\le2\left\Vert \Phi\bSigma_*-z\right\Vert _{2}.
\end{align*}
Observe that the RHS does not contain the $\wSigma_\Gamma$ anymore. Note that the $m$-th entry of $\Phi\bSigma_*-z$ is 
\begin{align*}
    \left\langle \Phi_{m},\bSigma_*\right\rangle -z_{m}=\phi_{m}^{\top}\bSigma_*\phi_{m}-\frac{1}{\Gamma/M}\sum_{t:\bX_{t}=\Phi_{m}}y_{t}.
\end{align*}


Let $\zeta_\Gamma \coloneqq \wtheta_\Gamma-\btheta_*$ where $\wtheta_\Gamma$ is the estimation of $\btheta_*$ after $\Gamma = \sqrt{n}$ rounds of exploration. The noise $\eta_t$ is $\sigma_t^2$ sub-Gaussian. Then 
\begin{align*}
y_{t} & =\left(\bx_{t}^{\top}\wtheta_t-r_{t}\right)^{2}\\
 & =(\eta_{t}+\bx_{t}^{\top}\zeta_\Gamma)^{2}\\
 & =\eta_{t}^{2}+2\eta_{t}\bx_{t}^{\top}\zeta_\Gamma+\left(\bx_{t}^{\top}\zeta_\Gamma\right)^{2}\\
& \overset{}{=}\bx_{t}^{\top}\bSigma_* \bx_{t}+\epsilon_{t} = \langle \bSigma_*, \bx_{t}\bx_{t}^{\top}\rangle + \epsilon_t \overset{(a)}{=} \langle \ttheta_*, \bz_t\rangle + \epsilon_t
\end{align*}
where, in $(a)$ we denote the $\ttheta_*\in\R^{d^2}$ as the vector reshaping $\bSigma_*$ and $\bz_t\in\R^{d^2}$ is the vector reshaping $\bx_{t}\bx_{t}^{\top}$. This shows that the feedback $y_t$ is linear.
Now we need to show that $\epsilon_t$ is sub-exponential.
%
We proceed as follows: We have that
\begin{align*}
    \epsilon_{t}:=y_{t}-\bx_{t}^{\top}\bSigma_* \bx_{t} \overset{}{=} \underbrace{\eta_{t}^{2} - \E[\eta_t^2]}_{\textbf{Part A}} + \underbrace{2\eta_{t}\bx_{t}^{\top}\zeta_\Gamma}_{\textbf{Part B}} + \underbrace{\left(\bx_{t}^{\top}\zeta_\Gamma\right)^{2}}_{\textbf{Part C}}
\end{align*}
The goal is to prove that $\Pb(\epsilon_t > \E[\epsilon_t] +  s) \leq \exp(- s/2\sigma^2_{\max})$ for some $s\in\R$. 


For part A, we know that the $\eta^2_t$ is a sub-exponential random variable with $\eta^2_t\sim\SE(\nu,\alpha)$ where $\nu=4 \sigma^2_t \sqrt{2}, \alpha=4 \sigma^2_t$, and $\sigma^2_t = \bx_t^\top\bSigma_*\bx_t $. This follows from Equation 37 in Appendix B of  \citet{honorio2014tight}. It shows that if $X$ is a centered sub-Gaussian random variable with sub-Gaussian parameter $\sigma^2$ then $X^2$ is sub-exponential with parameters $\nu=4 \sigma^2 \sqrt{2}, \alpha=4 \sigma^2$.
%





From \Cref{conc-lemma-sub-exp} we know that
\begin{align*}
    \Pb\left(\eta_t^2 \geq \E[\eta_t^2] + 8 \sigma_t^2 \log (A / \delta) \right) 
    &\leq \exp\left(-\frac{1}{2}\min\left\{\frac{8 \sigma_{t}^2 \log (A / \delta)}{4 \sigma_t^2}, \frac{64 \sigma_{t}^4 \log^2 (A / \delta)}{32 \sigma^4_t}\right\}\right) = \exp\left(-\log(A/\delta)\right).
\end{align*}
Hence, $\eta_t^2 \leq 4\sigma^2_t\sqrt{2} + 8 \sigma_t^2 \log (A / \delta)\leq 16\sigma_{\max}^2 \log (A / \delta)$ with probability $(1-\delta)$ as $\E[\eta_t^2] = \nu$. Equivalently we can write that
\begin{align}
    \Pb\left(\eta_t^2 \geq \E[\eta^2_t] +  16 \sigma_t^2 \log (A / \delta) \right) 
    \leq \exp\left(-\frac{s^2_1}{16\sigma^2_{\max}}\right) = \exp\left(-\frac{s^2_1}{2c'\sigma^2_{\max}}\right). \label{eq:conc-exp-1}
\end{align}
where $s_1 = \sqrt{2c'\sigma^2_{\max}\log(A/\delta)}$ and $c'>0$ is a constant. 

For part C we proceed as follows:
\begin{align*}
    (\bx_t^\top(\wtheta_\Gamma - \btheta_*))^2 &\leq (\bx^\top\bx)\|\wtheta_\Gamma - \btheta_*\|^2 \leq (\bx^\top\bx)\frac{MSE(\bX(\wtheta_\Gamma - \btheta_*))}{\lambda_{\min}} \\
    &\leq \frac{H^2_U}{\lambda_{\min} \Gamma} \left(8 \log (6) \sigma^2_{\max} r+8 \sigma^2_{\max} \log (1 / \delta)\right) \leq \frac{2c^{''} \sigma^2_{\max} d^2 \log (A / \delta)}{\Gamma}
\end{align*}
The first inequality follows by Cauchy Schwarz and the second by Remark 2.3 of \citet{rigollet2015high}. 
Therefore it follows that 
\begin{align*}
    \Pb\left((\bx_t^\top(\wtheta_\Gamma - \btheta_*))^2 \geq \frac{2c^{''} \sigma^2_{\max} d^2 \log (A / \delta)}{\Gamma}\right) \leq \delta.
\end{align*}
Assuming $\Gamma > d^2$ we can also show that 
\begin{align*}
    \Pb\left((\bx_t^\top(\wtheta_\Gamma - \btheta_*))^2 \geq 2c^{''} \sigma^2_{\max} \log (A / \delta)\right) \leq \delta 
\end{align*}
which drops the dependence on $\Gamma$ and $d$. 
Equivalently we can write that
\begin{align}
    \Pb\left((\bx_t^\top(\wtheta_\Gamma - \btheta_*))^2 \geq 2c^{''} \sigma^2_{\max} \log (A / \delta)\right)
    \leq \exp\left(-\frac{s_3^2}{2c^{''} \sigma^2_{\max}}\right). \label{eq:conc-exp-2}
\end{align}
where $s_3 = \sqrt{2c^{''} \sigma^2_{\max} \log (A / \delta)}$.

For part B we proceed as follows:
\begin{align*}
2 \underbrace{\eta_t}_{a} \underbrace{\bx_t^{\top}\left(\wtheta_\Gamma-\btheta_*\right)}_{b}&\overset{(a)}{\leq} 2\eta_t^2 + \dfrac{1}{2} \left(\bx_t^{\top}\left(\wtheta_\Gamma-\btheta_*\right)\right)^2 
\end{align*}
where, $(a)$ follows as $2ab \leq 2a^2 + \frac{1}{2}b^2$. It follows then that
%
\begin{align*}
    \Pb\left(2 \eta_t\bx_t^{\top}\left(\wtheta_\Gamma-\btheta_*\right) \geq s^2_1 +  s^3_3 \right) &\overset{(a)}{\leq} \Pb\left(2\eta_t^2 + \frac{1}{2}(\bx_t^{\top}\left(\wtheta_\Gamma-\btheta_*\right))^2 > s^2_1+s^2_3\right)\\
    &\leq \Pb\left(2\eta_t^2 >  s^2_1+s^2_3\right) + \Pb\left(\frac{1}{2}(\bx_t^{\top}\left(\wtheta_\Gamma-\btheta_*\right))^2 > s^2_1+s^2_3\right)\\
    &= \Pb\left(\eta_t^2 >  \frac{s^2_1+s^2_3}{2}\right) + \Pb\left((\bx_t^{\top}\left(\wtheta_\Gamma-\btheta_*\right))^2 > 2(s^2_1+s^2_3)\right)\\
    &\overset{(b)}{\leq} \Pb\left(\eta_t^2 >  \frac{s^2_1+s^2_3}{2}\right) + \Pb\left((\bx_t^{\top}\left(\wtheta_\Gamma-\btheta_*\right))^2 > \frac{s^2_1+s^2_3}{2}\right)\\
    &\overset{(c)}{\leq} \Pb\left(\eta_t^2 >  \frac{s^2_1}{2} \right) + \Pb\left((\bx_t^{\top}\left(\wtheta_\Gamma-\btheta_*\right))^2 > \frac{s^2_3}{2}\right)\\
    &\overset{(d)}{\leq} \exp\left(-\frac{s_1^2}{c'\sigma^2_{\max}}\right) + \exp\left(-\frac{s_3^2}{c^{''} \sigma^2_{\max}}\right)
\end{align*}
where, $(a)$ follows as LHS $2\eta_t^2 + \frac{1}{2}(\bx_t^{\top}\left(\wtheta_\Gamma-\btheta_*\right))^2 > 2 \eta_t\bx_t^{\top}\left(\wtheta_\Gamma-\btheta_*\right)$ (that is LHS is larger). The $(b)$ follows as RHS $\frac{s^2_1+s^2_3}{2} < 2(s^2_1+s^2_3)$ and $(c)$ follows as RHS $\frac{s^2_1+s^2_3}{2} < \frac{s^2_1}{2}$ (that is RHS is smaller). The $(d)$ follows from \eqref{eq:conc-exp-1}, and \eqref{eq:conc-exp-2}.


We now estimate the expectation of $\epsilon_t$. Observe that for $\Gamma > d^2$ we have that
\begin{align*}
    \E_{\eta, \zeta}[\epsilon_t] = \E_{\eta, \zeta}\left[\eta_{t}^{2} - \E_{\eta}[\eta_t^2] + 2\eta_{t}\bx_{t}^{\top}\zeta_\Gamma + \left(\bx_{t}^{\top}\zeta_\Gamma\right)^{2}\right] &= \E_{\eta}[\eta_{t}^{2}] - \E_{\eta, \zeta}[\E_\eta[\eta_t^2]] + 2\E_{\eta, \zeta}[\eta_{t}\bx_{t}^{\top}\zeta_\Gamma] + \E_{\zeta}[\left(\bx_{t}^{\top}\zeta_\Gamma\right)^{2}] \\
    &\geq 2\E_{\eta, \zeta}[\eta_{t}\bx_{t}^{\top}\zeta_\Gamma] + \E_{\zeta}[\left(\bx_{t}^{\top}\zeta_\Gamma\right)^{2}] \geq \frac{\sigma^2_{\max}d}{\Gamma}.
\end{align*}
Similarly we can get an upper bound to $\E[\epsilon_t]$ for $\Gamma > d^2$ as follows:
\begin{align*}
    \E_{\eta, \zeta}[\epsilon_t] &= \E_{\eta, \zeta}\left[\eta_{t}^{2} - \E_{\eta}[\eta_t^2] + 2\eta_{t}\bx_{t}^{\top}\zeta_\Gamma + \left(\bx_{t}^{\top}\zeta_\Gamma\right)^{2}\right] \overset{(a)}{\leq} 2\E_{\eta}[\eta_t^2] + \dfrac{1}{2} \E_{\zeta}[\left(\bx_{t}^{\top}\zeta_\Gamma\right)^{2}] + \E_{\zeta}[\left(\bx_{t}^{\top}\zeta_\Gamma\right)^{2}] \\
    &\leq 2\E_{\eta}[\eta_t^2] + \frac{2c^{''} \sigma^2_{\max} d^2 \log (A / \delta)}{\Gamma} \overset{(b)}{\leq} 16 \sigma^2_{\max} + \frac{2c^{''} \sigma^2_{\max} d^2 \log (A / \delta)}{\Gamma}
\end{align*}
where, $(a)$ follows for $2ab \leq 2a^2 + \tfrac{1}{2}b^2$, $(b)$ follows as $\E[\eta_t^2] = \nu \leq 8\sigma^2_{\max}$.

Define $s^2 = (s^2_1 + s_3^2)$.
Then combining Part A, B and C it follows that 
\begin{align}
    \Pb(\epsilon_t \geq s^2) = \Pb(\eta_{t}^{2}+2\eta_{t}\bx_{t}^{\top}\zeta_\Gamma+\left(\bx_{t}^{\top}\zeta_\Gamma\right)^{2} \geq s^2)  &\leq \Pb(\eta_{t}^{2}\geq s^2) +\Pb(2\eta_{t}\bx_{t}^{\top}\zeta_\Gamma\geq s^2) +\Pb(\left(\bx_{t}^{\top}\zeta_\Gamma\right)^{2} \geq s^2)\nonumber\\
    &\overset{(a)}{\leq} \Pb(\eta_{t}^{2}\geq s^2_1) +\Pb(2\eta_{t}\bx_{t}^{\top}\zeta_\Gamma\geq s^2_1+s^2_3) +\Pb(\left(\bx_{t}^{\top}\zeta_\Gamma\right)^{2} \geq s^2_3)\nonumber\\
    &\leq 2\exp\left(-\frac{s^2_1}{c'\sigma^2_{\max}}\right) + 2\exp\left(-\frac{s_3^2}{c^{''} \sigma^2_{\max}}\right) \label{eq:combined-eq-conc}
\end{align}
where $(a)$ follows as RHS $s_1^2 < s^2$, and $s_3^2 < s^2$ (that is the RHS is smaller). Let there be some constant $C>0$ such that $s^2/C < \max\{s^2_1/c', s^2_3/c''\}$. Then it follows that
\begin{align*}
    \Pb(\epsilon_t \geq \E[\epsilon_t] + s^2) \overset{(a)}{\leq} \Pb\left(\epsilon_t \geq \frac{\sigma^2_{\max} d}{\Gamma}+ s^2\right)\leq  \Pb(\epsilon_t \geq  s^2) \leq 4\exp\left(-\frac{s^2}{C\sigma^2_{\max}}\right).
\end{align*}
where, $(a)$ as the RHS $\E[\epsilon_t] \geq \frac{\sigma^2_{\max}d}{\Gamma}$ is smaller.
This shows that $\epsilon_t$ is a sub-exponential random variable using \Cref{conc-lemma-sub-exp}. Then using \Cref{conc-lemma-sub-exp} and $\Gamma > d^2$ we can show that
\begin{align*}
    \Pb\left(\epsilon_t \geq \E[\epsilon_t] + \dfrac{Cd^2 \sigma^2_{\max}\log (A/\delta)}{\Gamma}\right) \leq 4\exp\left(- \dfrac{C d^2 \sigma^2_{\max}\log(A/\delta)}{C\Gamma\sigma^2_{\max}}\right) \leq 4\frac{\delta}{A}.
\end{align*}
This implies that $\epsilon_t \leq \E[\epsilon_t] + \dfrac{Cd^2 \sigma^2_{\max}\log (A/\delta)}{\Gamma} \leq 16\sigma^2_{\max} + \dfrac{2Cd^2 \sigma^2_{\max}\log (A/\delta)}{\Gamma}$ with probability greater than $1-4\frac{\delta}{A}$.

Combining all of the steps above we can show that
\begin{align*}
    \Pb&\left( \left\langle \Phi_{m},\bSigma_*\right\rangle -z_{m} > \frac{d}{\sqrt{n}}\sum_{t:\bX_{t}=\Phi_{m}}\left(16\sigma^2_{\max} + \frac{2C d^2 \sigma^2_{\max} \log (A / \delta)}{\Gamma}\right) \right)\\
    &\overset{(a)}{=} \Pb\left( \left\langle \Phi_{m},\bSigma_*\right\rangle -z_{m} > \left(16\sigma^2_{\max} + \frac{2C d^2 \sigma^2_{\max} \log (A / \delta)}{\Gamma}\right) \right)\\
    &\overset{(b)}{\leq} \Pb\left( \left\langle \Phi_{m},\bSigma_*\right\rangle -z_{m} > \frac{2C d^2 \sigma^2_{\max} \log (A / \delta)}{\Gamma} \right)\leq 4\delta/A,
\end{align*}
where, 
$(a)$ follows by setting $\Gamma=\sqrt{n}$ and $M=d < A$ and noting that the $m$-th row consist of $\sqrt{n}/d$ entries. The $(b)$ follows as the  Hence the above implies that
\begin{align*}
    \Pb\left(\bx(a)^{\top}\wSigma_\Gamma\bx(a) - \bx(a)^{\top}\bSigma_*\bx(a) \geq \frac{2C d^2 \sigma^2_{\max} \log (A / \delta)}{\Gamma}\right)\leq 4\delta/A .
\end{align*}
%
Similarly, we can bound the other tail inequality as 
\begin{align*}
    \Pb\left(\bx(a)^{\top}\wSigma_\Gamma\bx(a) - \bx(a)^{\top}\bSigma_*\bx(a) \leq -\frac{2C d^2 \sigma^2_{\max}\log (A / \delta)}{\Gamma}\right)\leq 4\delta/A .
\end{align*}
Hence we can show by union bounding over all actions $A > d$ that
\begin{align*}
    \Pb\left(\forall a, \left|\bx(a)^{\top}\left(\wSigma_\Gamma - \bSigma_*\right)\bx(a)\right| \geq \frac{2C d^2 \sigma^2_{\max}\log (A / \delta)}{\Gamma}\right)\leq 2A\dfrac{4\delta}{A} = 8\delta .
\end{align*}
The claim of the lemma follows.
\end{proof}

\input{op_conc}

\subsection{Bounding the Loss of \Cref{alg:linear-bandit}}
\label{app:loss-alg-1}
\begin{customproposition}{6}
\label{prop:loss-bandit-tracker}
\textbf{(Loss of \Cref{alg:linear-bandit}, formal)}
Let $\wb^{}$ be the empirical \PE design followed by \Cref{alg:linear-bandit} and it samples each action $a$ as $\lceil n \wb(a)\rceil$ times. Then the MSE of \Cref{alg:linear-bandit} for for $n\geq \frac{2C d^2 \sigma^2_{\max} \log (A / \delta)}{\sigma^2_{\min}\Gamma}$ is given by
\begin{align*}
    \bL_n(\pi, \wb,\wSigma_\Gamma) \leq  
    \underbrace{O_{\kappa^2,H^2_U}\left(\frac{
    d^3\lambda_1(\bV)\log n}{\sigma^2_{\min} n}\right)}_{\substack{\textbf{\PE MSE}\\\textbf{and exploration error}}} + \underbrace{O_{\kappa^2,H^2_U}\left(\frac{d^2\lambda_1(\bV)\log n}{n^{3/2}}\right)}_{\textbf{Approximation error}} + \underbrace{O_{\kappa^2,H^2_U}\left(\frac{1}{n}\right)}_{\textbf{Failure event MSE}}.
\end{align*}
\end{customproposition}

\begin{proof}
Recall that the $\wSigma_\Gamma$ be the empirical co-variance after $\Gamma$ timesteps. Then \Cref{alg:linear-bandit} pulls each action $a \in \A$ exactly $\left\lceil (n-\Gamma) \wb_{}^{}(a)\right\rceil$ times for some $\sqrt{n}>A$ and computes the least squares estimator $\wtheta_n$. Recall that the estimate $\wtheta_n$ only uses the $(n-\Gamma)$ data sampled under $\wb$. 
Also recall we actually use $\wSigma_{\Gamma}$ as input for optimization problem \eqref{eq:opt-oracle-sol}, where $\Gamma=\sqrt{n}$. 
We first define the good event $\xi_{\delta}(n-\Gamma)$ as follows: 
\begin{align*}
    \xi_{\delta}(n-\Gamma) \coloneqq \bigg\{\left(\sum_{a=1}^A \bw(a)^{\top} (\wtheta_{n-\Gamma}-\btheta_*)\right)^2 &\leq \min\bigg\{\sqrt{\frac{(8 d\lambda_1(\bV) + \alpha_0 + \alpha)\log (1 / \delta)}{n-\Gamma}},\\
    &\qquad \frac{(8 d\lambda_1(\bV) + \alpha_0 + \alpha) \log (1 / \delta)}{n-\Gamma}\bigg\}\bigg\}
\end{align*}
where, $\alpha_0$, and $\alpha$ will be defined later. 
Also, define the good variance event as follows:
\begin{align}
    \xi^{var}_\delta(\Gamma) \coloneqq \left\{\forall a, \left|\bx(a)^{\top}\left(\wSigma_\Gamma - \bSigma_*\right)\bx(a)\right| < \frac{2C d^2 \sigma^2_{\max} \log (A / \delta)}{\Gamma}\right\}. \label{eq:good-variance-event}
\end{align}
Then we can bound the loss of the \sp\ as follows:

\begin{align}
    &\bL_n(\pi, \wb,\wSigma_\Gamma) = \E_{\D}\left[ \left(\sum_{a=1}^A\bw(a)^{\top}\left(\wtheta_{n-\Gamma} - \btheta_*\right)\right)^2\right] \nonumber\\
    & = \E_{\D}\left[ \left(\sum_{a=1}^A\bw(a)^{\top}\left(\wtheta_{n-\Gamma} - \btheta_*\right)\right)^2\indic{\xi_{\delta}(n-\Gamma)}\indic{\xi^{var}_{\delta}(\Gamma)}\right] +  \E_{\D}\left[ \left(\sum_{a=1}^A\bw(a)^{\top}\left(\wtheta_{n-\Gamma} - \btheta_*\right)\right)^2\indic{\xi^c_{\delta}(n-\Gamma)}\right] \nonumber\\
    &\qquad + \E_{\D}\left[ \left(\sum_{a=1}^A\bw(a)^{\top}\left(\wtheta_{n-\Gamma} - \btheta_*\right)\right)^2\indic{(\xi^{var}_{\delta}(\Gamma))^c}\right]. 
    \label{eq:loss-decomp}
\end{align}
Now we bound the first term of the \eqref{eq:loss-decomp}. Note that using weighted least square estimates we have
\begin{align*}
\wtheta_{n-\Gamma} \overset{(a)}{=} \wtheta_{n} \coloneqq \argmin_{\btheta}\sum_{t=\Gamma + 1}^{n}\frac{1}{\sigma^{2}(a_{t})}(r_t-\bv(a_t)^{\top}\btheta)^{2}
%
\end{align*}
where, in $(a)$ we $a_t$ is the action sampled at timestep $t$.  
Recall that the  $\mathbf{diag}(\wSigma_\Gamma) = [\wsigma_\Gamma^2(a_1), \wsigma_\Gamma^2(a_2), \ldots, \wsigma_\Gamma^2(a_n)]$, where $a_1, a_2, \ldots, a_{n-\Gamma}$ are the actions pulled at time $t=\Gamma + 1,2,\ldots,n$.
We have that:
\begin{align*}
    \wtheta_{n-\Gamma} &= (\bX_{n-\Gamma}^{\top}\wSigma_{\Gamma}^{-1}\bX_{n-\Gamma})^{-1}\bX_{n-\Gamma}^{\top}\wSigma_{\Gamma}^{-1}\mathbf{R}_n = (\bX_{n-\Gamma}^{\top}\wSigma_{\Gamma}^{-1}\bX_{n-\Gamma})^{-1}\bX_{n-\Gamma}^{\top}\wSigma_{\Gamma}^{-1}(\bX_{n-\Gamma}\btheta_* + \eta)\nonumber\\
    \wtheta_{n-\Gamma} - \btheta_* &= (\bX_{n-\Gamma}^{\top}\wSigma_{\Gamma}^{-1}\bX_{n-\Gamma})^{-1}\bX_{n-\Gamma}^{\top}\wSigma_{\Gamma}^{-1}\mathbf{\eta} 
\end{align*}
where the noise vector $\eta\sim\SG(0,\bSigma_{n-\Gamma})$ where $\mathbf{diag}(\bSigma_n) = [\sigma^2(a_1), \sigma^2(a_2), \ldots, \sigma^2(a_{n-\Gamma})]$.
For any $\bz \coloneqq \sum_a\bw(a) \in \R^d$ we have
\begin{align}
    \bz^{\top}(\wtheta_{n-\Gamma} - \btheta_*) = \bz^{\top}(\bX_{n-\Gamma}^{\top}\wSigma_{\Gamma}^{-1}\bX_{n-\Gamma})^{-1}\bX_{n-\Gamma}^{\top}\wSigma_{\Gamma}^{-1}\mathbf{\eta}.
    \label{eq:thm-loss-0}
\end{align}
%
%
%
%
%
%
It implies from \eqref{eq:thm-loss-0} that
\begin{align}
    \left(\bz^{\top}(\wtheta_{n-\Gamma} - \btheta_*)\right)^2 \sim \SE \left(0, \bz^\top(\bX_{n-\Gamma}^{\top}\wSigma_{\Gamma}^{-1}\bX_{n-\Gamma})^{-1}\bX_{n-\Gamma}^{\top}\wSigma_{\Gamma}^{-1}\E\left[\mathbf{\eta}\mathbf{\eta}^\top\right] \wSigma_{\Gamma}^{-1}\bX_{n-\Gamma}(\bX_{n-\Gamma}^{\top}\wSigma_{\Gamma}^{-1}\bX_{n-\Gamma})^{-1}\bz\right)
    \label{eq:loss-thm-1}
\end{align}
where $\SE$ denotes the sub-exponential distribution. Hence to bound the quantity $\left(\bz^{\top}(\wtheta_{n-\Gamma} - \btheta_*)\right)^2$ we need to bound the variance.
We first begin by rewriting the loss function for $n\geq \frac{2C d^2 \sigma^2_{\max} \log (A / \delta)}{\sigma^2_{\min}\Gamma} $ as follows
\begin{align}
    \E&\left[\left(\bz^{\top}(\wtheta_{n-\Gamma} - \btheta_*)\right)^2\right] = \bz^\top(\bX_{n-\Gamma}^{\top}\wSigma_{\Gamma}^{-1}\bX_{n-\Gamma})^{-1}\bX_{n-\Gamma}^{\top}\wSigma_{\Gamma}^{-1}\E\left[\mathbf{\eta}\mathbf{\eta}^\top\right] \wSigma_{\Gamma}^{-1}\bX_{n-\Gamma}(\bX_{n-\Gamma}^{\top}\wSigma_{\Gamma}^{-1}\bX_{n-\Gamma})^{-1}\bz\nonumber\\
    &\overset{(a)}{=} \bz^\top(\bX_{n-\Gamma}^{\top}\wSigma_{\Gamma}^{-1}\bX_{n-\Gamma})^{-1}\bX_{n-\Gamma}^{\top}\wSigma_{\Gamma}^{-1}\bSigma_n\wSigma_{\Gamma}^{-1}\bX_{n-\Gamma}(\bX_{n-\Gamma}^{\top}\wSigma_{\Gamma}^{-1}\bX_{n-\Gamma})^{-1}\bz\nonumber\\
    &\overset{}{=} \bz^\top(\bX_{n-\Gamma}^{\top}\wSigma_{\Gamma}^{-1}\bX_{n-\Gamma})^{-1}\bX_{n-\Gamma}^{\top}\wSigma_{\Gamma}^{-\frac{1}{2}}\wSigma_{\Gamma}^{-\frac{1}{2}}\bSigma_n\wSigma_{\Gamma}^{-\frac{1}{2}}\wSigma_{\Gamma}^{-\frac{1}{2}}\bX_{n-\Gamma}(\bX_{n-\Gamma}^{\top}\wSigma_{\Gamma}^{-1}\bX_{n-\Gamma})^{-1}\bz\nonumber\\
    &\overset{(b)}{=} \underbrace{\bz^\top(\bX_{n-\Gamma}^{\top}\wSigma_{\Gamma}^{-1}\bX_{n-\Gamma})^{-1}\bX_{n-\Gamma}^{\top}\wSigma_{\Gamma}^{-\frac{1}{2}}}_{\mathbf{m}^\top\in\R^{n-\Gamma}}\wSigma_{\Gamma}^{-\frac{1}{2}}\bSigma_n\wSigma_{\Gamma}^{-\frac{1}{2}}\underbrace{\wSigma_{\Gamma}^{-\frac{1}{2}}\bX_{n-\Gamma}(\bX_{n-\Gamma}^{\top}\wSigma_{\Gamma}^{-1}\bX_{n-\Gamma})^{-1}\bz}_{\mathbf{m}\in\R^{n-\Gamma}}\nonumber\\
    &\overset{(c)}{\leq} \bz^\top(\bX_{n-\Gamma}^{\top}\wSigma_{\Gamma}^{-1}\bX_{n-\Gamma})^{-1}\bX_{n-\Gamma}^{\top}\wSigma_{\Gamma}^{-1/2}\left(\left(1+2C_{\Gamma, \sigma^2_{\min}}(\delta)\right)\bI_n\right)\wSigma_{\Gamma}^{-1/2}\bX_{n-\Gamma}(\bX_{n-\Gamma}^{\top}\wSigma_{\Gamma}^{-1}\bX_{n-\Gamma})^{-1}\bz\nonumber\\
    &\overset{(d)}{=} \left(1+2C_{\Gamma, \sigma^2_{\min}}(\delta)\right)\bz^\top (\bX_{n-\Gamma}^{\top}\wSigma_{\Gamma}^{-1}\bX_{n-\Gamma})^{-1}\bz \label{eq:upper-bound-thm-1}
\end{align}
where, $(a)$ follows as $\E\left[\mathbf{\eta}\mathbf{\eta}^\top\right] = \bSigma_n$, in $(b)$  $\mathbf{m}$ is a vector in $\R^{n-\Gamma}$. The $(c)$ follows by first observing that 
\begin{align*}
    \wSigma_{\Gamma}^{-\frac{1}{2}}\bSigma_n\wSigma_{\Gamma}^{-\frac{1}{2}} \overset{}{=} \wSigma_{\Gamma}^{-1}\bSigma_n \overset{}{=} \mathbf{diag}(\wSigma_{\Gamma}^{-1}\bSigma_n) = \left[\dfrac{\sigma^2(I_1)}{\wsigma^2_\Gamma(I_1)}, \dfrac{\sigma^2(I_2)}{\wsigma^2_\Gamma(I_2)}, \ldots, \dfrac{\sigma^2(I_n)}{\wsigma^2_\Gamma(I_n)}\right].
\end{align*}
Then note that using \Cref{corollary:multiplicative-bound} we have
$$\frac{\sigma^2(I_t)}{\wsigma^2_\Gamma(I_t)} \leq 1 + 2\cdot\underbrace{\frac{2C d^2 \sigma^2_{\max} \log (A / \delta)}{\sigma^2_{\min}\Gamma}}_{\textbf{$\coloneqq C_{\Gamma, \sigma^2_{\min}}(\delta)$}}$$ 
for each $t\in[n]$, and $(d)$ follows as $1+2C_{\Gamma, \sigma^2_{\min}}(\delta)$ is not a random variable.
Let $\wb^{*}$ be the empirical \PE design  returned by the approximator after it is supplied with $\wSigma_\Gamma$. 
Now observe that the quantity of the samples collected (following $\wb^{*}$) after exploration is as follows:
\begin{align*}
   \left(\tX_{n-\Gamma}^{\top}\wSigma_{\Gamma}^{-1}\tX_{n-\Gamma}\right)^{-1} = \left(\sum_a\left\lceil(n-\Gamma)\wb^*(a)\wsigma^{-2}_\Gamma(a)\right\rceil\bw(a)\bw(a)^{\top}\right)^{-1} 
    = \dfrac{1}{n-\Gamma}\bA_{\wb^*,\wSigma_\Gamma}^{-1}.
\end{align*}
Hence we use the loss function
\begin{align}
    \L'_{n-\Gamma}(\pi,\wb,\wSigma_\Gamma) \coloneqq 
    \left(1+2C_{\Gamma, \sigma^2_{\min}}(\delta)\right)\bz^\top (\tX_{n-\Gamma}^{\top}\wSigma_{\Gamma}^{-1}\tX_{n-\Gamma})^{-1}\bz = \frac{\left(1+2C_{\Gamma, \sigma^2_{\min}}(\delta)\right)}{n - \Gamma}\sum_{a,a'}\bw(a)^\top\bA_{\wb^*, \wSigma_\Gamma}^{-1}\bw(a'). \label{eq:actual-obs-loss}
\end{align}
Also recall that we define
\begin{align*}
    \L_n(\pi,\bb_*,\wSigma_\Gamma) = \dfrac{1}{n}\sum_{a,a'}\bw(a)^\top\bA_{\bb_*, \wSigma_\Gamma}^{-1}\bw(a').
\end{align*}

So to minimize the quantity $\E\left[\left(\sum_a\bw(a)^{\top}(\wtheta_{n-\Gamma} - \btheta_*)\right)^2\right]$ is minimizing the quantity $ \frac{\left(1+2C_{\Gamma, \sigma^2_{\min}}(\delta)\right)}{n - \Gamma}\sum_{a,a'}\bw(a)^{\top}\bA^{-1}_{\wb^*, \wSigma_\Gamma}\bw(a')$.
Further recall that we can show that from \Cref{assm:oracle-approx} (approximation oracle) and Kiefer-Wolfowitz theorem in \Cref{corollary:kiefer} that for the proportion $\bb_*$ and any arbitrary positive semi-definite matrix $\wSigma_\Gamma$ the following holds
\begin{align}
    \sum_{a,a'}\bw(a)^{\top}\bA^{-1}_{\bb_*, \wSigma_\Gamma}\bw(a')  = \Tr\left(\sum_{a,a'}\bw(a)^{\top}\bA^{-1}_{\bb_*, \wSigma_\Gamma}\bw(a') \right) &= \Tr\bigg(\bA^{-1}_{\bb_*, \wSigma_\Gamma}\underbrace{\sum_{a,a'}\bw(a)\bw(a')^{\top}}_{\bV} \bigg) \nonumber\\
    &= \Tr\left(\bA^{-1}_{\bb_*, \wSigma_\Gamma}\bV\right) \leq d\lambda_1(\bV).\label{eq:kiefer-bound}
\end{align}
Then we can decompose the loss as follows:
\begin{align}
    \L'_{n-\Gamma}(\pi,\wb,\wSigma_\Gamma) 
    &= \L'_{n-\Gamma}(\pi,\wb,\wSigma_\Gamma) -  \L'_{n-\Gamma}(\pi,\wb^*,\wSigma_\Gamma) + \L'_{n-\Gamma}(\pi,\wb^*,\wSigma_\Gamma) 
    \nonumber\\
    &= \underbrace{\L'_{n-\Gamma}(\pi,\wb,\wSigma_\Gamma) - \L'_{n-\Gamma}(\pi,\wb^*,\wSigma_\Gamma)}_{\textbf{Approximation error}} + \underbrace{\L'_{n-\Gamma}(\pi,\wb^*,\wSigma_\Gamma) -  \L_n(\pi,\bb_*,\wSigma_\Gamma)}_{\textbf{Comparing two diff loss}} + \L_n(\pi,\bb_*,\wSigma_\Gamma). \label{eq:all-parts}
    \end{align}
For the approximation error we need access to an oracle (see \Cref{assm:oracle-approx}) that gives $\epsilon$ approximation error. 
Then setting $\epsilon=\frac{1}{\sqrt{n}}$ we have that
\begin{align}
    \L'_{n-\Gamma}(\pi,\wb,\wSigma_\Gamma) - \L'_{n-\Gamma}(\pi,\wb^*,\wSigma_\Gamma) &= \frac{\left(1+2C_{\Gamma, \sigma^2_{\min}}(\delta)\right)}{n-\Gamma}\underbrace{\Tr\left( \sum_{a,a'}\bw(a)^\top\bA_{\wb, \wSigma_\Gamma}^{-1}\bw(a') -  \sum_{a,a'}\bw(a)^\top\bA_{\wb^*, \wSigma_\Gamma}^{-1}\bw(a')\right)}_{\epsilon} \nonumber\\
    &\overset{(a)}{\leq} 
    O_{\kappa^2,H^2_U}\left(\dfrac{d^2 \sigma^2_{\max} \log(A/\delta)}{n^{3/2}}\right) \label{eq:approx-loss}
\end{align}
where, $(a)$ follows by setting $\Gamma = \sqrt{n}$, $\epsilon = 1/\sqrt{n}$ and $C_{\Gamma, \sigma^2_{\min}}(\delta) = \frac{2C d^2 \sigma^2_{\max} \log (A / \delta)}{\sigma^2_{\min}\Gamma} = \frac{2C d^2 \sigma^2_{\max} \log (A / \delta)}{\sigma^2_{\min}\sqrt{n}}$. 
Let us define $\bK_1 \coloneqq \Tr(\sum_{a,a'}\bw(a)^\top\bA_{\wb^*, \wSigma_\Gamma}^{-1}\bw(a'))$, and $\bK_2\coloneqq \Tr(\bw(a)^\top\bA_{\bb_*, \wSigma_\Gamma}^{-1}\bw(a'))$.
For the second part of comparing the two losses we can show that 

\begin{align}
    \L'_{n-\Gamma}(\pi,\wb^*,\wSigma_\Gamma) &-  \L_{n}(\pi,\bb_*,\wSigma_\Gamma) = \frac{1}{(n - \Gamma)}\Tr\left( \left(1+2C_{\Gamma, \sigma^2_{\min}}(\delta)\right)K_1\right) -  \dfrac{1}{n}K_2 \nonumber\\
    &= \dfrac{(1+2C_{\Gamma, \sigma^2_{\min}}(\delta))\bK_1}{n-\Gamma} - \dfrac{(1+2C_{\Gamma, \sigma^2_{\min}}(\delta))\bK_2}{n-\Gamma} + \dfrac{(1+2C_{\Gamma, \sigma^2_{\min}}(\delta)) \bK_2}{n-\Gamma}-\dfrac{1}{n}\bK_2\nonumber\\
    &= \dfrac{(1+2C_{\Gamma, \sigma^2_{\min}}(\delta))}{n-\Gamma}\left(\bK_1 - \bK_2\right) + \dfrac{2C_{\Gamma, \sigma^2_{\min}}(\delta) \bK_2}{n-\Gamma} + \dfrac{1}{n-\Gamma}\bK_2 -\dfrac{1}{n}\bK_2\nonumber\\
    &\overset{(a)}{=} \frac{\Gamma}{n(n - \Gamma)}\underbrace{\Tr\left(\sum_{a,a'}\bw(a)^\top\bA_{\wb^*, \wSigma_\Gamma}^{-1}\bw(a') -  \sum_{a,a'}\bw(a)^\top\bA_{\bb_*, \wSigma_\Gamma}^{-1}\bw(a')\right)}_{\leq 0} \nonumber\\
    &\qquad + \frac{2C_{\Gamma, \sigma^2_{\min}}(\delta)}{n - \Gamma}\Tr\left(\sum_{a,a'}\bw(a)^\top\bA_{\bb_*, \wSigma_\Gamma}^{-1}\bw(a')\right) + \dfrac{\Gamma}{n(n-\Gamma)}\Tr\left(\sum_{a,a'}\bw(a)^\top\bA_{\bb_*, \wSigma_\Gamma}^{-1}\bw(a')\right)\nonumber\\
    &\overset{(b)}{\leq} O_{\kappa^2,H^2_U}\left(\dfrac{d^3 \sigma^2_{\max} \lambda_1(\bV)\log(A/\delta)}{\sigma^2_{\min} n^{3/2}}\right) 
    \label{eq:comparing-two-loss}
\end{align}
%
where, $(a)$ follows by substituting the definition of $\bK_1$ and $\bK_2$. The $(b)$ follows by setting $\Gamma = \sqrt{n}$,  $C_{\Gamma, \sigma^2_{\min}}(\delta) = \frac{2C d^2 \sigma^2_{\max} \log (A / \delta)}{\sigma^2_{\min}\Gamma} = \frac{2C d^2 \sigma^2_{\max} \log (A / \delta)}{\sigma^2_{\min}\sqrt{n}}$, and $\Tr\left(\sum_{a,a'}\bw(a)^\top\bA_{\bb_*, \wSigma_\Gamma}^{-1}\bw(a')\right) \leq d\lambda_1(\bV)$. 


Now we combine all parts together in \eqref{eq:all-parts} using \eqref{eq:kiefer-bound}, \eqref{eq:approx-loss} and \eqref{eq:comparing-two-loss}. First we define the quantity 
$$
\alpha \coloneqq 2C_{\Gamma, \sigma^2_{\min}}(\delta)\Tr\left(\sum_{a,a'}\bw(a)^\top\bA_{\bb_*, \wSigma_\Gamma}^{-1}\bw(a')\right) + \frac{\Gamma}{n}\Tr\left(\sum_{a,a'}\bw(a)^\top\bA_{\bb_*, \wSigma_\Gamma}^{-1}\bw(a')\right).
$$
It follows then that \eqref{eq:all-parts} can be written as
\begin{align}
    \dfrac{1 + 2 C_{\Gamma, \sigma^2_{\min}}(\delta)}{n - \Gamma}\sum_{a,a'}\bw(a)^{\top}\bA^{-1}_{\wb, \wSigma_\Gamma}\bw(a')  &\leq \underbrace{\dfrac{(1+2C_{\Gamma}(\delta))\epsilon}{(n-\Gamma)}}_{\textbf{Approximation error}} + \dfrac{\alpha}{n-\Gamma} + \dfrac{1}{n}\sum_{a,a'}\bw(a)^{\top}\bA^{-1}_{\bb_*, \wSigma_\Gamma}\bw(a')\nonumber\\
    \implies (1 + 2 C_{\Gamma, \sigma^2_{\min}}(\delta))\sum_{a,a'}\bw(a)^{\top}\bA^{-1}_{\wb, \wSigma_\Gamma}\bw(a')
    &\leq \underbrace{(1+2C_{\Gamma}(\delta))\epsilon}_{\alpha_0} + \alpha + \dfrac{n-\Gamma}{n}\sum_{a,a'}\bw(a)^{\top}\bA^{-1}_{\bb_*, \wSigma_\Gamma}\bw(a')\nonumber\\
    &\overset{(a)}{\leq} \alpha_0 + \alpha + d\lambda_1(\bV) \label{eq:thm-oracle-2}
\end{align}
where, $(a)$ follows from \Cref{assm:oracle-approx},  \Cref{corollary:kiefer} and \eqref{eq:kiefer-bound}. Also observe that from \eqref{eq:loss-thm-1} we have that $\big(\sum_{a=1}^A \bw(a)^\top (\wtheta_n-\btheta_*)\big)^2$ is a sub-exponential random variable. 
Then using the sub-exponential concentration inequality we have with probability at least $1-\delta$

\begin{align*}
\left(\sum_{a=1}^A \bw(a)^\top (\wtheta_{n-\Gamma}-\btheta_*)\right)^2 &\leq \min\bigg\{\sqrt{(1+2C_{\Gamma, \sigma^2_{\min}}(\delta))\sum_{a,a'}\bw(a)\left(\bX_{n-\Gamma}^{\top}\wSigma_{\Gamma}^{-1}\bX_{n-\Gamma}\right)^{-1}\bw(a')2 \log (1 / \delta)},\\
&\qquad (1+2C_{\Gamma}(\delta))\sum_{a,a'}\bw(a)^{\top}\left(\bX_{n-\Gamma}^{\top}\wSigma_{\Gamma}^{-1}\bX_{n-\Gamma}\right)^{-1}\bw(a')2 \log (1 / \delta)\bigg\}\\
&= \min\bigg\{\frac{1}{\sqrt{n-\Gamma}}\sqrt{(1+2C_{\Gamma, \sigma^2_{\min}}(\delta))\sum_{a,a'}\bw(a)^{\top}\bA^{-1}_{\bb,\wSigma_\Gamma}\bw(a') 2 \log (1 / \delta)},\\
&\qquad\frac{(1+2C_{\Gamma, \sigma^2_{\min}}(\delta))}{n-\Gamma}\sum_{a,a'}\bw(a)^{\top}\bA^{-1}_{\bb,\wSigma_\Gamma}\bw(a') 2 \log (1 / \delta)\bigg\}\\
&\overset{(a)}{\leq} \min\left\{\sqrt{\frac{(8 d\lambda_1(\bV) + \alpha_0 + \alpha ) \log (1 / \delta)}{n-\Gamma}}, \frac{(8 d\lambda_1(\bV) + \alpha_0 + \alpha) \log (1 / \delta)}{n-\Gamma}\right\}
\end{align*}
where, $(a)$ follows from \eqref{eq:thm-oracle-2}, 
and we have taken at most $n-\Gamma$ pulls to estimate $\wtheta_n$ after forced exploration and $\sqrt{n} > d$. 
Thus, for any $\delta \in(0,1)$ we have 
\begin{align}
\mathbb{P}\left(\left\{\left(\sum_{a=1}^A \bw(a)^{\top} (\wtheta_n-\btheta_*)\right)^2 > \min\left\{\sqrt{\frac{(8 d\lambda_1(\bV) + \alpha_0 + \alpha) \log (1 / \delta)}{n-\Gamma}}, \frac{(8 d\lambda_1(\bV) + \alpha_0 + \alpha) \log (1 / \delta)}{n-\Gamma}\right\}\right\}\right) \leq \delta. \label{eq:loss-bandit-explore}
\end{align}
This gives us a bound on the first term of \eqref{eq:loss-decomp}.
%
%
%
%
Combining everything in \eqref{eq:loss-decomp} we can bound the loss of the \sp\ as follows:

\begin{align*}
    &\bL_n(\pi, \wb,\wSigma_\Gamma) 
    %
    \leq \E_{\D}\left[ \left(\sum_{a=1}^A\bw(a)^{\top}\left(\wtheta_{n-\Gamma} - \btheta_*\right)\right)^2\indic{\xi_{\delta}(n-\Gamma)}\indic{\xi^{var}_{\delta}(\Gamma)}\right] + \sum_{t=1}^n A H_U^2\eta^2\Pb(\xi^c_\delta(n-\Gamma)) \\
    & \qquad\qquad + \sum_{t=1}^n A H_U^2\eta^2\Pb\left(\left(\xi^{var}_\delta(\Gamma)\right)^c\right)\\
    &\leq \min\left\{\frac{2C d^2 \log (A / \delta)}{\Gamma}, \sqrt{\frac{(8 d\lambda_1(\bV) + \alpha_0 + \alpha) \log (A / \delta)}{n-\Gamma}}, \frac{(8 d\lambda_1(\bV) + \alpha_0 + \alpha) \log (A / \delta)}{n-\Gamma}\right\} \\
    &\qquad + \sum_{t=1}^n A H_U^2\eta^2\Pb(\xi^c_\delta(n-\Gamma)) + \sum_{t=1}^n A H_U^2\eta^2\Pb\left(\left(\xi^{var}_\delta(\Gamma)\right)^c\right)
\end{align*}
\begin{align*}
    &\overset{(a)}{\leq} \min\left\{\frac{8C d^2 \sigma^2_{\max} \log (n A)}{\sqrt{n}}, \sqrt{\frac{48 (d\lambda_1(\bV) + \alpha_0 + \alpha) \log (n A)}{n}}, \frac{48(d\lambda_1(\bV) + \alpha_0 + \alpha) \log (n A)}{n}\right\} + O\left(\frac{1}{n}\right)\\
    &\overset{}{\leq} \frac{48 d^2 \sigma^2_{\max} \lambda_1(\bV)
    \log (n A)}{n} + \frac{48  \alpha 
    \log (n A)}{n} +  \frac{48\alpha_0
    \log (n A)}{n} + O\left(\frac{1}{n}\right)\\
    &\overset{(b)}{\leq} \frac{48 d^2 \sigma^2_{\max} \lambda_1(\bV)
    \log (n A)}{n} + \frac{144 d\lambda_1(\bV) C_{\Gamma, \sigma^2_{\min}}(\delta)  
    \log (n A)}{n} + \frac{48 d\lambda_1(\bV)\Gamma
    \log (n A)}{n^{3/2}} + \frac{48 \epsilon
    \log (n A)}{n^{}} + O\left(\frac{1}{n}\right)
\end{align*}
where $(a)$ follows as \Cref{prop:loss-bandit-tracker} and setting $\delta=1/n^3$ and noting that $\sqrt{n}>d$. The $(b)$ follows by setting $(1+2C_{\Gamma}(\delta))\epsilon$ and the definition of $\alpha$. Recall that for $\Gamma = \sqrt{n}$ we have that $C_{\Gamma, \sigma^2_{\min}}(\delta) = \frac{2C d^2 \sigma^2_{\max} \log (A / \delta)}{\sigma^2_{\min}\Gamma} = \frac{2C d^2 \sigma^2_{\max} \log (A / \delta)}{\sigma^2_{\min}\sqrt{n}}$.
Then setting $\epsilon=1/\sqrt{n}$ we can bound the loss of the  following \PE $\wb$ as
\begin{align*}
    \bL_n(\pi, \wb, \wSigma_{\Gamma}) 
    &\overset{}{\leq}  O_{\kappa^2,H^2_U}\left(\frac{ d^3 \sigma^2_{\max}
    \lambda_1(\bV)\log (n A)}{\sigma^2_{\min} n}\right) + O_{\kappa^2,H^2_U}\left(\frac{
    d^2 \sigma^2_{\max}\lambda_1(\bV)\log (n A)}{n^{3/2}}\right) +  O_{\kappa^2,H^2_U}\left(\frac{1}{n}\right).
\end{align*}
The claim of the proposition follows. 
\end{proof}

\begin{remark}\textbf{(Discussion on loss)}
Observe that from \Cref{prop:loss-bandit-tracker} that the MSE for policy evaluation setting scales as $O(\frac{d^3\log (n)}{n})$. We contrast this result with \citet{chaudhuri2017active} who obtain a bound on the MSE $\E_\D[\|\btheta_* - \wtheta_n\|^2]\leq O(\frac{d\log (n)}{n})$ in a related setting. Note that \citet{chaudhuri2017active} only considers the setting when $\bSigma_*$ is rank $1$. We make no such assumption and get an additional factor of $d$ in our result due to exploration in $d^2$ dimension to estimate $\bSigma_*$. Finally we get the scaling as $d^3$ due to $\sum_{a,a'}\bw(a)^\top\bA_{\bb_*, \wSigma_\Gamma}^{-1}\bw(a')\leq d\lambda_1(\bV)$ from \Cref{corollary:kiefer}. 
Also observe that we estimate $\E_\D[\sum_a\bw(a)^\top(\btheta_* - \wtheta_n)^2]$ as opposed to $\E_\D[\|\btheta_* - \wtheta_n\|^2]$ in \citet{chaudhuri2017active}.
\end{remark}

\subsection{Regret of \Cref{alg:linear-bandit}}
\label{app:regret-linear-bandit}
\input{theorem1}

%% file: op_conc.tex
\begin{lemma}
\label{lemma:conc-operator}\textbf{(Operator Norm Concentration Lemma)} We have that
\begin{align*}
    \Pb\left(\|\wSigma_\Gamma - \bSigma_*\| \geq \frac{2C d^2 \sigma^2_{\max} \lambda^{-1}_{\min}(\bY)  \log (A / \delta)}{\Gamma}\right) \leq 8\delta
\end{align*}
for a constant $C>0$. 
\end{lemma}

\begin{proof}
Define the set of actions $\Z$ such that it has a span over $\bX$ and $\bX\bX^{\top}$.  Define the vector $\by(a) = \bx(a)\bx(a)^{\top} \in \R^{d^2}$. Also observe that $|\Z| = d^2$. Without loss of generosity, we assume that $\Z=\{1,2,\dots,d^2\}$. Now define the matrix $\bY\in\R^{d^2\times d^2}$ such that
\begin{align*}
    \bY = [\by(1), \by(2), \ldots, \by(|\Z|)]
\end{align*}
We further assume that the $\lambda_{\min}(\bY) > 0$. We already have from \Cref{lemma:conc} that 
\begin{align*}
    & \Pb\left( \forall a \in \A, \left|\bx(a)^{\top}(\wSigma_\Gamma - \bSigma_*)\bx(a)\right| \leq \frac{2C d^2 \sigma^2_{\max} \log (A / \delta)}{\Gamma}\right)\geq 1 - 8\delta\\
    \overset{(a)}{\implies} & \Pb\left( \forall a\in \Z, \left|\langle \wSigma_\Gamma, \by(a) \rangle - \langle \bSigma_*, \by(a) \rangle \right| \leq \frac{2C d^2 \sigma^2_{\max} \log (A / \delta)}{\Gamma}\right)\geq 1 - 8\delta.
\end{align*}
where, $(a)$ follows by the fact that $\Z \subset \A.$ 
Now take an arbitrary vector $\bx$ in unit ball such that $\|\bx\|_2 \leq 1$. Now we define the vector $\by = \bx\bx^{\top}$ such that $\by\in\R^{d^2}$. Then following \Cref{assm:target-dist} we have that
\begin{align*}
    \bx\bx^{\top} = \by = \sum_{a\in\Z} \alpha(a)\by(a) = \alpha\bY \overset{(a)}{\implies} \alpha = \bY^{-1}\by
\end{align*}
where, in $(a)$ we can take the inverse because $\lambda_{\min}(\bY) > 0$. Now we want to bound
\begin{align*}
    \|\wSigma_\Gamma-\bSigma_*\| = \left|\bx^\top\left(\wSigma_\Gamma-\bSigma_*\right)  \bx\right| = \left|\langle \wSigma_\Gamma - \bSigma_*, \by \rangle\right| &\leq \frac{2C d^2 \sigma^2_{\max} \log (A / \delta)}{\Gamma} \left\|\underbrace{\sum_{a}\alpha(a)}_{\alpha}\right\|\\
    &= \frac{2C d^2 \sigma^2_{\max} \log (A / \delta)}{\Gamma} \|\bY^{-1}\by\|\\
    &\leq \frac{2C d^2 \sigma^2_{\max} \log (A / \delta)}{\Gamma} \|\bY^{-1}\|\|\bx\|^2\\
    &\leq \frac{2C d^2 \sigma^2_{\max} \lambda^{-1}_{\min}(\bY)  \log (A / \delta)}{\Gamma}.
\end{align*}
The claim of the lemma follows.
\end{proof}

\begin{corollary}
\label{corollary:multiplicative-bound}
For, $n\geq 4C^2 d^2 \sigma^2_{\max} \log^2 (A/\delta)/ \sigma^2_{\min},$ we have that with probability at least $1-8\delta$, the following holds: for all action $a$, $\dfrac{\sigma^2(a)}{\wsigma^2_\Gamma(a)} \leq 1 + \frac{4C d^2 \sigma^2_{\max} \log (A / \delta)}{\sigma^2_{\min}\Gamma}$.
\end{corollary}

\begin{proof}
    From the \Cref{lemma:conc}, we know that $\left|\bx(a)^{\top}(\wSigma_\Gamma - \bSigma_*)\bx(a)\right| \leq \frac{2C d^2  \log (A / \delta)}{\Gamma}$ with probability at least $1-8\delta$. Hence we can show that
    \begin{align*}
        |\wsigma^2_\Gamma(a) - \sigma^2(a)| \leq \frac{2C d^2 \sigma^2_{\max} \log (A / \delta)}{\Gamma} 
        \implies &\sigma^2(a) - \frac{2C d^2 \sigma^2_{\max} \log (A / \delta)}{\Gamma} \leq \wsigma^2_\Gamma(a) \leq \sigma^2(a) + \frac{2C d^2 \sigma^2_{\max} \log (A / \delta)}{\Gamma}\\
        \implies & 1 - \frac{2C d^2 \sigma^2_{\max} \log (A / \delta)}{\sigma^2(a)\Gamma} \leq \dfrac{\wsigma^2_\Gamma(a)}{\sigma^2(a)} \leq 1 + \frac{2C d^2 \sigma^2_{\max} \log (A / \delta)}{\sigma^2(a)\Gamma}\\
        \implies & 1 - \frac{2C d^2 \sigma^2_{\max} \log (A / \delta)}{\sigma^2_{\min}\Gamma} \leq \dfrac{\wsigma^2_\Gamma(a)}{\sigma^2(a)} \leq 1 + \frac{2C d^2 \sigma^2_{\max}\sigma^2_{\max} \log (A / \delta)}{\sigma^2_{\min}\Gamma}\\
        \implies & \dfrac{1}{1 + \frac{2C d^2 \sigma^2_{\max} \log (A / \delta)}{\sigma^2_{\min}\Gamma}} \leq \dfrac{\sigma^2(a)}{\wsigma^2_\Gamma(a)} \leq \dfrac{1}{1 - \frac{2C d^2 \sigma^2_{\max} \log (A / \delta)}{\sigma^2_{\min}\Gamma}}.
    \end{align*}
    It follows then that
    \begin{align*}
        \dfrac{\sigma^2(a)}{\wsigma^2_\Gamma(a)} \leq \dfrac{1}{1 - \frac{2C d^2 \sigma^2_{\max} \log (A / \delta)}{\sigma^2_{\min}\Gamma}} \overset{(a)}{\leq} 1 + \frac{4C d^2 \sigma^2_{\max} \log (A / \delta)}{\sigma^2_{\min}\Gamma}
    \end{align*}
    where, $(a)$ follows for $x = \frac{2C d^2 \sigma^2_{\max} \log (A / \delta)}{\sigma^2_{\min}\Gamma}$ and
    \begin{align*}
        \dfrac{1}{1-x} \leq 1 + 2x \implies 1 \leq 1 + x - 2x^2 \implies x(1-2x) \geq 0
    \end{align*}
    which holds for $x = \frac{2C d^2 \sigma^2_{\max} \log (A / \delta)}{\sigma^2_{\min}\Gamma} \leq \frac{1}{2}$. For $n\geq 4C^2 d^2 \sigma^2_{\max} \log^2 (A/\delta)/ \sigma^2_{\min}$ we can show that $x\leq \frac{1}{2}$. The claim of the corollary follows.
\end{proof}

%% file: theorem1.tex
\begin{corollary}
\label{corollary:additive}
For, $n\geq 16C^2 d^4 \sigma^4_{\max} \log^2 (A/\delta)/ \sigma^4_{\min}$ we have that for all action $a$, $|\wsigma^2_\Gamma(a) - \sigma^2(a)|\leq \sigma^2_{\min}/2$.
\end{corollary}

\begin{proof}
    From the \Cref{lemma:conc}, we know that $\left|\bx(a)^{\top}(\wSigma_\Gamma - \bSigma_*)\bx(a)\right| \leq \frac{2C d^2 \sigma^2_{\max} \log (A / \delta)}{\Gamma}$ with probability $1-8\delta$. Hence we can show that
    \begin{align*}
        |\wsigma^2_\Gamma(a) - \sigma^2(a)| &\leq \frac{2C d^2 \sigma^2_{\max} \log (A / \delta)}{\Gamma} 
        =  \frac{2C d^2 \sigma^2_{\max} \log (A / \delta)}{\sqrt{n}}
        \overset{(a)}{\leq} \frac{2C d^2 \sigma^2_{\max} \log (A / \delta)}{\sqrt{16C^2 d^4 \sigma^4_{\max} \log^2 (A/\delta)/ \sigma^4_{\min}}} = \dfrac{\sigma^2_{\min}}{2},
    \end{align*}
    where $(a)$ follows for $n\geq 16C^2 d^4 \sigma^4_{\max} \log^2 (A/\delta)/ \sigma^4_{\min}$. 
    The claim of the corollary follows.
\end{proof}

\begin{lemma}
\textbf{(Loss Concentration of design matrix)} Let $\wSigma_\Gamma$ be the empirical estimate of $\bSigma_*$. Define $\bV=\sum_{a,a'}\bw(a)\bw(a')^{\top}$. We have that for any arbitrary proportion $\bb$ the following 
\begin{align*}
    \Pb&\bigg(\left|\sum_{a,a'} \bw(a)^{\top}(\bA^{-1}_{\bb_*, \wSigma_\Gamma} - \bA^{-1}_{\bb_*, \bSigma_*})\bw(a')\right|
    \leq \frac{2C B^* d^3 \log (A/ \delta)}{\Gamma}\bigg)\geq 1-\delta
\end{align*}
where $B^*$ is a problem-dependent quantity such that 
$$
B^* = \left(\left\|\bA^{-1}_{\bb_*, \bSigma_*} \bw\right\|^2  \left\|\sum_{a=1}^A\bb_*(a)\bw(a)\bw(a)^{\top}H^2_U\right\| \left\|\left(\sum_{a=1}^A\dfrac{\bb_*(a)\bw(a)\bw(a)^{\top}}{\sigma^2(a) + \frac{2C d^2 \sigma^2_{\max} \log (9 H^2_U / \delta)}{\sqrt{n}} }\right)^{-1}\bw\right\|\right)
$$ and $C>0$ is a universal constant.
\end{lemma}

\begin{proof}
We have the following
\begin{align}
    &\left|\sum_{a,a'} \bw(a)^{\top}\bA^{-1}_{\bb_*, \wSigma_\Gamma}\bw(a') - \sum_{a,a'} \bw(a)^{\top}\bA^{-1}_{\bb_*, \bSigma_*}\bw(a')\right| \overset{}{=} \left|\underbrace{\sum_{a} \bw(a)^{\top}}_{\bw}\left(\bA^{-1}_{\bb_*, \wSigma_\Gamma} - \bA^{-1}_{\bb_*, \bSigma_*}\right)\underbrace{\sum_{a}\bw(a)}_{\bw}\right| \nonumber\\
    &= \left| \bw^{\top}\left(\bA^{-1}_{\bb_*, \bSigma_*}\left(\bA_{\bb_*,\bSigma_*} - \bA_{\bb_*,\wSigma_\Gamma}\right)\bA^{-1}_{\bb_*, \wSigma_\Gamma}\right)\bw\right| = \left| \underbrace{\bw^{\top}\left(\right.\bA^{-1}_{\bb_*, \bSigma_*}}_{\bu}\left(\bA_{\bb_*,\bSigma_*} - \bA_{\bb_*,\wSigma_\Gamma}\right)\underbrace{\bA^{-1}_{\bb_*, \wSigma_\Gamma}\left.\right)\bw}_{\mathbf{v}}\right|\nonumber\\
    &= \left|\bu\left(\bA_{\bb_*,\bSigma_*} - \bA_{\bb_*,\wSigma_\Gamma}\right)\mathbf{v}\right| \overset{(a)}{\leq} \|\bu\|\underbrace{\left\|\bA_{\bb_*,\bSigma_*} - \bA_{\bb_*,\wSigma_\Gamma}\right\|}_{\Delta}\|\mathbf{v}\|
    \label{eq:lemma-conc1}
\end{align}
where, $(a)$ follows by Cauchy-Schwarz inequality. Now observe that the vector $\bu\in\R^d$ is a problem dependent quantity. We now bound the $\Delta$ in \eqref{eq:lemma-conc1} as follows
\begin{align*}
    \Delta = &\left\|\sum_{a=1}^A\dfrac{\bb_*(a)\bw(a)\bw(a)^{\top}}{\bx(a)^{\top}\bSigma_*\bx(a)} - \sum_{a=1}^A\dfrac{\bb_*(a)\bw(a)\bw(a)^{\top}}{\bx(a)^{\top}\wSigma_\Gamma\bx(a)}\right\|\\
    &\overset{(a)}{=} \left\|\sum_a\dfrac{\bb_*(a)\bw(a)\bw(a)^{\top}}{\sigma^2(a)} 
    - \sum_a\dfrac{\bb_*(a)\bw(a)\bw(a)^{\top}}{\wsigma^2_\Gamma(a)}\right\| \\
    &= \left\|\sum_{a}\bb_*(a)\bw(a)\bw(a)^{\top}\left(\dfrac{1}{\sigma^2(a)} - \dfrac{1}{\wsigma^2_\Gamma(a)}\right)\right\| \\
    &= \left\|\sum_{a}\bb_*(a)\bw(a)\bw(a)^{\top}\left(\dfrac{\wsigma^2_\Gamma(a) - \sigma^2(a)}{\wsigma^2_\Gamma(a)\sigma^2(a)}\right)\right\|\\
    &\overset{(b)}{\leq} \left\|\sum_{a}\bb_*(a)\bw(a)\bw(a)^{\top}\left(\dfrac{\wsigma^2_\Gamma(a) - \sigma^2(a)}{\sigma^4_{\min}}\right)\right\|\\
    &= \left\|\dfrac{1}{\sigma^4_{\min}}\sum_{a}\bb_*(a)\bw(a)\bw(a)^{\top}\left(\bx(a)^\top\wSigma_\Gamma\bx(a) - \bx(a)^\top\bSigma_*\bx(a)\right)\right\|\\
    &\overset{}{=} \dfrac{1}{\sigma^4_{\min}}\left\|\sum_{a=1}^A\underbrace{\bb_*(a)\bw(a)\bw(a)^{\top}}_{\textbf{Problem dependent quantity}}\underbrace{\left(\bx(a)^\top\left(\wSigma_\Gamma - \bSigma_*\right)\bx(a)\right)}_{\textbf{Random Quantity}}\right\|
\end{align*}
where, $(a)$ follows $\wsigma^2_\Gamma(a)= \bx(a)^\top\wSigma_\Gamma\bx(a)$ and $\sigma^2(a)= \bx(a)^\top\bSigma_*\bx(a)$, ad $(b)$ follows from \Cref{corollary:additive}. 
Now observe from \Cref{lemma:conc-operator} that we can bound the quantity 
\begin{align*}
    \|\wSigma_\Gamma - \bSigma_*\| \leq \frac{2C d^2 \sigma^2_{\max} \lambda^{-1}_{\min}(\bY)  \log (A / \delta)}{\Gamma}.
\end{align*}
Then we also have that the spread of maximum eigenvalue of $\|\wSigma_\Gamma - \bSigma_*\|_2$ is controlled which implies
\begin{align*}
    \dfrac{1}{\sigma^4_{\min}}&\left\|\sum_{a=1}^A\underbrace{\bb_*(a)\bw(a)\bw(a)^{\top}}_{\textbf{Problem dependent quantity}}\underbrace{\left(\bx(a)^\top\left(\wSigma_\Gamma - \bSigma_*\right)\bx(a)\right)}_{\textbf{Random Quantity}}\right\| \\
    &\qquad \overset{(a)}{\leq}  \left\|\sum_{a=1}^A\bb_*(a)\bw(a)\bw(a)^{\top}(\bx(a)^{\top}\bx(a))\right\|\frac{2C d^2 \sigma^2_{\max} \lambda^{-1}_{\min}(\bY)  \log (A / \delta)}{\Gamma}
\end{align*}
where, $(a)$ follows by \Cref{lemma:conc-operator}. Next for the third quantity in \eqref{eq:lemma-conc1} we can bound as follows

\begin{align*}
    \|\mathbf{v}\| = \|\bA^{-1}_{\bb_*, \wSigma_\Gamma}\bw\| &= \left\|\left(\sum_{a=1}^A\dfrac{\bb_*(a)\bw(a)\bw(a)^{\top}}{\wsigma^2_\Gamma(a)}\right)^{-1}\bw\right\| \overset{(a)}{\leq} \left\|\left(\sum_{a=1}^A\dfrac{\bb_*(a)\bw(a)\bw(a)^{\top}}{\sigma^2(a) + \frac{2C d^2 \sigma^2_{\max} \log (A / \delta)}{\sqrt{n}} }\right)^{-1}\bw\right\|
\end{align*}
where, $(a)$ follows as
\begin{align*}
    \wsigma^2(a) \leq \sigma^2(a) + \frac{2C d^2 \sigma^2_{\max} \log (A / \delta)}{\Gamma}
\end{align*}
from \Cref{lemma:conc}.
Finally observe that the first part of \eqref{eq:lemma-conc1} we have that $\bw^{\top}\bA^{-1}_{\bb_*, \bSigma_*}$ is a problem dependent parameter. 
Finally, plugging back everything in \eqref{eq:lemma-conc1} we get
\begin{align*}
    \|\bu\|&\left\|\bA_{\bb_*,\bSigma_*} - \bA_{\bb_*,\wSigma_\Gamma}\right\|\|\mathbf{v}\|\\
    &\leq \left\|\bA^{-1}_{\bb_*, \bSigma_*} \bw\right\|  \left\|\sum_{a=1}^A\bb_*(a)\bw(a)\bw(a)^{\top}(\bx(a)^{\top}\bx(a))\right\|\frac{2C d^2 \sigma^2_{\max} \log (A / \delta)}{\sigma^4_{\min}\Gamma}
    \left\|\left(\sum_{a=1}^A\dfrac{\bb_*(a)\bw(a)\bw(a)^{\top}}{\sigma^2(a) + \frac{2C d^2 \sigma^2_{\max} \log (A / \delta)}{\Gamma} }\right)^{-1}\bw\right\|\\
    &\leq \underbrace{\left(\left\|\bA^{-1}_{\bb_*, \bSigma_*} \bw\right\|^2  \left\|\sum_{a=1}^A\bb_*(a)\bw(a)\bw(a)^{\top}H^2_U\right\| \left\|\left(\sum_{a=1}^A\dfrac{\bb_*(a)\bw(a)\bw(a)^{\top}}{\sigma^2(a) + \frac{2C d^2 \sigma^2_{\max} \log (A / \delta)}{\Gamma} }\right)^{-1}\bw\right\|\right)}_{B^*}\frac{2C d^3 \log (A / \delta)}{\Gamma}\\
    &\overset{(a)}{=} \frac{2C B^* d^3 \sigma^2_{\max} \lambda^{-1}_{\min}(\bY)  \log (A / \delta)}{\Gamma}
\end{align*}
where, $(a)$ follows by substituting the value of $B^*$.
\end{proof}

\subsection{Regret Bound of \sp}
\label{app:speed-regret}

\begin{customtheorem}{1}\textbf{(formal)}
Running \Cref{alg:linear-bandit} with budget $n\geq 16C^2 d^4\log^2 (A/\delta)/ \sigma^4_{\min}$ the resulting regret satisfies
\begin{align*}
    \cR_n &\leq \dfrac{1}{n^{3/2}} + O_{\kappa^2,H^2_U}\left(\dfrac{d^2 \sigma^2_{\max}\log (n)}{\sigma^2_{\min} n^{3/2}}\right)+\dfrac{2B^*C d^3 \sigma^2_{\max} \log (n)}{\sigma^2_{\min} n^{3/2}}  + \dfrac{d^2}{n^2}\Tr\left(\sum_{a,a'}\bw(a)\bw(a')^\top\right) + \dfrac{2A H_U^2\kappa^2}{n^2}\\ 
    &=  O_{\kappa^2,H^2_U}\left(\dfrac{B^* d^3 \sigma^2_{\max}\log (n)}{\sigma^2_{\min} n^{3/2}}\right).
\end{align*}
\end{customtheorem}

\begin{proof}
We follow the same steps as in \Cref{prop:loss-bandit-tracker}. Observe that $\tfrac{16C^2 d^4 \sigma^4_{\max} \log^2 (A/\delta)}{\sigma^4_{\min}} > \frac{2C d^2 \sigma^2_{\max} \log (A / \delta)}{\sigma^2_{\min}\Gamma}$. 
Hence for $\bz = \sum_a\bw(a)$ the loss function for $n\geq \frac{2C d^2 \sigma^2_{\max} \log (A / \delta)}{\sigma^2_{\min}\Gamma}$  as follows
\begin{align*}
    \bL_n(\pi,\wb,\wSigma_\Gamma) \coloneqq \E\left[\left(\bz^{\top}(\wtheta_{n-\Gamma} - \btheta_*)\right)^2\right] 
    &\overset{(a)}{\leq} \left(1+2C_{\Gamma, \sigma^2_{\min}}(\delta)\right)\bz^\top (\tX_{n-\Gamma}^{\top}\wSigma_{\Gamma}^{-1}\tX_{n-\Gamma})^{-1}\bz.
\end{align*}
where, $(a)$ follows from \eqref{eq:upper-bound-thm-1}.
Recall that the quantity of the samples collected (following $\wb^{*}$) after exploration is as follows:
\begin{align*}
   \left(\tX_{n-\Gamma}^{\top}\wSigma_{\Gamma}^{-1}\tX_{n-\Gamma}\right)^{-1} = \left(\sum_a\left\lceil(n-\Gamma)\wb^*(a)\wsigma^{-2}_\Gamma(a)\right\rceil\bw(a)\bw(a)^{\top}\right)^{-1} 
    = \dfrac{1}{n-\Gamma}\bA_{\wb^*,\wSigma_\Gamma}^{-1}.
\end{align*}
Hence we use the loss function
\begin{align*}
    \L'_{n-\Gamma}(\pi,\wb,\wSigma_\Gamma) \coloneqq 
    \left(1+2C_\Gamma(\delta)\right)\bz^\top (\tX_{n-\Gamma}^{\top}\wSigma_{\Gamma}^{-1}\tX_{n-\Gamma})^{-1}\bz = \frac{\left(1+2C_{\Gamma, \sigma^2_{\min}}(\delta)\right)}{n - \Gamma}\sum_{a,a'}\bw(a)^\top\bA_{\wb^*, \wSigma_\Gamma}^{-1}\bw(a').
\end{align*}
Also, recall that we define
\begin{align*}
    \L_n(\pi,\bb_*,\wSigma_\Gamma) = \dfrac{1}{n}\sum_{a,a'}\bw(a)^\top\bA_{\bb_*, \wSigma_\Gamma}^{-1}\bw(a').
\end{align*}
%
Then we can decompose the regret as follows:
\begin{align*}
    \cR_n 
    &=  \bL_n(\pi,\wb,\wSigma_\Gamma) - \L^*_n(\pi,\bb_*,\bSigma_*)\\
    &\leq \L'_{n-\Gamma}(\pi,\wb,\wSigma_\Gamma) -  \L'_{n-\Gamma}(\pi,\wb^*,\wSigma_\Gamma) + \L'_{n-\Gamma}(\pi,\wb^*,\wSigma_\Gamma) - \L^*_n(\pi,\bb_*,\bSigma_*)
    \\
    &= \underbrace{\L'_{n-\Gamma}(\pi,\wb,\wSigma_\Gamma) - \L'_{n-\Gamma}(\pi,\wb^*,\wSigma_\Gamma)}_{\textbf{Approximation error}} + \underbrace{\L'_{n-\Gamma}(\pi,\wb^*,\wSigma_\Gamma) -  \L_n(\pi,\bb_*,\wSigma_\Gamma)}_{\textbf{Comparing two diff loss}} + \underbrace{\L_n(\pi,\bb_*,\wSigma_\Gamma) - \L^*_n(\pi,\bb_*,\bSigma_*)}_{\textbf{Estimation error of $\bSigma_*$}}
    \end{align*}
First recall that the good variance event as follows:
\begin{align*}
    \xi^{var}_\delta(\Gamma) \coloneqq \left\{\forall a, \left|\bx(a)^{\top}\left(\wSigma_\Gamma - \bSigma_*\right)\bx(a)\right| < \frac{2C d^2 \sigma^2_{\max} \log (A / \delta)}{\Gamma}\right\}.
\end{align*}
%
%
Now first observe that $n\geq 16C^2 d^4 \sigma^4_{\max} \log^2 (A/\delta)/ \sigma^4_{\min}$ is a larger regime than $n\geq \frac{2C d^2 \sigma^2_{\max} \log (A / \delta)}{\sigma^2_{\min}\Gamma}$ required for \Cref{prop:loss-bandit-tracker}.
Then under the good variance event, following the same steps as \Cref{prop:loss-bandit-tracker} we can bound the approximation error setting $\delta = 1/n^3$ as follows
\begin{align*}
    \L'_{n-\Gamma}(\pi,\wb,\wSigma_\Gamma) - \L'_{n-\Gamma}(\pi,\wb^*,\wSigma_\Gamma) 
    &\overset{}{\leq}  
    O_{\kappa^2,H^2_U}\left(\dfrac{d^2 \sigma^2_{\max} \log(A/\delta)}{\sigma^2_{\min} n^{3/2}}\right)\indic{ \xi^{var}_\delta(\Gamma)} + \sum_{t=1}^n A H_U^2\kappa^2\Pb\left(\left(\xi^{var}_\delta(\Gamma)\right)^c\right)\\
    &\overset{}{\leq}  
    O_{\kappa^2,H^2_U}\left(\dfrac{d^2 \sigma^2_{\max} \log(A/\delta)}{\sigma^2_{\min} n^{3/2}}\right) + \dfrac{A H_U^2\kappa^2}{n^2}
\end{align*}
and the second part of comparing the two losses as
\begin{align*}
    \L'_{n-\Gamma}(\pi,\wb^*,\wSigma_\Gamma) -  \L_{n}(\pi,\bb_*,\wSigma_\Gamma) 
    &\overset{}{\leq} O_{\kappa^2,H^2_U}\left(\dfrac{d^2 \sigma^2_{\max} \log(A/\delta)}{\sigma^2_{\min} n^{3/2}}\right)\indic{ \xi^{var}_\delta(\Gamma)} + \sum_{t=1}^n A H_U^2\kappa^2\Pb\left(\left(\xi^{var}_\delta(\Gamma)\right)^c\right)\\
    &\leq O_{\kappa^2,H^2_U}\left(\dfrac{d^2 \sigma^2_{\max} \log(A/\delta)}{\sigma^2_{\min} n^{3/2}}\right) + \dfrac{A H_U^2\kappa^2}{n^2}
\end{align*}

We define the good estimation event as follows:
\begin{align*}
    \xi^{est}_\delta(\Gamma) \coloneqq \left\{\left|\sum_{a,a'} \bw(a)^{\top}\bA^{-1}_{\bb_*, \wSigma_\Gamma}\bw(a') - \sum_{a,a'} \bw(a)^{\top}\bA^{-1}_{\bb_*, \bSigma_*}\bw(a')\right| \leq \frac{2C B^* d^3 \sigma^4_{\max} \log (9 H^2_U / \delta)}{\sigma^4_{\min}\Gamma}\right\} 
\end{align*}
Under the good estimation event $\xi^{est}(\Gamma)$ and using \Cref{lemma:gradient-conc} we can show that the estimation error is given by

\begin{align*}
    &\L_n(\pi,\bb_*,\wSigma_\Gamma) - \L_n(\pi,\bb_*,\bSigma_*) \leq \left(\frac{1}{n} \sum_{a,a'}\bw(a)^\top\bA_{\bb_*, \wSigma_\Gamma}^{-1}\bw(a') - \frac{1}{n} \sum_{a,a'}\bw(a)^\top\bA_{\bb_*,\bSigma_*}^{-1}\bw(a')\right)\indic{\xi^{est}_\delta(\Gamma)} \\
    &\qquad +\left(\frac{1}{n} \sum_{a,a'}\bw(a)^\top\bA_{\bb_*, \wSigma_\Gamma}^{-1}\bw(a') - \frac{1}{n} \sum_{a,a'}\bw(a)^\top\bA_{\bb_*, \bSigma_*}^{-1}\bw(a')\right)\indic{\xi^{est}_\delta(\Gamma)^C}\\
    &= \left(\frac{1}{n} \sum_{a,a'}\bw(a)^\top\bA_{\bb_*, \wSigma_\Gamma}^{-1}\bw(a') - \frac{1}{n} \sum_{a,a'}\bw(a)^\top\bA_{\bb_*,\bSigma_*}^{-1}\bw(a')\right)\indic{\xi^{est}_\delta(\Gamma)} \\
    &\qquad +\frac{1}{n} \Tr\left(\left(\bA_{\bb_*, \wSigma_\Gamma}^{-1} - \bA_{\bb_*, \bSigma_*}^{-1}\right) \left(\sum_{a,a'}\bw(a)\bw(a')^\top\right)\right)\indic{\xi^{est}_\delta(\Gamma)^C}\\
    &\overset{(a)}{\leq} \dfrac{1}{n}2B^*\dfrac{Cd^3 \sigma^2_{\max} \log (1/\delta)}{\sigma^2_{\min} \Gamma}  + \dfrac{1}{n}\Tr\left(\bA_{\bb_*, \bSigma_*}^{-1}\right)\Tr\left(\bA_{\bb_*, \wSigma_\Gamma}^{-1}\right)\Tr\left(\sum_{a,a'}\bw(a)\bw(a')^\top\right)\delta\\
    &\overset{(b)}{\leq} \dfrac{1}{n}2B^*\dfrac{Cd^3 \sigma^2_{\max} \log (n)}{\sigma^2_{\min} \sqrt{n}}  + \dfrac{d^2}{n^2}\Tr\left(\sum_{a,a'}\bw(a)\bw(a')^\top\right)
    = \dfrac{2B^*C d^3 \sigma^2_{\max} \log (n)}{\sigma^2_{\min} n^{3/2}}  + \dfrac{d^2}{n^2}\Tr\left(\sum_{a,a'}\bw(a)\bw(a')^\top\right)
\end{align*}
where, $(a)$ follows from \Cref{lemma:gradient-conc}, $(b)$ follows as $\Gamma=\sqrt{n}$ and setting $\delta=\frac{1}{n^3}$. Combining everything we have the following regret as
\begin{align*}
    \cR_n &\leq \dfrac{1}{n^{3/2}} + O_{\kappa^2,H^2_U}\left(\dfrac{d^2 \sigma^2_{\max} \log (n)}{\sigma^2_{\min} n^{3/2}}\right)+\dfrac{2B^*Cd^3 \sigma^2_{\max} \log (n)}{\sigma^2_{\min} n^{3/2}}  + \dfrac{d^2}{n^2}\Tr\left(\sum_{a,a'}\bw(a)\bw(a')^\top\right) + \dfrac{2A H_U^2\kappa^2}{n^2}\\
    &=  O_{\kappa^2,H^2_U}\left(\dfrac{B^* d^3 \sigma^2_{\max} \log (n)}{\sigma^2_{\min} n^{3/2}}\right)
\end{align*}
where, $B^* = \left(\left\|\bA^{-1}_{\bb_*, \bSigma_*} \bw\right\|^2  \left\|\sum_{a=1}^A\bb_*(a)\bw(a)\bw(a)^{\top}H^2_U\right\| \left\|\left(\sum_{a=1}^A\dfrac{\bb_*(a)\bw(a)\bw(a)^{\top}}{\sigma^2(a) + \frac{2C d^3 \log (9 H^2_U / \delta)}{\sqrt{n}} }\right)^{-1}\bw\right\|\right)$.
The claim of the theorem follows.
\end{proof}

\begin{remark}\textbf{(Discussion on Sample regime and $B_*$):} Observe that combining \Cref{prop:loss-oracle} and \Cref{thm:regret-linear-bandit} we can have a loss of \sp\ that scales as
\begin{align*}
O_{\kappa^2,H^2_U, \sigma^2_{\max}, \sigma^2_{\min}}\left(\dfrac{d \log(n)}{n}\right) + O_{\kappa^2,H^2_U, \sigma^2_{\max}, \sigma^2_{\min}}\left(\dfrac{B^\star d^3 \log(n)}{n^{3/2}}\right) 
\end{align*}
which seems to contradict the loss bound in \Cref{prop:loss-bandit-tracker}.

However, this is not the case. Observe that the $B_*$ is a problem-dependent quantity that depends on a number of samples $n$. We define it 
as
$$
B_* := \left\|\bA_{\bb_*, \bSigma_*}^{-1} \bw\right\|^2\left\|\sum_{a=1}^A \bb_*(a) \bw(a) \bw(a)^{\top} H_U^2\right\|\left\|\left(\sum_{a=1}^A \frac{\bb_*(a) \bw(a) \bw(a)^{\top}}{\sigma^2(a)+\frac{2 C d^2 \log (A / \delta)}{\sqrt{n}}}\right)^{-1} \bw\right\|.
$$
However, there are two regimes when $n \leq \frac{16 C^2 d^4 \sigma_{\max }^4 \log ^2(A / \delta)}{\sigma_{\min }^4}$ then $B_* = \Theta(\sqrt{n})$ and for $n > \frac{16 C^2 d^4 \sigma_{\max }^4 \log ^2(A / \delta)}{\sigma^4_{\min }}$ then $B_* = o(\sqrt{n})$ .
In the first case when $n \leq \frac{16 C^2 d^4 \sigma_{\max }^4 \log ^2(A / \delta)}{\sigma^4_{\min }}$ with $B_* = \Theta(\sqrt{n})$ we have the loss that scales as
\begin{align*}
 O_{\kappa^2,H^2_U, \sigma^2_{\max}, \sigma^2_{\min}}\left(\dfrac{d \log(n)}{n}\right) + O_{\kappa^2,H^2_U, \sigma^2_{\max}, \sigma^2_{\min}}\left(\dfrac{B_* d^3 \log(n)}{n^{3/2}}\right) = O_{\kappa^2,H^2_U, \sigma^2_{\max}, \sigma^2_{\min}}\left(\dfrac{d^3\log(n)}{n}\right)
\end{align*}
This is the regime of \Cref{prop:loss-bandit-tracker} as it holds for all $n \geq \frac{2 C d^2 \sigma^2_{\max} \log (A / \delta)}{\sigma_{\min }^2 \Gamma}$ for $\Gamma \geq 1$. Note that $\frac{2 C d^2 \sigma^2_{\max} \log (A / \delta)}{\sigma_{\min }^2 \Gamma}$ is less than $\frac{16 C^2 d^4 \sigma^4_{\max} \log ^2(A / \delta)}{\sigma_{\min }^4}$.

In the second case when $n > \frac{16 C^2 d^4 \sigma^4_{\max}\log ^2(A / \delta)}{\sigma_{\min }^4}$ with $B_* =o(\sqrt{n})$ we have a tighter bound as the first term dominates and we have the loss scaling as
     \begin{align*}
       O_{\kappa^2,H^2_U, \sigma^2_{\max}, \sigma^2_{\min}}\left(\dfrac{d \log(n)}{n}\right) + O_{\kappa^2,H^2_U, \sigma^2_{\max}, \sigma^2_{\min}}\left(\dfrac{B^{\star} d^3 \log(n)}{n^{3/2}}\right) = O_{\kappa^2,H^2_U, \sigma^2_{\max}, \sigma^2_{\min}}\left(\dfrac{d\log(n)}{n}\right)
      \end{align*}
Intuitively this is a larger sample regime where the \sp\ has a good estimation of $\mathbf{\bSigma_*}$ and the design matrix estimation has also concentrated.
Combining both the regimes we can show that for $n \geq \frac{2 C d^2 \log (A / \delta)}{\sigma\_{\min }^2 \Gamma}$ the loss of \sp\ scales by
\begin{align*}
\max\left\lbrace O_{\kappa^2,H^2_U, \sigma^2_{\max}, \sigma^2_{\min}}\left(\dfrac{d\log(n)}{n}\right), O_{\kappa^2,H^2_U, \sigma^2_{\max}, \sigma^2_{\min}}\left(\dfrac{d^3\log(n)}{n}\right) \right\rbrace = O_{\kappa^2,H^2_U, \sigma^2_{\max}, \sigma^2_{\min}}\left(\dfrac{d^3\log(n)}{n}\right)
\end{align*}
which is the bound of \Cref{prop:loss-bandit-tracker}. So in summary \Cref{prop:loss-bandit-tracker} is a more general bound for a larger regime size than \Cref{thm:regret-linear-bandit} and does not contradict the theorem statement.
\end{remark}

%% file: regret_lower_bound.tex
\begin{customtheorem}{2}\textbf{(Lower Bound)}
Let $|\bTheta| = 2^d$ and $\btheta_*\in\bTheta$. Then any $\delta$-PAC policy $\pi$ following the design $\bb\in\triangle(\A)$ 
satisfies $\cR'_{n} = \L_{n}(\pi,\wb,\bSigma_*) - \L_{n}(\pi,\bb^*,\bSigma_*) \geq \Omega\left(\dfrac{d^2\lambda_d(\bV)\log({n})}{{n}^{3/2}}\right)$ for the environment 
in \eqref{eq:minimax-environment}. 
\end{customtheorem}

\begin{proof}
\textbf{Step 1 (Define Environment):} We define an environment model $B_j$ consisting of $A$ actions and $J$ hypotheses with true hypothesis $\btheta_* = \btheta_j$ ($j$-th column) 
as follows:
\begin{align}
\begin{matrix} 
    \btheta &= & \btheta_1 &\btheta_2  &  \btheta_3 & \ldots & 
    \btheta_J \\\hline
    \mu_1(\btheta) &=  & \beta & \beta\!-\!\frac{\beta}{J} & \beta\!-\!\frac{2\beta}{J} & \ldots & \beta\!-\!\frac{(J-1)\beta}{J}\\
    \mu_2(\btheta) &=   & \iota_{21} & \iota_{22} & \iota_{23} & \ldots & \iota_{2J}\\
    &\vdots & &&\vdots\\
    \mu_A(\btheta) &=   & \iota_{A1} & \iota_{A2} & \iota_{A3} & \ldots & \iota_{AJ}
\end{matrix}
\label{eq:minimax-environment}
\end{align}
where, each $\iota_{ij}$ is distinct and satisfies $\iota_{ij} < \beta/4J$.
\(\btheta_1\) is the optimal hypothesis in \(B_1\), \(\btheta_2\) is the optimal hypothesis in \(B_2\) and so on such that for each $B_j$ and $j\in[J]$ we have column $j$ as the optimal hypothesis.

Finally, assume that $\bSigma = \btheta\btheta^\top$ is a rank one matrix. To distinguish between the covariance matrix between two distributions we denote $\bSigma_{\btheta} = \btheta\btheta^\top$. 
Therefore we have that $\sigma^2_i(\btheta) = \bx_i^\top\bSigma_{\btheta}\bx_i = (\bx_i^\top \btheta)^2 = \mu^2_i(\btheta)$. 
%
%
%
Hence for any algorithm, identifying the co-variance matrix $\bSigma_{\btheta_*}$ is the same as identifying the $\btheta_*$. Also assume that $\pi(a) = \frac{1}{A}$. Hence each action is equally weighted by the target policy. 

This is a general hypothesis testing setting where the functions $\mu_a(\btheta)$ can be thought of as linear functions of $\btheta$ such that $\mu_a(\btheta) = \bx(a)^\top\btheta$. Assume that $0 < \mu_a(\btheta) \leq 1$, and $\log(\mu_a(\btheta)/\mu_a(\btheta')) > 1/4$. 


Now observe that between any two hypothesis $\btheta$ and $\btheta'$ we have the following
\begin{align}
    \KL\bigg(\N(\mu_{i}(\btheta), \bx_i^\top\bSigma_{\btheta}\bx_i))\big|\big|\N(\mu_{i}(\btheta'), &\bx_i^\top\bSigma_{\btheta'}\bx_i))\bigg) = 2\log(\dfrac{\sigma_{i}(\btheta')}{\sigma_{i}(\btheta)}) + \dfrac{\sigma^2_{i}(\btheta) + (\mu_{i}(\btheta) - \mu_{i}(\btheta'))^2}{2\sigma^2_{i}(\btheta')} - \dfrac{1}{2}\nonumber\\
    &\overset{(a)}{=} 2\log(\dfrac{\mu_{i}(\btheta')}{\mu_{i}(\btheta)}) + \dfrac{\mu^2_{i}(\btheta) + (\mu_{i}(\btheta) - \mu_{i}(\btheta'))^2}{2\mu^2_{i}(\btheta')} - \dfrac{1}{2} \overset{(a)}{\geq}  \dfrac{ (\mu_{i}(\btheta) - \mu_{i}(\btheta'))^2}{8} \label{eq:oracle-proof-1}
\end{align}
where, $(a)$ follows from the condition that $0 < \mu_a(\btheta) \leq 1$, and $\log(\mu_a(\btheta)/\mu_a(\btheta')) > 1/4$.

\textbf{Step 2 (Minimum samples to verify $\btheta_*$):} Let, \(\Lambda_1\) be the set of alternate  models having a different optimal hypothesis than \(\btheta^{*} = \btheta_1\) such that all models having different optimal hypothesis than $\btheta_1$ such as  $B_2, B_3, \ldots B_J$ are in $\Lambda_1$. Let $\tau_\delta$ be the stopping time for any $\delta$-PAC policy $\bb$. That is $\tau_\delta$ is the time that any algorithm stops and outputs its estimate $\wtheta_{\tau_\delta}$. Let $T_t(a)$ denote the number of times the action $a$ has been sampled till round $t$.
Let $\wtheta_{\tau_\delta}$ be the predicted optimal hypothesis at round $\tau_\delta$. We first consider the  model $B_1$. Define the event \(\xi=\{\wtheta_{\tau_\delta} \neq \btheta_*\}\) as the error event in model $B_1$. Let the event \(\xi'=\{\wtheta_{\tau_\delta} \neq \btheta^{'*}\}\) be the corresponding error event in model $B_2$. Note that $\xi^{\complement} \subset \xi'$. Now since $\bb$ is $\delta$-PAC policy we have $\Pb_{B_1,\bb}(\xi) \leq \delta$ and $\Pb_{B_2,\bb}(\xi^{\complement}) \leq \delta$. Hence we can show that,

\begin{align}
2 \delta \geq \Pb_{B_1, \bb}(\xi)+ \Pb_{B_2, \bb}(\xi^{\complement})
&\overset{(a)}{\geq} \frac{1}{2} \exp\left(-\KL\left(P_{B_1, \bb} || P_{B_2, \bb}\right)\right)\nonumber\\
 \KL\left(P_{B_1, \bb} || P_{B_2, \bb}\right) & \geq \log\left(\dfrac{1}{4\delta}\right)\nonumber\\
 \dfrac{1}{8}\sum_{i=1}^{A} \E_{B_1, \bb}[T_{\tau_\delta}(i)]\cdot\left(\mu_{i}(\btheta_*)^{}-\mu_{i}( \btheta^{'}_*)_{}\right)^2&\overset{(b)}{\geq} \log\left(\dfrac{1}{4\delta}\right)\nonumber\\
 \dfrac{1}{8}\left(\beta - \beta +\frac{\beta}{J}\right)^2\E_{B_1, \bb}[T_{\tau_\delta}(1)] + \dfrac{1}{8}\sum_{i=2}^A(\iota_{i1} - \iota_{i2})^2\E_{B_1, \bb}[T_{\tau_\delta}(i)] &\overset{(c)}{\geq} \log\left(\dfrac{1}{4\delta}\right)\nonumber
  \end{align}
 \begin{align}
 \dfrac{1}{8}\left(\dfrac{1}{J}\right)^2\beta^2\E_{B_1, \bb}[T_{\tau_\delta}(1)] + \dfrac{1}{8}\sum_{i=2}^A(\iota_{i1} - \iota_{i2})^2\E_{B_1, \bb}[T_{\tau_\delta}(i)] &\overset{}{\geq} \log\left(\dfrac{1}{4\delta}\right)\nonumber\\
 \dfrac{1}{8}\left(\dfrac{1}{J}\right)^2\beta^2\E_{B_1, \bb}[T_{\tau_\delta}(1)] + \dfrac{1}{8}\sum_{i=2}^A\frac{\beta^2}{4J^2}\E_{B_1, \bb}[T_{\tau_\delta}(i)] &\overset{(d)}{\geq} \log\left(\dfrac{1}{4\delta}\right)\label{eq:minimax-1}
\end{align}
where, $(a)$ follows from \Cref{lemma:tsybakov}, $(b)$ follows from \Cref{lemma:divergence-decomp},  $(c)$ follows from the construction of the bandit environments and \eqref{eq:oracle-proof-1}, and $(d)$ follows as  $(\iota_{ij} - \iota_{ij'})^2 \leq \frac{\beta^2}{4J^2}$ for any $i$-th action and $j$-th hypothesis.

Now, we consider the alternate model $B_3$. Again define the event \(\xi=\{\wtheta_{\tau_\delta} \neq \btheta_*\}\) as the error event in model $B_1$ and the event  \(\xi'=\{\wtheta_{\tau_\delta} \neq \btheta^{''}_*\}\) be the corresponding error event in model $B_3$. Note that $\xi^{\complement} \subset \xi'$. Now since $\bb$ is $\delta$-PAC policy we have $\Pb_{B_1,\bb}(\xi) \leq \delta$ and $\Pb_{B_3,\bb}(\xi^{\complement}) \leq \delta$. Following the same way as before we can show that,
\begin{align}
 \dfrac{1}{8}\left(\dfrac{2}{J}\right)^2\beta^2\E_{B_3, \bb}[T_{\tau_\delta}(1)] + \dfrac{1}{8}\sum_{i=2}^A\frac{\beta^2}{4J^2}\E_{B_3, \bb}[T_{\tau_\delta}(i)] &\overset{(d)}{\geq} \log\left(\dfrac{1}{4\delta}\right)\label{eq:minimax-2}.
\end{align}
Similarly, we get the equations for all the other $(J-2)$ alternate models in $\Lambda_1$. Now consider an optimization problem (ignoring the constant factor of $\frac{1}{8}$ across all the constraints) 
\begin{align*}
    &\min_{t_i : i \in [A]} \sum t_i\\
    s.t. \quad & \left( \frac{1}{J}\right)^2\beta^2 t_1 + \frac{\beta^2}{4J^2} \sum_{i=2}^A t_i \geq \log(1/4\delta)\\
    &\left(\frac{2}{J}\right)^2\beta^2 t_1 + \frac{\beta^2}{4J^2} \sum_{i=2}^A t_i \geq \log(1/4\delta)\\
    &\vdots\\
    &\left(\frac{J-1}{J}\right)^2\beta^2 t_1 + \frac{\beta^2}{4J^2} \sum_{i=2}^A t_i \geq \log(1/4\delta)\\
    &t_i \geq 0,  \forall i\in [A]
\end{align*}
where the optimization variables are $t_i$. 
It can be seen that the optimum objective value is $J^2\beta^{-2} \log(1/4\delta)$. Interpreting $t_i = \mathbb{E}_{B_1,\bb}[T_{\tau_{\delta}}(i)]$ for all $i$, we get that $\E_{B_1,\bb}[\tau_\delta] = \sum_{i}t_i = t_1 \geq J^2\beta^{-2} \log(1/4\delta)$ which gives us the required lower bound to the number of pulls of action $1$. Observe that the optimum objective value is reached by substituting $t_1 = J^2\beta^{-2} \log(1/4\delta)$ and $t_2 = \ldots = t_A = 0$. 
It follows that for verifying any hypothesis $\btheta_j\neq \btheta_*$ the verification proportion is given by $\bb_{\btheta_j} = (1,\underbrace{0,0,\ldots,0}_{\text{(A-1) zeros}})$. Observe setting $\beta = J\sqrt{\log(1/4\delta)/n}$ recovers $\tau_\delta = n$ which implies that a budget of $n$ samples is required for verifying hypothesis $\btheta_j = \btheta_*$. For the remaining steps we take $\beta = J\sqrt{\log(1/4\delta)/n}$. 

\textbf{Step 3 (Lower Bounding Regret):} Then we can show that the MSE of any hypothesis $\btheta_j = \btheta_*$
\begin{align*}
    \E_\D\left[\left(\sum_a \pi(a)\bx(a)^\top(\btheta_j - \wtheta_n)\right)^2\right] = \dfrac{1}{n}\sum_{a,a'}\bw(a)\bA^{-1}_{\bb_{\btheta_j},\bSigma_{\btheta_*}}\bw(a') = \dfrac{1}{n}\Tr\bigg(\bA^{-1}_{\bb_{\btheta_j},\bSigma_{\btheta_*}} \underbrace{\sum_{a,a'}\bw(a)\bw(a')}_{\bV}\bigg)
\end{align*}
where, $\bb_{\btheta_j}(a)$ is the number of samples allocated to action $a$. 
%
First we will bound the loss of the oracle for this environment given by $\L_{n}(\pi,\bb,\bSigma_{\btheta_*}) = \frac{1}{n}\Tr(\bA^{-1}_{\bb_{\btheta_j},\bSigma_{\btheta_*}}\bV)$. Note that the oracle has access to the $\bSigma_{\btheta_*}$, so it only need to verify whether $\btheta_j = \btheta_*$ by following $\bb_{\btheta_j}$. Then we have that
\begin{align*}
    &\bA^{}_{\bb_{\btheta_j},\bSigma_{\btheta_*}} = \sum_a \bb_{\btheta_j}(a)\dfrac{\bx(a)\bx(a)^\top}{\sigma^2(a)} = \dfrac{\bw(1)\bw(1)^\top}{(\bx(1)^\top\btheta_j)^2} = \dfrac{\bw(1)\bw(1)^\top}{(\beta - \frac{j\beta}{J})^2} \\
    \implies & \Tr(\bA^{-1}_{\bb_{\btheta_j},\bSigma_{\btheta_*}}) = \dfrac{(\beta - \frac{j\beta}{J})^2}{\Tr(\bw(1)\bw(1)^\top)}
\end{align*}
Now we will bound the loss of the algorithm that uses $\wSigma_\beta$ to estimate $\wb$. It then collects the $\D$ and uses it to estimate $\btheta_*$ following the WLS estimation using $\bSigma_{\btheta_*}$.

Denote the number of times the algorithm samples each action $i$ be $T'_{n}(i)$. Let the algorithm allocate $T'_{n}(1) = J^2\beta^{-2} \log(1/4\delta) - d$ samples to action $1$ and to any other action $i'$ it allocates $T'_{n}(i') =  d$ samples such that $d\geq 1$. WLOG let $i'= 2$. Finally let $T'_{n}(3) = \ldots = T'_{n}(A) = 0$. Hence the optimal action $1$ is under-allocated and the sub-optimal action $2$ is over-allocated. The loss of such an algorithm now is given by
\begin{align*}
    \L_{n}(\pi,\wb,\bSigma_{\btheta_*}) = \frac{1}{n}\Tr(\bA^{-1}_{\wb_{},\bSigma_{\btheta_*}}\bV).
\end{align*}
%
%
%
%
%
%
Hence it follows by setting $\delta = 1/(n J)$ that
\begin{align*}
    \bA^{}_{\wb_{},\bSigma_{\btheta_*}} = \dfrac{1}{n}\sum_a n\wb_{}(a)\dfrac{\bx(a)\bx(a)^\top}{\sigma^2(a)} &= \dfrac{1}{n}\sum_a T'_n(a)\dfrac{\bx(a)\bx(a)^\top}{\sigma^2(a)} \\
    &= \dfrac{1}{n} T'_{n}(1)\dfrac{\bx(1)\bx(1)^\top}{\sigma^2(1)} + \underbrace{\dfrac{1}{n} T'_{n}(2)\dfrac{\bx(2)\bx(2)^\top}{\sigma^2(2)}}_{\geq 0}\\
    &\geq \dfrac{1}{n} T'_{n}(1)\dfrac{\bw(1)\bw(1)^\top}{(\bx(1)^\top\btheta_j)^2}\\
    &\overset{(a)}{=} \dfrac{J^2\beta^{-2} \log(nJ) - d}{n} \dfrac{\bw(1)\bw(1)^\top}{(\beta - \frac{j\beta}{J})^2} 
\end{align*}
where, $(a)$ follows by substituting the value of $T'_{n}$.
Then we have that
\begin{align*}
    \Tr(\bA^{-1}_{\wb_{},\bSigma_{\btheta_*}}) \geq \dfrac{n}{J^2\beta^{-2} \log(n J) - d} \dfrac{(\beta - \frac{j\beta}{J})^2}{\Tr(\bw(1)\bw(1)^\top)} &= \dfrac{n}{J^2\beta^{-2} (\log(n J) - \frac{d}{J^2\beta^{-2}})} \dfrac{(\beta - \frac{j\beta}{J})^2}{\Tr(\bw(1)\bw(1)^\top)}\\
    &\overset{(a)}{\geq} \dfrac{\beta^2\log(n J) + \frac{d}{J^2\beta^{-2}}}{J^2}\dfrac{(\beta - \frac{j\beta}{J})^2}{\Tr(\bw(1)\bw(1)^\top)}\\
    &\overset{}{\geq} \dfrac{\beta^2\log(n J)}{J^2}\dfrac{(\beta - \frac{j\beta}{J})^2}{\Tr(\bw(1)\bw(1)^\top)}
\end{align*}
where, $(a)$ follows as for $d\geq 1$ we have that $$n - (\log (n J))^2 \geq - \frac{d^2}{(J^2\beta^{-2})^2} \implies (\log(n J) - \frac{d}{J^2\beta^{-2}})^{-1} \geq \log(n J) + \frac{d}{J^2\beta^{-2}}.$$ 

\textbf{Step 4 (Lower Bound regret):}
Hence we have the regret for verifying any hypothesis $\btheta_j=\btheta_*$ as follows:
\begin{align*}
    \cR'_{n} &= \L_{n}(\pi,\wb,\bSigma_{\btheta_*}) - \L_{n}(\pi,\bb^*,\bSigma_{\btheta_*})\\ 
    &\geq \dfrac{1}{{n}}\Tr\bigg(\bA^{-1}_{\wb_{},\bSigma_{\btheta_*}} \bV\bigg) - \dfrac{1}{{n}}\Tr\bigg(\bA^{-1}_{\bb_{\btheta_j},\bSigma_{\btheta_*}} \bV\bigg) = \dfrac{1}{{n}}\Tr\bigg(\bigg(\bA^{-1}_{\wb_{},\bSigma_{\btheta_*}}-\bA^{-1}_{\bb_{\btheta_j},\bSigma_{\btheta_*}} \bigg)\bV\bigg)\\
    &\geq \dfrac{\lambda_d(\bV)}{{n}}\Tr\bigg(\bA^{-1}_{\wb_{},\bSigma_{\btheta_*}}-\bA^{-1}_{\bb_{\btheta_j},\bSigma_{\btheta_*}} \bigg)\\
    &= \dfrac{\lambda_d(\bV)}{{n}}\left[\Tr\bigg(\bA^{-1}_{\wb_{},\bSigma_{\btheta_*}}\bigg)-\Tr\bigg(\bA^{-1}_{\bb_{\btheta_j},\bSigma_{\btheta_*}} \bigg)\right]\\
    &=  \dfrac{\lambda_d(\bV)}{{n}}\left[\dfrac{\beta^2\log({n} J)}{J^2}\dfrac{(\beta - \frac{j\beta}{J})^2}{\Tr(\bw(1)\bw(1)^\top)}  -  \dfrac{(\beta - \frac{j\beta}{J})^2}{\Tr(\bw(1)\bw(1)^\top)}\right]\\
    &= \dfrac{\lambda_d(\bV)\beta^2(\beta - \frac{j\beta}{J})^2}{{n} \Tr(\bw(1)\bw(1)^\top)}\left[\dfrac{\log({n}J)}{J^2} - 1\right]
\end{align*}
\begin{align*}
    &\overset{(a)}{\geq} \dfrac{\lambda_d(\bV)\beta^2(\beta - \frac{j\beta}{J})^2}{{n} \Tr(\bw(1)\bw(1)^\top)}\left[\dfrac{\log({n}J)}{2J^2}\right]\\
    &\overset{(b)}{\geq} \dfrac{d\lambda_d(\bV)\beta^2}{{n}^{3/2} \Tr(\bw(1)\bw(1)^\top)}\left[\dfrac{\log({n}J)}{2J^2}\right] \\
    &\overset{(c)}{\geq} \dfrac{d^2\lambda_d(\bV)\beta^2}{{n}^{3/2} \Tr(\bw(1)\bw(1)^\top)}\log(2{n}) \\
    & = \Omega\left(\dfrac{d^2\lambda_d(\bV)\log({n})}{{n}^{3/2}}\right)
\end{align*}
where, $(a)$ follows as $\frac{\log(n J)}{J^2} - 1 \geq \frac{\log(n J)}{2J^2}$, $(b)$ follows as gap $(\beta - \frac{j\beta}{J})^2\geq \frac{d}{\sqrt{n}}$ for any $\btheta_j$, and $(c)$ follows by substituting $|\bTheta| = J = 2^d$. 
\end{proof}

\begin{lemma}
\label{lemma:divergence-decomp}
\textbf{(Restatement of Lemma 15.1 in \citet{lattimore2018bandit}, Divergence Decomposition)}
Let $B$ and $B'$ be two bandit models having different optimal hypothesis $\btheta_*$ and $\btheta^{'*}$ respectively. Fix some policy \(\pi\) and round $n$. Let \(\Pb_{B, \pi}\) and \(\Pb_{B', \pi}\) be two probability measures induced by some $n$-round interaction of \(\pi\) with \(B\) and \(\pi\) with
\(B'\) respectively. Then
\begin{align*}
    \KL\left(\Pb_{B, \pi}|| \Pb_{B', \pi}\right)=  \sum_{i=1}^{A} \E_{B, \pi}[T_n(i)]\cdot\KL(\N(\mu_{i}(\btheta),1)||\N(\mu_{i}(\btheta_*),1)) 
\end{align*}
where, $\KL\left(.||.\right)$ denotes the Kullback-Leibler divergence between two probability measures and $T_n(i)$ denotes the number of times action $i$ has been sampled till round $n$.
\end{lemma}

\begin{lemma}
\label{lemma:tsybakov}
\textbf{(Restatement of Lemma 2.6 in \citet{tsybakov2008introduction})}
Let \(\mathbb{P}, \mathbb{Q}\) be two probability measures on the same measurable space \((\Omega, \F)\) and let \(\xi \subset \F \) be any arbitrary event then
\[
\mathbb{P}(\xi)+ \mathbb{Q}\left(\xi^{\complement}\right) \geqslant \frac{1}{2} \exp\left(-\KL(\mathbb{P}|| \mathbb{Q})\right)
\]
where \(\xi^{\complement}\) denotes the complement of event \(\xi\) and \(\KL(\mathbb{P}||\mathbb{Q})\) denotes the Kullback-Leibler divergence between \(\mathbb{P}\) and \(\mathbb{Q}\).
\end{lemma}


\textbf{Environment $\mathcal{E}$:} Consider the environment $\mathcal{E}$ which consist of $3$ actions in $\R^2$ such that $\bx(1)= [1,0]$ is along $x$-axis, $\bx(2) = [0,1]$ is along $y$-axis and $\bx(3) = [1/\sqrt{2}, 1/\sqrt{2}]$. Let $\btheta_* = [1,0]$ and so the optimal action is action $1$. 
    Let the target policy $\pi = [0.9, 0.1, 0.0]$. Finally, let the variances be $\sigma^2(1) = 5/100$, $\sigma^2(2) = 1.0$ and $\sigma^2(3) =  5/100$.

\begin{customproposition}{8}\textbf{(\onp\ regret)} 
    Let the \onp\ algorithm have access to the variance in environment $\mathcal{E}$. Then the regret of \onp\ scales as $O\left(\dfrac{\lambda_1(\bV)}{n}\right)$.
\end{customproposition}

\begin{proof}
    Recall that in $\mathcal{E}$, there are $3$ actions in $\R^2$ such that $\bx(1)= [1,0]$ is along $x$-axis, $\bx(2) = [0,1]$ is along $y$-axis and $\bx(3) = [1/\sqrt{2}, 1/\sqrt{2}]$. The $\btheta_* = [1,0]$ and so the optimal action is action $1$. 
    The target policy $\pi = [0.9, 0.1, 0.0]$. Finally, let the variances be $\sigma^2(1) = 1.0$, $\sigma^2(2) = 1.0$ and $\sigma^2(3) =  5/100$.
    Hence, \PE design results in $\bb^* = [0.5, 0.5, 0.0]$. 
    \begin{align*}
    \bA^{}_{\pi_{},\bSigma_{\btheta_*}} &= \sum_a \pi_{}(a)\dfrac{\bx(a)\bx(a)^\top}{\sigma^2(a)} = \frac{9}{10}\cdot \bx(1)\bx(1)^\top + \frac{1}{10}\bx(2)\bx(2)^\top\\
    \bA^{}_{\bb_{*},\bSigma_{\btheta_*}} &= \sum_a \bb_{*}(a)\dfrac{\bx(a)\bx(a)^\top}{\sigma^2(a)} = \frac{1}{2}\cdot \bx(1)\bx(1)^\top + \frac{1}{2}\bx(2)\bx(2)^\top
\end{align*}
Recall that $\bV = \sum_{a}\bw(a)\bw(a)^\top$. Hence, the regret scales as 
\begin{align*}
    \cR_{n} = \L_{n}(\pi,\pi,\bSigma_{\btheta_*}) - \L_{n}(\pi,\bb^*,\bSigma_{\btheta_*})
    &\leq \dfrac{1}{{n}}\Tr\bigg(\bA^{-1}_{\pi_{},\bSigma_{\btheta_*}} \bV\bigg) - \dfrac{1}{{n}}\Tr\bigg(\bA^{-1}_{\bb_{*},\bSigma_{\btheta_*}} \bV\bigg) = \dfrac{1}{{n}}\Tr\bigg(\bigg(\bA^{-1}_{\pi_{},\bSigma_{\btheta_*}}-\bA^{-1}_{\bb_{*},\bSigma_{\btheta_*}} \bigg)\bV\bigg)\\
    & \overset{(a)}{\leq} O\left(\dfrac{\lambda_1(\bV)}{n}\right)
\end{align*}
where, $(a)$ follows by substituting the value of $\bA^{}_{\pi_{},\bSigma_{\btheta_*}}$ and $\bA^{}_{\bb_{*},\bSigma_{\btheta_*}}$.
\end{proof}

%% file: addl_expt.tex
In this section, we state additional experimental details.

\textbf{Unit Ball:} This experiment consists of a set of $4$ actions that are arranged in a unit ball in $\R^2$,  
and $\|\bx(a)\|=1$ for all $a\in\A$. We consider three groups of actions: \textbf{a)} the reward-maximizing action in the direction of $\btheta^*$, \textbf{b)} the informative action (orthogonal to optimal action) that maximally reduces the uncertainty of $\wtheta_t$ and \textbf{c)} the  less-informative actions as shown in \Cref{fig:linear-expt} (Top-Left). The variance of the most informative action is chosen to be high $(0.35)$, but the target probability is set as low $0.1$, which forces the on-policy algorithm to sample the high variance action less.
\Cref{fig:linear-expt} (Top-Right) shows that \sp\  outperforms \onp, \go, and \ao. Note that we experiment with \ao\ design \citep{fontaine2021online} because this criterion results in minimizing the average variance of the estimates of the regression coefficients and is most closely aligned with our goal than G-, or, D-optimal designs \citep{jamieson2022interactive}.

\textbf{Air Quality:} We perform this  experiment on real-world dataset Air Quality from UCI datasets. The Air quality dataset consists of $1500$ samples each of which consists of $6$ features. We first select $400$ samples which are the actions in our setting.  We then fit a weighted least square estimate to the original dataset and get an estimate of $\btheta_*$ and $\bSigma_*$. The reward model is linear and given by $\mathbf{x}_{I_t}^\top\btheta_* + \text{noise}$ where $\bx_{I_t}$ is the observed action at round $t$, and the noise is a zero-mean additive noise with variance scaling as $\bx_{I_t}^\top\bSigma_*\bx_{I_t}$. 
Hence the variance of each action depends on their feature vectors and $\bSigma_*$. Finally, we set a level $\tau$, such that $30$ actions having variance crossing $\tau$ are set with low target probability, and the remaining probability mass is uniformly distributed among the rest $370$ action. Hence, again high variance actions are set with a low target probability, which forces the on-policy algorithm to sample the high-variance action less number of times.
We apply \sp\ to this problem and compare it to baselines \ao, \go, and the \onp\ algorithm. 

\begin{figure}[!ht]
\centering
\begin{tabular}{cc}
\label{fig:expt-linear3-tab-oracle}\hspace*{-1.2em}\includegraphics[scale = 0.33]{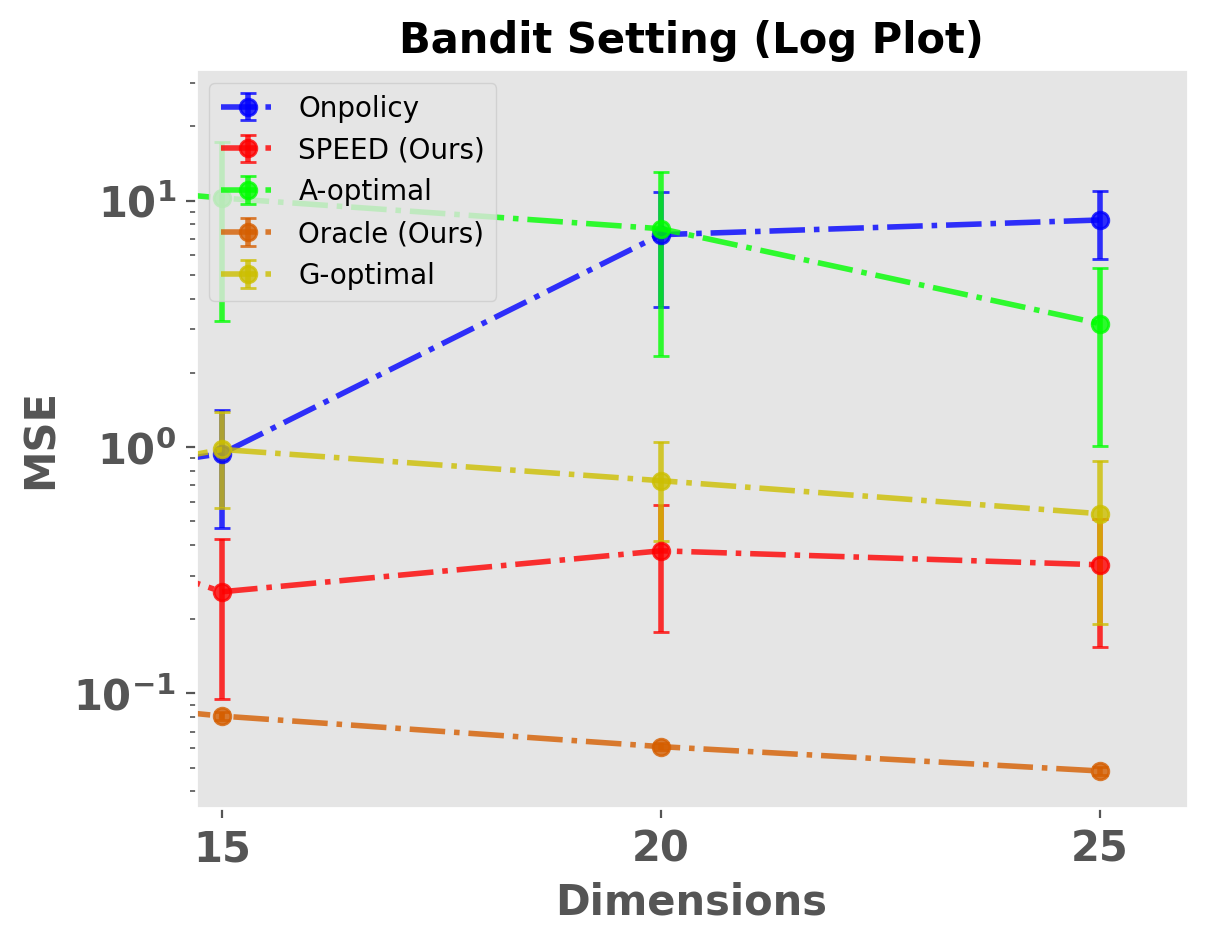} &
\hspace*{-1.2em}\includegraphics[scale = 0.37]{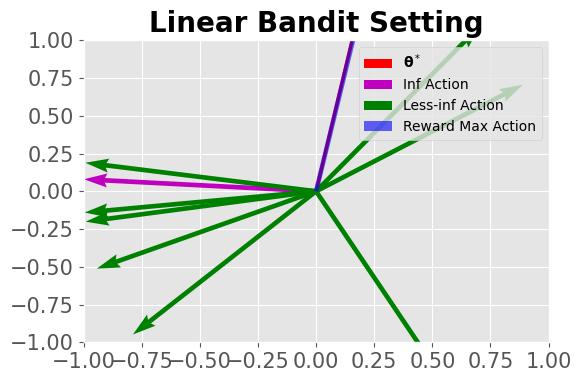} 
\end{tabular}
\caption{
$10$ action unit ball environment 
}
\label{fig:linear-expt3}
\vspace{-1.0em}
\end{figure}

\textbf{Red Wine Quality:} The UCI Red Wine Quality dataset consist of $1600$ samples of red wine with each sample $i$ having feature $\mathbf{x}_i\in\mathbb{R}^{11}$. We first fit a weighted least square estimate to the original dataset and get an estimate of $\btheta^*$ and $\bSigma_*$. The reward model is linear and given by $\mathbf{x}_{I_t}^T\btheta^* + \text{noise}$ where $\bx_{I_t}$ is the observed action at round $t$, and the noise is a zero-mean additive noise with variance scaling as $\bx_{I_t}^\top\bSigma_*\bx_{I_t}$. Note that we consider the $1600$ samples as actions. Then we run each of our benchmark algorithms on this dataset and reward model. Finally, we set a level $\tau$, such that $40$ actions having variance crossing $\tau$ are set with low target probability, and the remaining probability mass is uniformly distributed among the rest $1560$ action. Hence, again high variance actions are set with a low target probability, which forces the on-policy algorithm to sample the high-variance action less number of times.
We apply \sp\ to this problem and compare it to baselines \ao, \go, and the \onp\ algorithm.

\textbf{Movielens:} We experiment with a movie recommendation problem on the MovieLens 1M dataset \citep{movielens}. This dataset contains one million ratings given by $6\,040$ users to $3\,952$ movies. We first apply a low-rank factorization to the rating matrix to obtain $5$-dimensional representations: $\btheta_j \in \R^5$ for user $j \in [6\,040]$ and $\bx(a) \in \R^5$ for movie $a \in [3\,952]$.  In each run, we choose one user $\btheta_j$ and $100$ movies $\bx(a)$ randomly, and they represent the unknown model parameter and known feature vectors of actions, respectively. 

\textbf{Increasing Dimension:} We perform this experiment to show how the MSE of \sp\ scales with increasing dimensions and number of actions. We choose dimension $d\in\{15,20,25\}$. For each dimension $d\in\{15,20,25\}$ we choose the number of actions $|\A|=d^2+20$. Hence we ensure that the number of actions are greater than $d^2$ dimensions. We also choose the horizon as $T\in \{13000, 18000, 25000\}$ for each $d\in\{15,20,25\}$. We choose the same environment as the unit ball experiment. So the actions arranged in a unit ball in $\R^2$ and $\|\bx(a)\|=1$ for all $a\in\A$. Again we consider three groups of actions: \textbf{a)} the reward-maximizing action in the direction of $\btheta^*$, \textbf{b)} the informative action (orthogonal to optimal action) that maximally reduces the uncertainty of $\wtheta_t$ and \textbf{c)} the  less-informative actions as shown in \Cref{fig:linear-expt3} but scaled to a larger set of actions. 
For each case of dimension $d\in\{15,20,25\}$, the variance of the most informative actions along the directions orthogonal to the reward maximizing action are chosen to be high, but the target probability is set as low, which forces the on-policy algorithm to sample the high variance action less.
We again show the performance in \Cref{fig:linear-expt} (Bottom-left). 
We observe that with increasing dimensions $d$ the \sp\  outperforms on-policy. Also, observe that the oracle with knowledge of $\bSigma_*$ performs the best.